\documentclass[11pt]{article}
\pdfoutput=1
\usepackage{arxiv}

\usepackage[utf8]{inputenc} 
\usepackage[T1]{fontenc}    
\usepackage{hyperref}       
\usepackage{url}            
\usepackage{booktabs}       
\usepackage{amsfonts}       
\usepackage{nicefrac}       
\usepackage{microtype}      
\usepackage{lipsum}		
\usepackage{graphicx}
\usepackage{natbib}
\usepackage{doi}

\usepackage{enumitem}
\usepackage{amsmath,amssymb,amsthm,amsfonts}
\usepackage{algorithm}
\usepackage{algorithmic}
\usepackage{mathtools} 
\usepackage{bbm}
\usepackage{subcaption}
\usepackage{xcolor}

\newcommand{\convp}{\overset{\textsf{p}}\rightarrow}

\newtheorem{thm}{Theorem}[section]

\newtheorem{coro}{Corollary}[section]

\theoremstyle{definition}
\newtheorem{defn}{Definition}[section]
\newtheorem{lemma}{Lemma}[section]
\newtheorem{rmk}{Remark}[section]

\newtheorem{asp}{Assumption}

\title{Transformers Handle Endogeneity in In-Context Linear Regression}

\author{Haodong Liang\\
	UC Davis\\
	\texttt{hdliang@ucdavis.edu} \\
	\And
	Krishnakumar Balasubramanian \\
	UC Davis\\
	\texttt{kbala@ucdavis.edu} \\
	\And
	Lifeng Lai \\
	UC Davis\\
	\texttt{lflai@ucdavis.edu} \\
}

\renewcommand{\headeright}{}
\renewcommand{\undertitle}{}

\begin{document}
\maketitle

\begin{abstract}
	We explore the capability of transformers to address endogeneity in in-context linear regression. Our main finding is that transformers inherently possess a mechanism to handle endogeneity effectively using instrumental variables (IV). First, we demonstrate that the transformer architecture can emulate a gradient-based bi-level optimization procedure that converges to the widely used two-stage least squares (\textsf{2SLS}) solution at an exponential rate. Next, we propose an in-context pretraining scheme and provide theoretical guarantees showing that the global minimizer of the pre-training loss achieves a small excess loss. Our extensive experiments validate these theoretical findings, showing that the trained transformer provides more robust and reliable in-context predictions and coefficient estimates than the \textsf{2SLS} method, in the presence of endogeneity.
\end{abstract}
\section{Introduction}
The transformer architecture \citep{vaswani} has demonstrated remarkable in-context learning (ICL) capabilities across various domains, such as natural language processing \citep{Devlin2019BERTPO,radford2019language,brown2020language}, computer vision \citep{dosovitskiy2021image, carion2020endtoend}, and reinforcement learning \citep{lee2022multigame,parisotto2019stabilizing}. Self-attention mechanism, a core component of transformers, allows these models to capture long-range dependencies in data, which is critical for success in these tasks. Despite their impressive performance, the theoretical understanding of transformers remains limited, leaving important questions unanswered about their true capabilities and the underlying mechanisms driving their exceptional results.

Recent efforts to theoretically understand transformers' ICL capabilities have focused on their performance in fundamental statistical tasks. Focusing on simple function classes,~\cite{garg2023transformers} highlighted that transformers, when trained on sufficiently large and diverse data from a specific function class, can generalize across most functions of that class without task-specific fine-tuning. Building on this, subsequent work by~\cite{statistician} established that attention layers enable transformers to perform gradient descent, implementing algorithms like linear regression, logistic regression, and LASSO; see also~\cite{akyurek2023what,von2023transformers,li2023transformers,fu2023transformers,ahn2024transformers,jin2025incontextlearningmixturelinear}. The learning dynamics of transformer trained via gradient descent for in-context learning linear function classes was analyzed in~\cite{Huang:ICML:24}. Furthermore \cite{zhang2023trained,zhang2024incontext}  showed that \emph{trained} transformers’ ICL abilities for linear regression tasks are theoretically robust under certain distributional shifts and characterized the corresponding sample complexities.

Existing works on analyzing the ICL ability of transformers for linear regression tasks, however, ignore \emph{endogeneity} and have mainly focused on the \emph{exogenous} setup where the additive noise is uncorrelated with the explanatory variables. Ignoring \emph{endogeneity} in linear regression leads to biased and inconsistent estimates, resulting from issues like omitted variable bias, simultaneity, and measurement error, which can distort causal inferences and lead to incorrect policy conclusions~\citep{hausman2001mismeasured, wooldridge2015introductory, angrist2009mostly,greene2003econometric}. Instrumental variable (IV) regression is a widely adopted method to handle endogeneity by utilizing instruments that are correlated with the endogenous variables but uncorrelated with the error term \citep{10.1257/jep.15.4.69}. A naturally intriguing question that therefore arises is:
\begin{center}
\emph{Can transformers leverage instrumental variables and provide reliable predictions\\ and coefficient estimates, in the presence of endogeneity?}
\end{center}
In this work, we aim to answer this question and offer new insights on in-context linear regression tasks. Our key contributions include:
\begin{itemize}[noitemsep,leftmargin=0.2in]
    \item We demonstrate that looped transformers can address endogeneity in linear regression by leveraging instrumental variables. Specifically, we show that transformers can implement two-stage least squares (\textsf{2SLS}) regression through a bi-level gradient descent procedure, where each iteration is executed by a two-layer transformer block. Moreover, the convergence rate to the \textsf{2SLS} estimator is exponential with respect to the number of blocks.
    \item We propose an ICL training scheme for transformers to efficiently handle endogeneity. Under this scheme, we show that the global minimizer of the in-context pre-training loss achieves a small excess loss compared to the global optimal expected loss.
    \item We evaluate the performance of the trained transformer model through extensive experiments, finding that it not only matches the performance of the \textsf{2SLS} estimator on standard IV tasks but also generalizes effectively to more complex scenarios, including the challenging cases of weak instruments, non-linear IV, and underdetermined IV problems.
    \item As part of our analysis, we derive the first non-asymptotic bound for the \textsf{2SLS} estimator under random design, providing valuable insights for future theoretical work.
\end{itemize}

\subsection{Related works}

\textbf{In-context Learning.} Initial works by \cite{garg2023transformers} and \cite{statistician} adopted the standard multi-layer transformer architecture to conduct the experiments. Later, \cite{giannou2023loopedtransformersprogrammablecomputers} and \cite{yang2024loopedtransformersbetterlearning} showed that a looped architecture reduces the required depth of transformers and exhibits better efficiency in learning algorithms. \cite{gao2024expressivepowervariantlooped} illustrated that the looped transformer architecture with extra pre-processing and post-processing layers can achieve higher expressive power than a standard transformer with the same number of parameters.  Apart from works concerning the implementability of first-order gradient descent algorithms by transformers, other works have also examined higher-order and non-parametric optimization methods. Specifically, \cite{giannou2024transformersemulateincontextnewtons} showed that transformers can emulate Newton's method for logistic regression. \cite{cheng2024transformersimplementfunctionalgradient} showed that transformers can implement functional gradient descent and hence enable them to learn non-linear functions in-context. Relationship between in-context learning and Bayesian inference is also studied in~\cite{ye2024pre,falck2024context}.

\cite{nichani2024transformerslearncausalstructure} illustrated how the transformers can learn the causal structure by encoding the latent causal graph in the first attention layer.~\cite{goel2024can} explored the representational power of transformer for learning linear dynamical systems.~\cite{makkuva2024local,makkuva2024attention,rajaraman2024transformers,edelman2024evolution} considered ICL Markov chains with transformers, including both landscape and training dynamics analyses. To the best of our knowledge, we are not aware of prior works on handling endogeniety with transformers.

\textbf{Instrumental Variable Regression.} IV regression has been widely studied in econometrics~\citep{10.1257/jep.15.4.69,angrist2009mostly}. Recent works in machine learning explored the optimization based approaches for the IV regression problem. \cite{singh2020kernelinstrumentalvariableregression} proposed the kernel IV regression to model non-linear relationship between variables. \cite{muandet2020dualinstrumentalvariableregression} proposed that a non-linear IV regression problem can be formulated as a convex-concave saddle point problem. \cite{della2023stochastic,chen2024stochasticoptimizationalgorithmsinstrumental,peixoto2024nonparametric} proposed a stochastic optimization algorithm for IV regression. 

\textit{Notation:}\quad Throughout this paper, unless otherwise specified, lower-case letters denote random variables or samples, while upper-case letters represent datasets (collections of samples). Bolded letters indicate vectors or matrices, whereas unbolded letters indicate scalars. The notation $\boldsymbol{X}_{:,i}$ refers to the $i$-th column, and $\boldsymbol{X}_{i,:}$ refers to the $i$-th row of matrix $\boldsymbol{X}$. $\lambda_{\min}(\cdot)$ denotes the minimum eigenvalue, and $\sigma_{\min}(\cdot)$ denotes the minimum singular value of a matrix. By default, $\|\cdot\|$ denotes the Euclidean norm for a vector, or the spectral norm for a matrix.

\section{Endogeneity and Instrumental Variable Regression}\label{subsec:IV Regression}

Suppose we are interested in estimating the relationship between response variable $y \in \mathbb{R}$ and predictor variable $\boldsymbol{x} \in \mathbb{R}^p$ with endogeneity. Given instruments $\boldsymbol{z} \in \mathbb{R}^q$, we consider the model
\begin{align} \label{eq:endogeneity}
y=\boldsymbol{\beta}^\top\boldsymbol{x}  +\epsilon_1,\quad\text{and}\quad
\boldsymbol{x}=\boldsymbol{\Theta}^\top \boldsymbol{z}+\boldsymbol{\epsilon_2},
\end{align}
where $\boldsymbol{\beta}\in\mathbb{R}^{p}$, and $\boldsymbol{\Theta}\in\mathbb{R}^{q\times p}$ are the true model parameters, $\epsilon_1\in\mathbb{R}$ and $\boldsymbol{\epsilon}_2\in\mathbb{R}^{p}$ are (centered) random noise terms with variance $\sigma_1^2$ and covariance matrix $\boldsymbol{\Sigma}_2$, respectively. Further, $\boldsymbol{\epsilon}_2$ is an unobserved noise correlated with $\epsilon_1$, leading to the correlation between $\boldsymbol{x}$ and $\epsilon_1$, which introduces confounding in the model between $\boldsymbol{x}$ and $y$. Under this setting, the standard ordinary least squares (\textsf{OLS}) estimator is a biased and inconsistent estimator of $ \boldsymbol{\beta}$ (see \cite{wooldridge2015introductory}, Chapter 9). To address this issue, instrumental variable (IV) regression is a widely used method to provide a consistent estimate for $\boldsymbol{\beta}$.

\begin{defn}[\textsf{2SLS} estimator] \label{def:2sls}
	IV regression is a regression model to provide consistent estimate on the causal effect $\boldsymbol{\beta}$ for the endogeneity problem (\ref{eq:endogeneity}), by utilizing the instrument $\boldsymbol{z}$. 
	Given observational values $(\boldsymbol{Z},\boldsymbol{X},\boldsymbol{Y})=\{(\boldsymbol{z}_i,\boldsymbol{x}_i, y_i)\}_{i=1}^n$, the standard approach to estimate the IV regression model is \textsf{2SLS}; see, for example, \cite{wooldridge2015introductory}, Chapter 15.
	   \begin{itemize}
		   \item [i.] \emph{First stage}: Regress $\boldsymbol{X}$ on $\boldsymbol{Z}$ to obtain $\hat{\boldsymbol{\Theta}}$ 
		   \begin{align*}
			   \hat{\boldsymbol{\Theta}} = (\boldsymbol{Z}^\top\boldsymbol{Z})^{-1}\boldsymbol{Z}^\top\boldsymbol{X}.
		   \end{align*}
		   \item [ii.] \emph{Second stage}: Regress $\boldsymbol{Y}$ on $\boldsymbol{Z}\hat{\boldsymbol{\Theta}}$ to obtain:
		   \begin{align}\label{eq:2SLS second stage}
			   \hat{\boldsymbol{\beta}}_{\textsf{2SLS}} = (\hat{\boldsymbol{\Theta}}^\top\boldsymbol{Z}^\top\boldsymbol{Z}\hat{\boldsymbol{\Theta}})^{-1}\hat{\boldsymbol{\Theta}}^\top\boldsymbol{Z}^\top\boldsymbol{Y}.
		   \end{align}
	   \end{itemize}
   \end{defn}
   
   We introduce the standard assumptions required to show the convergence rate of the above estimator.

   \begin{asp}[Instrumental variable]\label{iv assumption} 
    A random variable $\boldsymbol{z}\in \mathbb{R}^q$ is a valid IV, if it satisfies the following conditions:
    \begin{itemize}[noitemsep]
        \item[i.] Fully identification: $q\geq p$ (without loss of generality, we assume data $\boldsymbol{Z},\boldsymbol{X}$ are full rank).
        \item [ii.] Correlated to $\boldsymbol{x}$: $\textsf{Corr}(\boldsymbol{z},\boldsymbol{x})\neq\boldsymbol{0}$.
        \item [iii.] Conditional uncorrelated to $y$: $\textsf{Corr}(\boldsymbol{z},\epsilon_1)=0$.
    \end{itemize}
\end{asp}

In particular, condition (i) above ensures the existence of unique solution for $\hat{\boldsymbol{\beta}}_{\textsf{2SLS}}$. We refer to \citet[Chapter 12]{stock2011introduction}  for additional elaborate discussions on the above conditions. To derive non-asymptotic convergence rates, we further assume the following regularity conditions.

\begin{asp}[Regularity conditions]\label{regularity assumption}
	Suppose instrument $\boldsymbol{z}$ is a centered random variable. We assume the following conditions hold:
	\begin{itemize}
		\item[i.] Bounded parameters: $\|\boldsymbol{\beta}\|\leq B_\beta$, $\|\boldsymbol{\Theta}\|\leq B_\Theta$.
        \item[ii.] Bounded variables: 
        $\|\boldsymbol{z}\|\leq B_z,\|\boldsymbol{x}\|\leq B_x, |\epsilon_{1}|\leq B_{\epsilon_1}, \|\boldsymbol{\epsilon}_{2}\|\leq B_{\epsilon_2}$.
        \item[iii.] Linear instrument: $\mathbb{E}\left[x_k\middle| \boldsymbol{z}\right]=\langle\boldsymbol{\Theta}_k,\boldsymbol{z}\rangle$.
	\end{itemize}
\end{asp}
The boundedness condition in (ii) is required to invoke matrix Bernstein inequalities \citep{tropp2015introductionmatrixconcentrationinequalities} in the analysis. We anticipate that this condition  may be relaxed to subgaussian or moment conditions by using more sophisticated matrix concentration results. 

\begin{thm}[MSE of \textsf{2SLS} estimator]\label{thm:consistency}
    Given Assumptions~\ref{iv assumption} and \ref{regularity assumption}, consider clipping operation 
    $$
    \textsf{clip}_{B_\beta}(\hat{\boldsymbol{\beta}}):=\begin{cases}
        \hat{\boldsymbol{\beta}} & \text{if } \|\hat{\boldsymbol{\beta}}\| \leq B_\beta\\
        \frac{B_\beta}{\|\hat{\boldsymbol{\beta}}\|} \hat{\boldsymbol{\beta}} & \text{if } \|\hat{\boldsymbol{\beta}}\| > B_\beta 
        \end{cases}.
        $$
        When 
        $$
        n\geq \max\left\{4c^2B_z^4\left(q+\log\left(\frac{4c^2B_z^4K}{q}\right)-\frac{3}{2}\right), \frac{qe^{\frac{3}{2}}}{K}, \frac{p^2(q+1)^2K}{qK_0^2}\right\},
        $$
        where $K:=\frac{\lambda_{\min}(\boldsymbol{\Sigma}_z)}{6B_z^2}$ and $K_0:=\frac{\lambda_{\min}(\boldsymbol{\Sigma}_z)\sigma_{\min}^2(\boldsymbol{\Theta})}{2B_{\epsilon_2}^2}$, the mean squared error of the \textsf{2SLS} estimate is bounded by:
    \begin{align}\label{eq:Emse of 2SLS}
        \begin{split}
        \mathbb{E}\left[\|\textsf{clip}_{B_\beta}(\hat{\boldsymbol{\beta}}_{\textsf{2SLS}})-\boldsymbol{\beta}\|^2\right]
        &\leq\mathcal{O}\left(\frac{q}{n}\left(\frac{B_\beta^2}{K}+{C^2(n)\sigma_1^2}\right)\right),
        \end{split}
    \end{align}
    where
    $$
        C(n):=\frac{\left(B_{\Theta}+\sqrt{\frac{2p(q+1)B_{\epsilon_2}^2\log\left(\frac{K}{q}n\right)}{\lambda_{\min}(\boldsymbol{\Sigma}_z)n}}\right)B_z}{\lambda_{\min}(\boldsymbol{\Sigma}_z)\left(1-\frac{cB_{z}^2\left(\sqrt{q}+\sqrt{\log\left(\frac{K}{q}n\right)-\frac{1}{2}}\right)}{\sqrt{n}}\right)^2\left(\sigma_{\min}(\boldsymbol{\Theta})-\sqrt{\frac{2p(q+1)B_{\epsilon_2}^2\log\left(\frac{K}{q}n\right)}{\lambda_{\min}(\boldsymbol{\Sigma}_z)n}}\right)^2},
    $$ 
$\boldsymbol{\Sigma}_z:=\mathbb{E}[\boldsymbol{zz}^\top],$ and $c$ is an absolute constant.
\end{thm}
\begin{rmk}
    We keep the slightly complicated form \eqref{eq:Emse of 2SLS} so that the $\mathcal{O}$ notation only hides some absolute constant multipliers that are independent of problem-related constants. Note that when $n$ is large enough, we have $C(n)\rightarrow \frac{B_\Theta B_z}{\lambda_{\min}(\boldsymbol{\Sigma}_z)\sigma_{\min}^2(\boldsymbol{\Theta})}$, so $C(n)$ is also bounded. Thus the error bound \eqref{eq:Emse of 2SLS} decays with rate $\mathcal{O}(\frac{1}{n})$.
\end{rmk}

We note that although the consistency of the \textsf{2SLS} estimator is a standard result in econometrics, most existing works focus on the asymptotic properties of the estimator. Theorem \ref{thm:consistency} provides the first non-asymptotic bound for estimation error $\|\hat{\boldsymbol{\beta}}_{\textsf{2SLS}}-\boldsymbol{\beta}\|^2$, under random design. The detailed proof is provided in Appendix \ref{proof:thm:consistency}.

\section{Transformers Handle Endogeniety}\label{sec:mainresults}

\subsection{Transformer Architecture}
Denote the input matrix as $\boldsymbol{H} = [\boldsymbol{h}_1,\ldots,\boldsymbol{h}_n]\in \mathbb{R}^{D\times n}$, where each column corresponds to one sample vector.
\begin{defn}[Attention layer]
	A self-attention layer with $M$ heads is denoted as $\textsf{ATTN}_{\boldsymbol{\theta}}(\cdot)$, with parameters $\boldsymbol{\theta} = \{(\boldsymbol{Q}_m, \boldsymbol{K}_m, \boldsymbol{V}_m)\}_{m \in [M]} \subseteq \mathbb{R}^{D \times D}$. Given input $\boldsymbol{H},$
   \begin{align}\label{def:H}
	   \tilde{\boldsymbol{H}} = \textsf{ATTN}_{\boldsymbol{\theta}}(\boldsymbol{H}) := \boldsymbol{H} + \frac{1}{n}\sum_{m=1}^M(\boldsymbol{V}_m\boldsymbol{H})\times \sigma((\boldsymbol{Q}_m\boldsymbol{H})^\top(\boldsymbol{K}_m\boldsymbol{H}))\in\mathbb{R}^{D\times n},
   \end{align}
   or element-wise:
   \begin{align}\label{def:h}
	   \tilde{\boldsymbol{h}}_i = [\textsf{ATTN}_{\boldsymbol{\theta}}(\boldsymbol{H})]_i := \boldsymbol{h}_i + \sum_{m=1}^M\frac{1}{n}\sum_{j=1}^n\sigma(\langle\boldsymbol{Q}_m\boldsymbol{h}_i, \boldsymbol{K}_m\boldsymbol{h}_j\rangle)\cdot \boldsymbol{V}_m\boldsymbol{h}_j\in\mathbb{R}^{D},
   \end{align}
   where $\sigma(\cdot)$ is the ReLU function.
\end{defn}
\begin{defn}[MLP layer] 
    An MLP layer is denoted as $\textsf{MLP}_{\boldsymbol{\theta}}(\cdot)$, with parameters $\boldsymbol{\theta} =  (\boldsymbol{W}_1,\boldsymbol{W}_2) \in \mathbb{R}^{D' \times D\times D\times D'}$. Given input $\boldsymbol{H},$
    \begin{align*}
        \tilde{\boldsymbol{H}} = \textsf{MLP}_{\boldsymbol{\theta}}(\boldsymbol{H}) := \boldsymbol{H} + \boldsymbol{W}_2\sigma(\boldsymbol{W}_1\boldsymbol{H}),
    \end{align*}
    or element-wise: 
    \begin{align*}
        \tilde{\boldsymbol{h}}_i = [\textsf{MLP}_{\boldsymbol{\theta}}(\boldsymbol{H})]_i := \boldsymbol{h}_i +\boldsymbol{W}_2\sigma(\boldsymbol{W}_1\boldsymbol{h}_i).
    \end{align*}
\end{defn}
\begin{defn}[Transformer] 
    An L-layer transformer is denoted as $\textsf{TF}_{\boldsymbol{\theta}}(\cdot)$, with parameters $\boldsymbol{\theta}=(\boldsymbol{\theta}_{\textsf{ATTN}}^{(1:L)},\boldsymbol{\theta}_{\textsf{MLP}}^{(1:L)})$. Given input $\boldsymbol{H}=\boldsymbol{H}^{(0)},$
    \begin{align*} 
        \boldsymbol{{H}}^{(l)} = \textsf{MLP}_{\boldsymbol{\theta}_{\textsf{MLP}}^{(l)}}(\textsf{ATTN}_{\boldsymbol{\theta}_{\textsf{ATTN}}^{(l)}}(\boldsymbol{H}^{(l-1)})) , \quad l=1,\ldots,L.
    \end{align*}
    The output of this transformer is the final layer output: $\tilde{\boldsymbol{H}}:=\boldsymbol{H}^{(L)}=\textsf{TF}_{\boldsymbol{\theta}}(\boldsymbol{H}^{(0)})$.
\end{defn}
\begin{defn}[Looped transformer]\label{def:looped transformer}
    An $\bar{L}$-looped transformer is a special transformer architecture, denoted as $\textsf{LTF}_{\bar{\boldsymbol{\theta}}, \bar{L}}(\cdot)$, with parameters $\bar{\boldsymbol{\theta}}=(\bar{\boldsymbol{\theta}}_{\textsf{ATTN}}^{(1:L_0)},\bar{\boldsymbol{\theta}}_{\textsf{MLP}}^{(1:L_0)})$. Given input $\boldsymbol{H}=\boldsymbol{H}^{(0)},$
    \begin{align*}
        \boldsymbol{{H}}^{(l)} = \textsf{TF}_{\bar{\boldsymbol{\theta}}}(\boldsymbol{H}^{(l-1)}), \quad l=1,\ldots,\bar{L}.
    \end{align*}
    The output of this looped transformer is the final loop output: $\tilde{\boldsymbol{H}}:=\boldsymbol{H}^{(\bar{L})}=\textsf{LTF}_{\bar{\boldsymbol{\theta}},\bar{L}}(\boldsymbol{H}^{(0)})$.
\end{defn}

Previous works (e.g., \cite{statistician}, \cite{zhang2023trained}) have shown that transformers can perform in-context linear regression by emulating gradient descent (GD) with in-context pretraining. However, these studies have two key limitations. First, their analysis is based on single-level optimization algorithms, which is insufficient to demonstrate that transformers can efficiently learn more complex algorithms like \textsf{2SLS} (Definition~\ref{def:2sls}). Second, most ICL-related research focuses on the predictive performance of transformers, paying little attention to their ability to provide accurate coefficient estimates. We extend the current ICL framework by showing that transformers can implement a bi-level GD procedure (see Section~\ref{sec:gdiv}) with looped transformer architecture (Definition~\ref{def:looped transformer}), allowing them to efficiently emulate \textsf{2SLS} and provide coefficient estimates that are at least as accurate as \textsf{2SLS} in the presence of endogeneity (as in \eqref{eq:endogeneity}).

\subsection{Gradient descent based IV regression}\label{sec:gdiv}
We first introduce a gradient-based bi-level optimization procedure to obtain the \textsf{2SLS} estimator in \eqref{eq:2SLS second stage}. Given the dataset $(\boldsymbol{Z},\boldsymbol{X},\boldsymbol{Y})=\{(\boldsymbol{z}_i,\boldsymbol{x}_i, y_i)\}_{i=1}^n$, the objective function of IV regression can be formulated as the following bi-level optimization problem:
\begin{align}
    \begin{aligned}\label{eq:IV loss objective}
    \min_{\boldsymbol{\beta}} \quad\mathcal{L}(\boldsymbol{\beta})=\frac{1}{n}\sum_{i=1}^n(y_i-\boldsymbol{z}_i^\top\hat{\boldsymbol{\Theta}}\boldsymbol{\beta})^2,\quad
    \text{where }\quad \hat{\boldsymbol{\Theta}}:=\underset{\boldsymbol{\Theta}}{\arg\min}\quad\frac{1}{n}\sum_{j=1}^n(\boldsymbol{x}_j-\boldsymbol{z}_j^\top\boldsymbol{\Theta})^2.
    \end{aligned}
\end{align}
Consider the following gradient updates with learning rates $\alpha,\eta$:
\begin{subequations}\label{eq:2SLS GD Update}
    \begin{align}
        \boldsymbol{\Theta}^{(t+1)} &= \boldsymbol{\Theta}^{(t)} - \eta \boldsymbol{Z}^\top(\boldsymbol{Z\Theta}^{(t)}-\boldsymbol{X})\label{eq:InnerLoopUpdate},\\
        \boldsymbol{\beta}^{(t+1)} &= \boldsymbol{\beta}^{(t)} - \alpha \boldsymbol{\Theta}^{(t)\top}\boldsymbol{Z}^\top(\boldsymbol{Z\Theta}^{(t)}\boldsymbol{\beta}^{(t)}-\boldsymbol{Y}).
        \label{eq:OuterLoopUpdate}
    \end{align}
\end{subequations}

Note that the GD-\textsf{2SLS} updates in \eqref{eq:2SLS GD Update} are designed to solve \eqref{eq:IV loss objective}. We now show that regardless the convergence of $\boldsymbol{\Theta}^{(t)}$, the GD estimator $\boldsymbol{\beta}^{(t)}$ will always converge to the \textsf{2SLS} estimator in \eqref{eq:2SLS second stage} with exponential rate.

\begin{thm}[Implementing \textsf{2SLS} with gradient-based method]\label{thm:2SLS GD}
    Given training data  $(\boldsymbol{Z},\boldsymbol{X},\boldsymbol{Y})=\{(\boldsymbol{z}_i,\boldsymbol{x}_i, y_i)\}_{i=1}^n$. Suppose the learning rates $\alpha,\eta$ satisfy the following conditions: 
    \begin{align*}
        0<\alpha<\frac{2}{\sigma_{\max}^2(\boldsymbol{Z}\hat{\boldsymbol{\Theta}})}\quad\text{and}\quad0<\eta<\frac{2}{\sigma_{\max}^2(\boldsymbol{Z})},
    \end{align*}
    where $\sigma_{\max}(\cdot)$ denotes the largest singular value of a matrix. Then, the GD updates in \eqref{eq:2SLS GD Update} converge to the \textsf{2SLS} estimator at an exponential rate:
    \begin{align*}
        \|\boldsymbol{\beta}^{(t)}-\hat{\boldsymbol{\beta}}_{\textsf{2SLS}}\|\leq\mathcal{O}\left(\Lambda^t\right),
    \end{align*}
    where, with $\rho(\cdot)$ denoting the spectral radius of the matrix,
    \begin{align}\label{eq:lambda}
      \Lambda:= \max\{\gamma(\alpha),\kappa(\eta)\},\quad\quad\gamma(\alpha):=\rho(\boldsymbol{I}-\alpha\hat{\boldsymbol{\Theta}}^\top\boldsymbol{Z}^\top\boldsymbol{Z}\hat{\boldsymbol{\Theta}}),\quad
            \kappa(\eta):=\rho(\boldsymbol{I}-\eta\boldsymbol{Z}^\top\boldsymbol{Z}).
        \end{align}
\end{thm}
To the best of our knowledge, Theorem \ref{thm:2SLS GD} provides the first theoretical result demonstrating that \textsf{2SLS} can be efficiently implemented using a gradient-based method, with an exponential convergence rate. We provide the proof in Appendix \ref{proof:thm:2SLS GD} and present simulation results in Appendix \ref{sec:exp 2SLS GD} to examine the convergence behavior of the optimization process.

\subsection{Transformers Can Efficiently Implement GD-\textsf{2SLS}}
The looped transformer architecture (Definition~\ref{def:looped transformer}), as proposed by \cite{giannou2023loopedtransformersprogrammablecomputers}, introduces an efficient approach to learn iterative algorithms by cascading the same transformer block for multiple times. With the GD updates in \eqref{eq:2SLS GD Update}, we will show that there exists a looped transformer architecuture that can efficiently learn the \textsf{2SLS} estimator. We emphasize here that although we can implement \textsf{2SLS} by sequentially attaching two separate GD iterates (each handling \textsf{OLS} for one stage), the overall convergence depends heavily on the convergence of the first stage estimate $\hat{\boldsymbol{\Theta}}$. Hence, significantly more number of layers are needed to ensure convergence. In addition, the advantage of looped transformer architecture cannot be fully exploited with this approach.

\begin{thm}[Implement a step of GD-\textsf{2SLS} with a transformer block]\label{thm:transformer}
	Suppose the embedded  input matrix takes the form:
	
	\begin{align}\label{eq:transformed input}
		\boldsymbol{H}^{(2l)}=\begin{bmatrix}
			\boldsymbol{z}_1 & \cdots & \boldsymbol{z}_n &\boldsymbol{z}_{n+1}\\ \boldsymbol{x}_1 & \cdots &\boldsymbol{x}_n &\boldsymbol{x}_{n+1}\\ y_1 &\cdots & y_n & 0\\ \boldsymbol{\Theta}_{:,1}^{(l)} & \cdots& \boldsymbol{\Theta}_{:,1}^{(l)} & \boldsymbol{\Theta}_{:,1}^{(l)}\\\vdots & \vdots &\vdots&\vdots \\\boldsymbol{\Theta}_{:,p}^{(l)} & \cdots &\boldsymbol{\Theta}_{:,p}^{(l)} & \boldsymbol{\Theta}_{:,p}^{(l)} \\\boldsymbol{\beta}^{(l)} &\cdots&\boldsymbol{\beta}^{(l)}&\boldsymbol{\beta}^{(l)}\\ \hat{\boldsymbol{x}}_1^{(l)} & \cdots&\hat{\boldsymbol{x}}_n^{(l)}& \hat{\boldsymbol{x}}_{n+1}^{(l)}\\
			1&\cdots &1&1\\1&\cdots &1&0
		\end{bmatrix}  \in \mathbb{R}^{D\times (n+1)}.
	\end{align}
	Given $\boldsymbol{H}^{(2l)}$, there exists a double-layer attention-only transformer block with parameters $\boldsymbol{\theta}=\boldsymbol{\theta}_{\textsf{ATTN}}^{(2l+1:2l+2)} = \{(\boldsymbol{Q}_m^{(2l+1:2l+2)}, \boldsymbol{K}_m^{(2l+1:2l+2)}, \boldsymbol{V}_m^{(2l+1:2l+2)})\}_{m\in[M^{(2l+1:2l+2)}]}\subset \mathbb{R}^{D\times D}$, where the number of heads  $M^{(2l+1)}=2p$, $M^{(2l+2)}=2(p+1)$ and embedding dimension $D=qp+3p+q+3$, that implements a \textsf{2SLS} gradient update in \eqref{eq:2SLS GD Update} with any given learning rates $\alpha,\eta$:
	\begin{align*}
		\boldsymbol{H}^{2(l+1)}=\textsf{TF}_{\boldsymbol{\theta}_{\textsf{ATTN}}^{(2l+1:2l+2)}}(\boldsymbol{H}^{(2l)})=\begin{bmatrix}
			\boldsymbol{z}_1 & \cdots & \boldsymbol{z}_n &\boldsymbol{z}_{n+1}\\ \boldsymbol{x}_1 & \cdots &\boldsymbol{x}_n &\boldsymbol{x}_{n+1}\\ y_1 &\cdots & y_n & 0\\ \boldsymbol{\Theta}_{:,1}^{(l+1)} & \cdots& \boldsymbol{\Theta}_{:,1}^{(l+1)} & \boldsymbol{\Theta}_{:,1}^{(l+1)}\\\vdots & \vdots &\vdots&\vdots \\\boldsymbol{\Theta}_{:,p}^{(l+1)} & \cdots &\boldsymbol{\Theta}_{:,p}^{(l+1)} & \boldsymbol{\Theta}_{:,p}^{(l+1)} \\\boldsymbol{\beta}^{(l+1)} &\cdots&\boldsymbol{\beta}^{(l+1)}&\boldsymbol{\beta}^{(l+1)}\\ \hat{\boldsymbol{x}}_1^{(l+1)} & \cdots&\hat{\boldsymbol{x}}_n^{(l+1)}& \hat{\boldsymbol{x}}_{n+1}^{(l+1)}\\
			1&\cdots &1&1\\1&\cdots &1&0
		\end{bmatrix}  \in \mathbb{R}^{D\times (n+1)}.
	\end{align*}
	\end{thm}
	Our existence proof specifies an attention structure such that one layer updates only the first-stage estimate $\hat{\boldsymbol{x}}_i^{(l)}$ for all samples, followed by another layer to update the parameters $\boldsymbol{\Theta}^{(l)}$ and $\boldsymbol{\beta}^{(l)}$. Furthermore, as noted in the proof of Theorem \ref{thm:transformer} (ref. Appendix \ref{proof:thm:transformer}), regardless of the initial values of $\boldsymbol{\Theta}^{(l)}, \boldsymbol{\beta}^{(l)}$ and $\hat{\boldsymbol{x}}^{(l)}$, the structures of the transformer blocks remain the same. This allows us to exploit the looped transformer architecture  to significantly reduce the number of parameters and improve learning efficiency \citep{yang2024loopedtransformersbetterlearning}. 

	By cascading the transformer block $\bar{L}$ times, with Theorem $\ref{thm:2SLS GD}$, one can show that transformers are able to mimic the \textsf{2SLS} estimator with exponential convergence rate, as described in the following corollary.

\begin{coro}[Implementing GD-\textsf{2SLS} with looped transformer]\label{coro:transformer}
    For any $0<\varepsilon<1$, given learning rates $\alpha,\eta,$ and $\Lambda\in(0,1)$, as defined in~\eqref{eq:lambda}, there exists a transformer formulated as $\textsf{TF}_{\boldsymbol{\theta}}(\cdot):=\textsf{TF}_{\boldsymbol{\theta}'}(\textsf{LTF}_{\bar{\boldsymbol{\theta}},\bar{L}}(\cdot))$, which consists of an $\bar{L}$-looped transformer $\textsf{LTF}_{\bar{\boldsymbol{\theta}},\bar{L}}$ with $\bar{\boldsymbol{\theta}}=\bar{\boldsymbol{\theta}}_{\textsf{ATTN}}^{(1:2)}= \{(\bar{\boldsymbol{Q}}_m^{(1:2)}, \bar{\boldsymbol{K}}_m^{(1:2)}, \bar{\boldsymbol{V}}_m^{(1:2)})\}_{m\in[\bar{M}^{(1:2)}]}\subset \mathbb{R}^{D\times D}$, $\bar{L}=\lceil\mathcal{O}(\log_{\Lambda}(\varepsilon))\rceil$, and a final attention layer\footnote{This layer updates the prediction $\hat{y}_{n+1}:=\boldsymbol{\beta}^{(\bar{L})\top}\boldsymbol{x}_{n+1}$, which can be constructed with 2 attention heads using the same architecture as \citet[Theorem 13]{statistician}} $\boldsymbol{\theta}'=\boldsymbol{\theta}_{\textsf{ATTN}}'= \{(\boldsymbol{Q}_m', \boldsymbol{K}_m', \boldsymbol{V}_m')\}_{m\in[M']}\subset \mathbb{R}^{D\times D}$, where $\bar{M}^{(1)}=2p, \bar{M}^{(2)}=2(p+1)$, $M'=2$, such that given embedded input $\boldsymbol{H}^{(0)}$ taking the format in \eqref{eq:transformed input}, the model output satisfies: 
    \begin{align*}
        |\textsf{read}_y(\textsf{TF}_{\boldsymbol{\theta}}(\boldsymbol{H}^{(0)}))-\hat{\boldsymbol{\beta}}_{\textsf{2SLS}}^\top\boldsymbol{x}_{n+1}|\leq B_x\varepsilon,
    \end{align*}
	where $\textsf{read}_y(\cdot)$ is a function that reads the prediction $\hat{y}_{n+1}$ from the output of the transformer.
\end{coro}

We emphasize here that our construction differs from the implementation of \citet[Theorem 4]{statistician} for \textsf{OLS} in the following aspects:
\begin{itemize}
    \item [i.] We apply the square loss as defined in \eqref{eq:IV loss objective} to learn the \textsf{2SLS} estimator, which simplifies the loss function's sum-of-ReLU representation.
    \item [ii.] The dimension of the input embedding is $D=qp+3p+q+3$, where the extra dimensions store the vectorized parameters $\boldsymbol{\Theta}^{(l)}, \boldsymbol{\beta}^{(l)}$, and the first stage estimate $\hat{\boldsymbol{x}}^{(l)}$. 
    \item [iii.] We use a two-layer attention-only transformer block $\bar{\boldsymbol{\theta}}$ to implement a \textsf{2SLS} GD update \eqref{eq:2SLS GD Update}, with the first layer to update the current first-stage estimate $\hat{\boldsymbol{x}}^{(l)}$, and the second layer to update the parameters $\boldsymbol{\Theta}^{(l)}$ and $\boldsymbol{\beta}^{(l)}$.
    \item [iv.] For each transformer block, in the first layer, we equip $2$ heads to update each dimension of $\hat{\boldsymbol{x}}_i^{(l)}\in\mathbb{R}^p$ for all samples. In the second layer, we equip 2 heads to update each column of $\boldsymbol{\Theta}^{(l)}\in\mathbb{R}^{q\times p}$ and $\boldsymbol{\beta}^{(l)}\in\mathbb{R}^p$.
\end{itemize}

\subsection{Pretraining and Excess Loss Bound}\label{sec:pretraining}

With slightly abuse of notations, we denote the (formulated) training prompt as:
\begin{align*}
    \boldsymbol{H}_{k}=\begin{bmatrix}
        \boldsymbol{z}_{1,k} & \cdots & \boldsymbol{z}_{n,k} &\boldsymbol{z}_{n+1,k}\\ \boldsymbol{x}_{1,k} & \cdots &\boldsymbol{x}_{n,k} &\boldsymbol{x}_{n+1,k}\\ 
        y_{1,k} &\cdots & y_{n,k} & 0
    \end{bmatrix}  \in \mathbb{R}^{(p+q+1)\times (n+1)}, \quad k=1,\ldots,N.
\end{align*}
Note that we denote each training prompt by the subscript $k=1,\ldots,N$, where $N$ is the total number of prompts. Each training prompt consists of $n$ labeled training samples $\{(\boldsymbol{z}_i,\boldsymbol{x}_i,y_i)\}_{i=1}^{n}$, and one unlabeled query sample $(\boldsymbol{z}_{n+1}, \boldsymbol{x}_{n+1})$. Our goal is to predict $y_{n+1}$ given the context provided by the prompt.

We introduce the following ICL data generating scheme such that endogeneity occurs in the training samples, but does not extend to the query sample. Each training prompt is generated by the in-context distribution $\boldsymbol{\mathcal{P}}$, described by Algorithm \ref{alg:data generating process}.

\begin{algorithm}[h]
    \caption{In-Context Distribution $\boldsymbol{\mathcal{P}}$} \label{alg:data generating process}
    \begin{algorithmic}[1]
    \STATE \textbf{Parameters:} Sample size n, clipping thresholds $B_z, B_x, B_y$. Task parameters $\boldsymbol{\Theta}, \boldsymbol{\beta}, \boldsymbol{\Phi}, \boldsymbol{\phi}$, $\boldsymbol{\Sigma}_{z}, \boldsymbol{\Sigma}_{u}, \boldsymbol{\Sigma}_{\omega}$, $\sigma_{\epsilon}$  from meta distribution $\boldsymbol{\pi}$.
    \STATE \textbf{Output:} Training samples $\{(\boldsymbol{z}_{i},\boldsymbol{x}_{i},y_{i})\}_{i=1}^{n}$, query sample $(\boldsymbol{z}_{n+1}, \boldsymbol{x}_{n+1}, \boldsymbol{y}_{n+1})$.
    \FOR{$i = 1, \ldots, n$}
    \STATE \textbf{Generate:} $\boldsymbol{z}_{i} \sim \mathcal{N}(0, \boldsymbol{\Sigma}_{z})$, $\boldsymbol{u}_{i} \sim \mathcal{N}(0, \boldsymbol{\Sigma}_{u})$, $\boldsymbol{\omega}_{i} \sim \mathcal{N}(0, \boldsymbol{\Sigma}_{\omega})$, $\epsilon_{i} \sim \mathcal{N}(0, \sigma_{\epsilon}^2)$.
    \STATE \textbf{Compute:} $\boldsymbol{x}_{i} = \boldsymbol{\Theta}^\top \boldsymbol{z}_{i} + \boldsymbol{\Phi}^\top \boldsymbol{u}_{i} + \boldsymbol{\omega}_{i}$.
    \STATE \textbf{Compute:} $y_{i} = \boldsymbol{\beta}^\top \boldsymbol{x}_{i} + \boldsymbol{\phi}^\top \boldsymbol{u}_{i} + \epsilon_{i}$.
    \ENDFOR
    \STATE \textbf{Generate:} $\boldsymbol{z}_{{n+1}} \sim \mathcal{N}(0, \boldsymbol{\Sigma}_{z})$, $\boldsymbol{\omega}_{{n+1}} \sim \mathcal{N}(0, \boldsymbol{\Sigma}_{\omega})$, $\epsilon_{{n+1}} \sim \mathcal{N}(0, \sigma_{\epsilon}^2)$.
    \STATE \textbf{Compute:} $\boldsymbol{x}_{{n+1}} = \boldsymbol{\Theta}^\top \boldsymbol{z}_{{n+1}} + \boldsymbol{\omega}_{{n+1}}$.
    \STATE \textbf{Compute:} $y_{{n+1}} = \boldsymbol{\beta}^\top \boldsymbol{x}_{{n+1}} + \epsilon_{{n+1}}$.
    \STATE \textbf{Clip:} $\boldsymbol{z}_{i}=\textsf{clip}_{B_z}(\boldsymbol{z}_{i})$, $\boldsymbol{x}_{i}=\textsf{clip}_{B_x}(\boldsymbol{x}_{i})$, $y_{i} = \textsf{clip}_{B_y}(y_{i})$ for $i=1,\ldots,n+1$.
\end{algorithmic}
\end{algorithm}

In Algorithm \ref{alg:data generating process}, $\boldsymbol{u}\in\mathbb{R}^{p}$ is the source of endogenous error, $\boldsymbol{w}\in\mathbb{R}^p, {\epsilon}\in\mathbb{R}$ are the exogenous errors. Note that we have $\epsilon_{1,i}=\boldsymbol{\phi}^\top \boldsymbol{u}_{i} + \epsilon_{i}$ and $\boldsymbol{\epsilon}_{2,i}= \boldsymbol{\Phi}^\top \boldsymbol{u}_{i} + \boldsymbol{\omega}_{i}$, corresponding to the notations in \eqref{eq:endogeneity}. $\boldsymbol{\Theta}\in\mathbb{R}^{q\times p}, \boldsymbol{\beta}\in\mathbb{R}^{p}, \boldsymbol{\Phi}\in\mathbb{R}^{p\times p}, \boldsymbol{\phi}\in\mathbb{R}^{p}, \boldsymbol{\Sigma}_z\in\mathbb{R}^{q\times q}, \boldsymbol{\Sigma}_u\in\mathbb{R}^{p\times p}, \boldsymbol{\Sigma}_\omega\in\mathbb{R}^{p\times p}, \sigma_{\epsilon}\in\mathbb{R}$ are task-specific parameters following meta distribution $\boldsymbol{\pi}$. $\textsf{clip}_{B}(\cdot)$ is a clipping operator to bound the norm of input within radius $B$. We say that the in-context samples $\{(\boldsymbol{z}_i,\boldsymbol{x}_i,y_i)\}_{i=1}^{n+1}$ are drawn from the in-context distribution $\boldsymbol{\mathcal{P}}$, and $\boldsymbol{\mathcal{P}}\sim\boldsymbol{\pi}$ if the task parameters $(\boldsymbol{\Theta}, \boldsymbol{\beta}, \boldsymbol{\Phi}, \boldsymbol{\phi}, \boldsymbol{\Sigma}_z, \boldsymbol{\Sigma}_u, \boldsymbol{\Sigma}_\omega, \sigma_{\epsilon})$ are sampled from $\boldsymbol{\pi}$.
One can check that Assumption \ref{iv assumption} and Assumption \ref{regularity assumption}(ii)(iii) are directly satisfied with the data generated from the in-context distribution $\boldsymbol{\mathcal{P}}$.

Following the theoretical framework of \citep{statistician}, we define the population ICL loss\footnote{All the clipping operations are only for analytical purpose. In practice, the behavior of the trained transformer is consistent even without the clipping bounds.}:   
    \begin{align} \label{def:population icl loss}
        L_{\textsf{ICL}}(\boldsymbol{\theta})=\mathbb{E}_\pi\mathbb{E}_{\mathcal{P}}[y_{n+1}-\textsf{clip}_{B_y}(\textsf{read}_y(\textsf{TF}_{\boldsymbol{\theta}}^R(\boldsymbol{H}^{(0)})))]^2,
    \end{align}
where $\boldsymbol{H}^{(0)}$ is the embedded input as defined in \eqref{eq:transformed input}, $\textsf{TF}_{\boldsymbol{\theta}}^R$ is the transformer model with parameter $\boldsymbol{\theta}$ and clipping operation $\textsf{clip}_R(\cdot)$ applied to each layer output. For simplicity, we denote $\widetilde{\textsf{TF}}_{\theta}(\boldsymbol{H}):=\textsf{clip}_{B_y}(\textsf{read}_y(\textsf{TF}_{\boldsymbol{\theta}}^R(\boldsymbol{H}^{(0)})))$.

The transformer is trained to minimize the in-context loss in \eqref{def:population icl loss} with the following empirical loss:
\begin{align} \label{def:empirical icl loss}
    \hat{L}_{\textsf{ICL}}(\boldsymbol{\theta})=\frac{1}{N}\sum_{k=1}^N(y_{n+1,k}-\widetilde{\textsf{TF}}_{\boldsymbol{\theta}}(\boldsymbol{H}_{k}))^2.
\end{align}

We consider the following constrained optimization problem:
\begin{align}\label{def:trainedtf}
    \begin{split}
        \hat{\boldsymbol{\theta}}&:=\underset{\boldsymbol{\theta}\in\boldsymbol{\vartheta}_{L,M,D',B_{\theta}}}{\arg\min}\hat{L}_{\textsf{ICL}}(\boldsymbol{\theta}),\\
        \boldsymbol{\vartheta}_{L,M,D',B_\theta}:=\{\boldsymbol{\theta}=(\boldsymbol{\theta}_{\textsf{Attn}}^{(1:L)}, \boldsymbol{\theta}_{\textsf{MLP}}^{(1:L)}&):\underset{l\in[L]}{\max}\: M^{(l)}\leq M, \:\underset{l\in[L]}{\max}\:D^{(l)}\leq D',\: \left|\!\left|\!\left|\boldsymbol{\theta}\right|\!\right|\!\right|\leq B_\theta\},
     \end{split}
\end{align}
where $\left|\!\left|\!\left|\boldsymbol{\theta}\right|\!\right|\!\right|:=\underset{l\in[L]}{\max}\{\underset{m\in[M]}{\max}\{\|\boldsymbol{Q}_m^{(l)}\|, \|\boldsymbol{K}_m^{(l)}\|\}+\sum_{m=1}^M\|\boldsymbol{V}_m^{(l)}\|+\|\boldsymbol{W}_1^{(l)}\|+\|\boldsymbol{W}_2^{(l)}\|\}$.

We now establish excess loss bound for the trained transformer model.
\begin{thm}[Excess loss bound for in-context pretrained transformer]\label{thm:excessloss}
    Suppose condition (i) in Assumption~\ref{regularity assumption} holds and the meta distribution $\boldsymbol{\pi}$ satisfies the following conditions:
    \begin{align}\label{eq:meta distribution condition}
		\mathbb{E}_\pi\left[\boldsymbol{\phi}^\top\boldsymbol{\Sigma}_u\boldsymbol{\phi}+\sigma_\epsilon^2\right]\leq& \tilde{\sigma}^2 \text{ and } \mathbb{E}_\pi\left[\sigma_\epsilon^2\right]\leq\tilde{\sigma}_\epsilon^2.
    \end{align}
    Let the in-context distribution $\boldsymbol{\mathcal{P}}\sim\boldsymbol{\pi}$ such that the samples $(\boldsymbol{z}_i, \boldsymbol{x}_i, y_i)_{i=1}^{n+1}$ are drawn independently from $\boldsymbol{\mathcal{P}}$ (ref. Algorithm \ref{alg:data generating process}). With training prompts $\boldsymbol{H}_k, k=1,\ldots,N$, under ICL loss in \eqref{def:population icl loss}, the trained transformer in \eqref{def:trainedtf} with $L=2\bar{L}+1, M=2(p+1), D=qp+3p+q+3, D'=0$ (attention-only) achieves the following excess loss with probability at least $1-\zeta$:
    \begin{align*}
        L_{\textsf{ICL}}(\hat{\boldsymbol{\theta}}) &- \mathbb{E}_{\pi}\mathbb{E}_{\mathcal{P}}\left[(y_{n+1}-\langle\boldsymbol{\beta},\boldsymbol{x}_{n+1}\rangle)^2\right] \leq  \mathcal{O}\bigg((\Lambda^\star)^{\bar{L}}\left(B_x^2\sqrt{\frac{q}{n}\left(\frac{B_\beta^2}{K}+C^2(n)\tilde{\sigma}^2\right)}+B_x\tilde{\sigma}_\epsilon\right)\nonumber\\
        &+B_x^2\left(\frac{q}{n}\left(\frac{B_\beta^2}{K}+C^2(n)\tilde{\sigma}^2\right)+\mu_{\Lambda,2}^\star\right)+B_y^2\sqrt{\frac{L^2MD^2\log(2+\max\{B_\theta,R,B_y\})+\log(1/\zeta)}{N}}\bigg),
    \end{align*}
    where $\Lambda^\star:=\underset{\alpha,\eta}{\min}\;\mathbb{E}_\pi\mathbb{E}_{\mathcal{P}}[\Lambda|\boldsymbol{H},\alpha,\eta]<1$, and $\mu_{\Lambda,2}^\star:=\mathbb{E}_\pi\mathbb{E}_{\mathcal{P}}[\Lambda^{2\bar{L}}|\boldsymbol{H},\alpha^\star,\eta^\star]$ is close to $0$.
\end{thm}
In practical training, the number of prompts $N$ is usually large enough such that the last term of the above bound is negligible. Thus, given a meta distribution $\boldsymbol{\pi}$, the excess loss is dominated by two factors: (i) number of attention layers, and (ii) number of in-context samples. The proof of Theorem \ref{thm:excessloss} is provided in Appendix \ref{proof:thm:excessloss}. 

\subsection{Extracting the regression coefficients}
The primary goal of IV regression is to estimate the causal effect, i.e. the coefficient $\boldsymbol{\beta}$ under the stated endogeneity in \eqref{eq:endogeneity}. For \textsf{2SLS}, the estimated causal effect is given by the coefficients of the endogenous variable in the second stage regression (\ref{eq:2SLS second stage}). For transformer models, we propose a straightforward method to extract these estimated coefficients by differentiating the output with respect to each dimension of the endogenous variable. The specific approach is summarized in Algorithm \ref{alg:extract}.
\begin{algorithm}[t]
\caption{Extracting the regression coefficients}
\label{alg:extract}
\begin{algorithmic}[1]
\STATE \textbf{Input:} Trained transformer model $\textsf{TF}_{\hat{\boldsymbol{\theta}}}$, input matrix $\boldsymbol{H}$, perturbation $\Delta$.
\STATE \textbf{Output:} Estimated coefficient $\hat{\boldsymbol{\beta}}$.
\STATE \textbf{Procedure:}
\STATE Compute the output of the transformer model: $\hat{\boldsymbol{Y}} = \widetilde{\textsf{TF}}_{\hat{\boldsymbol{\theta}}}(\boldsymbol{H})$.
\FOR{each dimension $k=1,\ldots,p$}
    \STATE Copy $\boldsymbol{H}_{\Delta(k)} = \boldsymbol{H}$. Set the $k$-th dimension of $\boldsymbol{x}_{n+1}$ to be $(\boldsymbol{x}_{n+1})_k+\Delta$ for $\boldsymbol{H}_{\Delta(k)}$.
    \STATE Compute the new output value: $\hat{\boldsymbol{Y}}_{\Delta(k)} = \widetilde{\textsf{TF}}_{\hat{\boldsymbol{\theta}}}(\boldsymbol{H}_{\Delta(k)})$.
    \STATE Compute the estimated coefficient: $\hat{\beta}_k = \frac{\hat{\boldsymbol{Y}}_{\Delta(k)} - \hat{\boldsymbol{Y}}}{\Delta}$.
\ENDFOR
\end{algorithmic}
\end{algorithm}
We observe that the choice of $\Delta$ within a reasonable range does not significantly affect the estimation of the coefficients. In practice, usually a slightly larger $\Delta$ (for example $\Delta=5$) can lead to a more stable estimation, which is possibly due to the elimination of rounding errors during computation.

\section{Experiments}\label{sec:simulation}
\subsection{Experiment Setup}
We conduct a simulation study to evaluate the performance of the ICL-pretrained transformer model in handling endogeneity. We set the maximum input sample size to 51 ($n=50$ training samples and one query sample), the dimension of endogenous variable $p=5$, and the dimension of instrument $q=10$. The training prompts are generated using Algorithm \ref{alg:data generating process}, with task parameters $\boldsymbol{\Theta}, \boldsymbol{\beta}, \boldsymbol{\Phi}, \boldsymbol{\phi}$ sampled from standard Gaussian distribution, and the covariance matrices $\boldsymbol{\Sigma}_z, \boldsymbol{\Sigma}_u, \boldsymbol{\Sigma}_\omega$ set to be identity matrices. The noise level $\sigma_\epsilon$ is set to 1. We ignore all the clipping bounds in the experiment ($B_\beta, B_\Theta, B_z, B_x, B_y, B_\theta, R$ set to infinity). 

The backbone of the transformer block is initialized using GPT-2 settings, with 12 attention heads $(M=12)$, 80-dimensional embedding space $(D=80)$ and 2 layers $(L_0=2)$, following the theoretical guidelines in Theorem \ref{thm:transformer}. We employ the looped transformer architecture, consisting of 10 identical cascading transformer blocks. The transformer model is trained under the ICL loss (\ref{def:empirical icl loss}) with a batch size of $N=64$, over a total of 300,000 training steps.

We evaluate the trained transformer model on test prompts that are not included during training. As benchmarks, we compare the transformer's performance against the \textsf{2SLS} and the \textsf{OLS} estimators, which are obtained by directly fitting the training samples $\{(\boldsymbol{z}_{i},\boldsymbol{x}_i,y_i)\}_{i=1}^n$ within the text prompts. In contrast, the same trained transformer model is used without any parameter adjustments for each task. We compare the performance of these models from two aspects: the in-context prediction error (ICPE) on the query sample $y_{n+1}$, and the mean squared error (MSE) on the coefficient $\boldsymbol{\beta}$. 

\subsection{Results}\label{sec:sim results}

We first investigate the performance of the trained transformer model over endogeneity tasks with varying training sample sizes from 20 to 50. The results are shown in Figure \ref{fig:mse_vs_N}. Under endogeneity, our transformer model achieves similar performance to that of the \textsf{2SLS} estimator, with only small gaps in ICPE and MSE, both outperforming the \textsf{OLS} estimator. 

Next, we examine the performance of the trained transformer model in handling varying levels of IV strength. The strength of an instrument is measured by the correlation between the IV and the endogenous variable. To vary the IV strength, we generate prompts with $\boldsymbol{z}_{i}$ and $\boldsymbol{x}_{i}$ following different correlation levels. Specifically, in Algorithm \ref{alg:data generating process}, we adjust the IV strength by multiplying $\boldsymbol{\Theta}$ by a factor $r\in(0,2)$ when generating test prompts. The results are shown in Figure \ref{fig:mse_vs_ivstrength}.

Interestingly, the trained transformer model outperforms the \textsf{2SLS} estimator in handling weaker IVs (when IV strength $< 0.5$). This suggests that, beyond merely mimicking \textsf{2SLS}, the ICL training process may equip the transformer model with a more advanced mechanism for handling endogeneity with weak IVs than the \textsf{2SLS} estimator. At the same time, when the IV is strong, the transformer model maintains performance comparable to that of the \textsf{2SLS} estimator.

\begin{figure}[h]
    \centering
    \begin{subfigure}[b]{0.48\textwidth}
        \includegraphics[width=\textwidth]{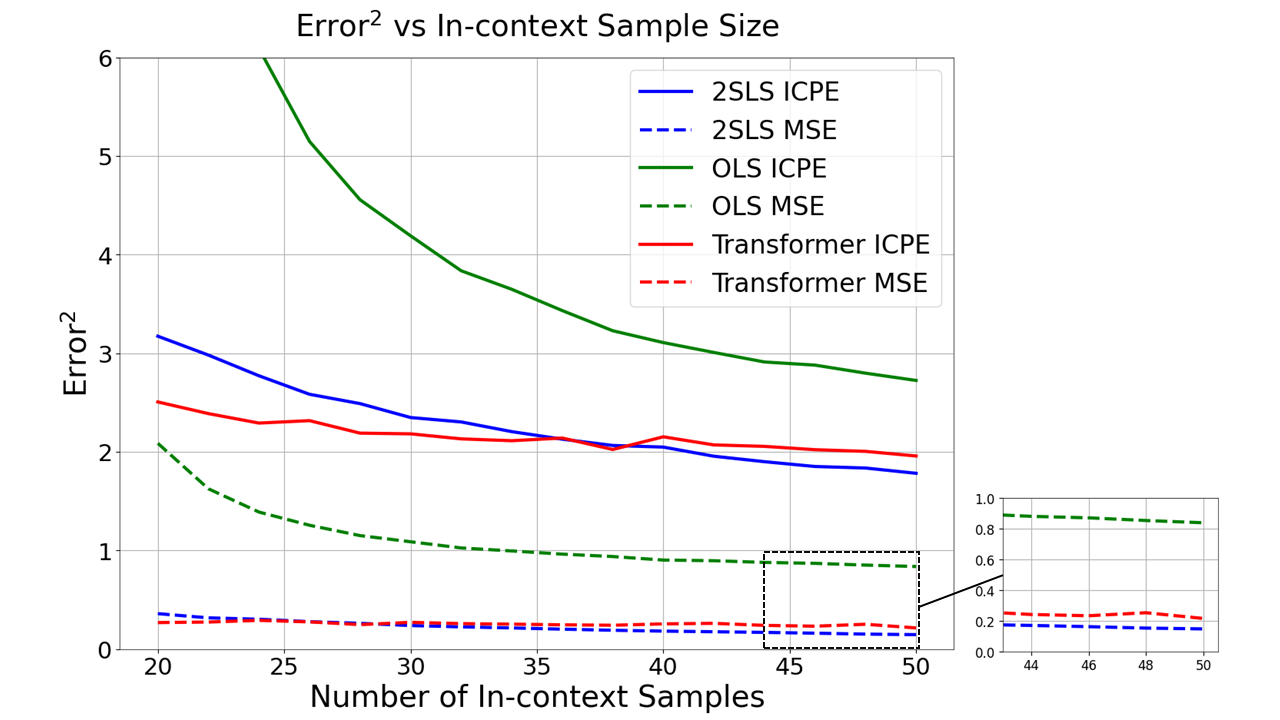}
        \caption{}
        \label{fig:mse_vs_N}
    \end{subfigure}
    \begin{subfigure}[b]{0.48\textwidth}
        \includegraphics[width=\textwidth]{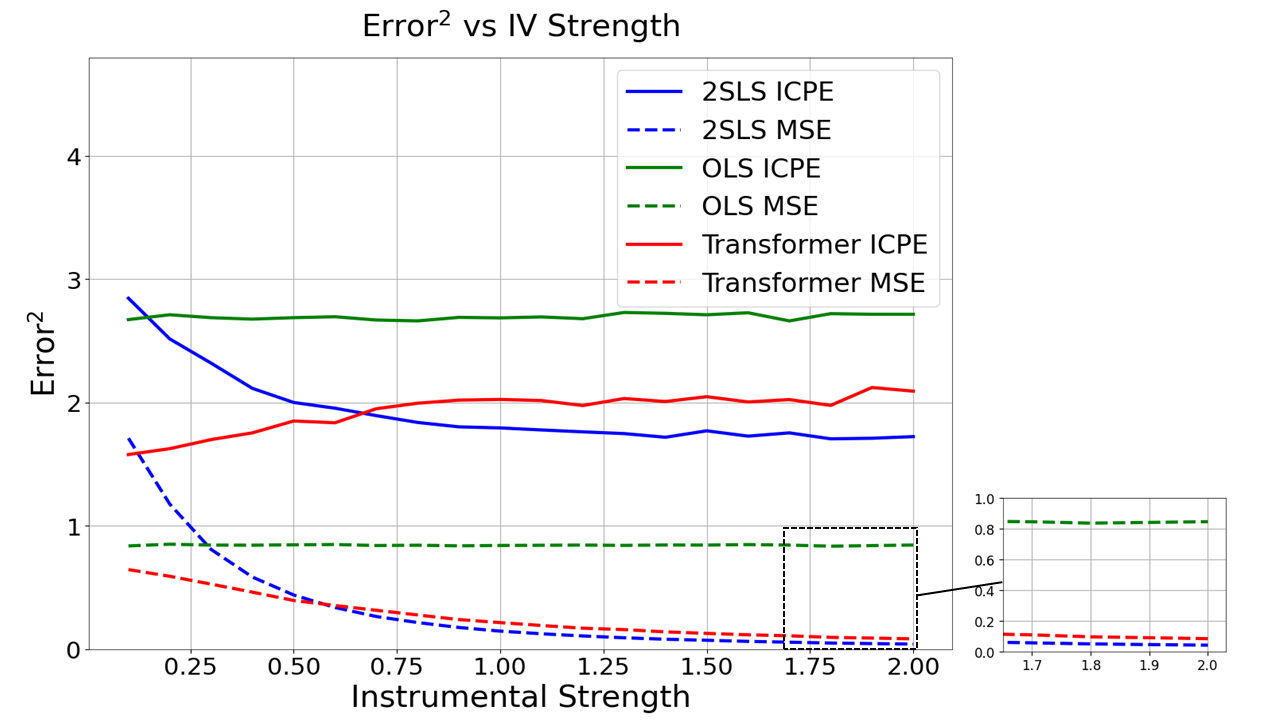}
        \caption{}
        \label{fig:mse_vs_ivstrength}
    \end{subfigure}
    \caption{The ICL performance of the trained transformer model in endogeneity tasks. We compare in-context prediction error (ICPE) and coefficient MSE versus (a) the number of in-context samples; (b) the IV strength. The curves are averaged over 500 simulations.}
\end{figure}

This finding motivates us to further examine the performance of the trained transformer model in non-standard endogeneity tasks. We consider two scenarios: (a) the IV has a quadratic effect on the endogenous variable, i.e. $\boldsymbol{x}_{i,k}=\boldsymbol{\Theta}_k^\top \boldsymbol{z}_{i,k}^2+\textsf{error}_{i,k}$ in Algorithm \ref{alg:data generating process}, and (b) the dimension of IV is not sufficient to identify the endogenous variable\footnote{For \textsf{2SLS} estimate, the actual computation uses pseudoinverse to handle rank deficiency.}, where we set $q=3$ (by zeroing out the remaining dimensions of $\boldsymbol{z}$ in test prompts) and $p=5$. 

We evaluate the same trained transformer model as before, with results presented in Figure \ref{fig:mse_vs_N_quadratic} and Figure \ref{fig:mse_vs_ivstrength_underidentify}, respectively. Once again, the trained transformer model consistently outperforms both \textsf{2SLS} and \textsf{OLS} estimators in handling these non-standard endogeneity tasks. All these results suggest that the trained transformer can be generalized effectively to a broader range of endogeneity tasks while still providing reliable in-context predictions and coefficient estimates. To further illustrate this capability, we also examine other cases including multicollinearity, complex non-linear IV, and varying endogeneity strengths, see Appendix \ref{sec:exp multicollinearity},\ref{sec:exp complex non-linear},\ref{sec:exp vary endogeneity}. We suspect that, in our pretraining scheme, although the \textsf{2SLS} estimator already achieves small excess loss, a gap remains between the \textsf{2SLS} estimator and the optimal predictor that the transformer model successfully bridges. Finally, we conclude that through ICL training, the transformer model performs at least as well as \textsf{2SLS} and appears to be a promising tool for handling endogeneity in difficult scenarios.

\begin{figure}[h]
    \centering
    \begin{subfigure}[b]{0.48\textwidth}
        \includegraphics[width=\textwidth]{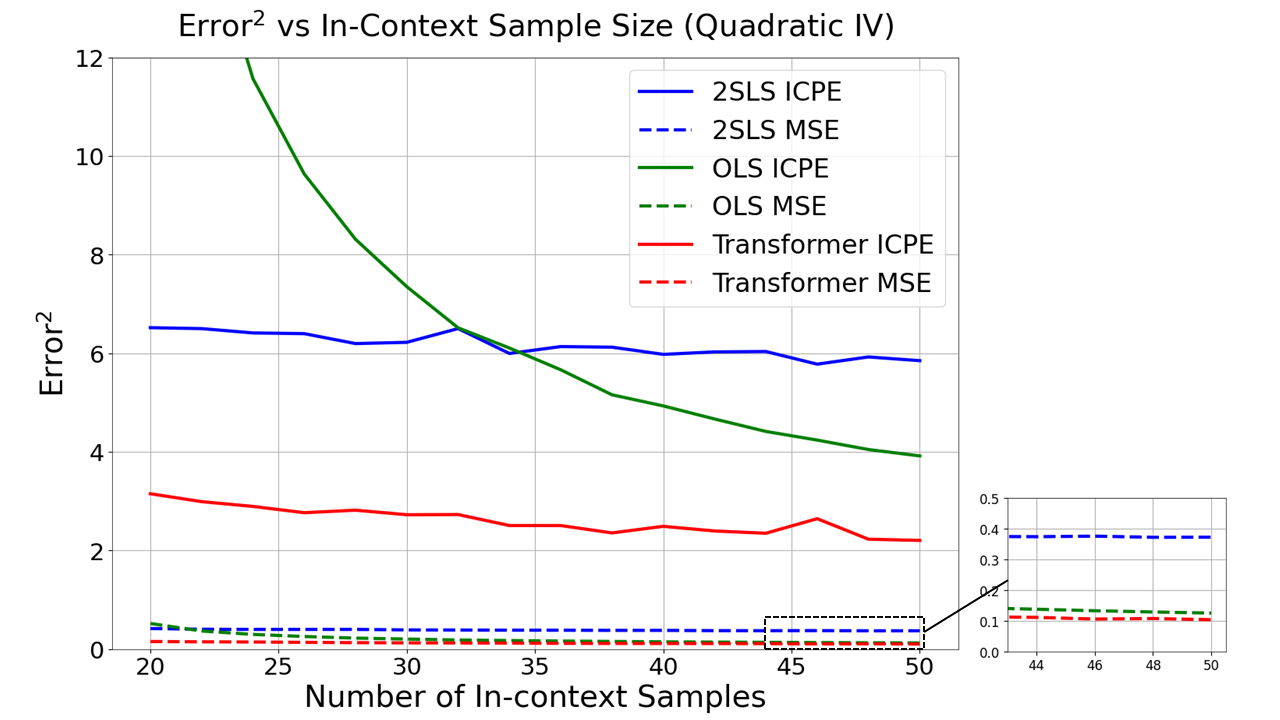}
        \caption{}
        \label{fig:mse_vs_N_quadratic}
    \end{subfigure}
    \begin{subfigure}[b]{0.48\textwidth}
        \includegraphics[width=\textwidth]{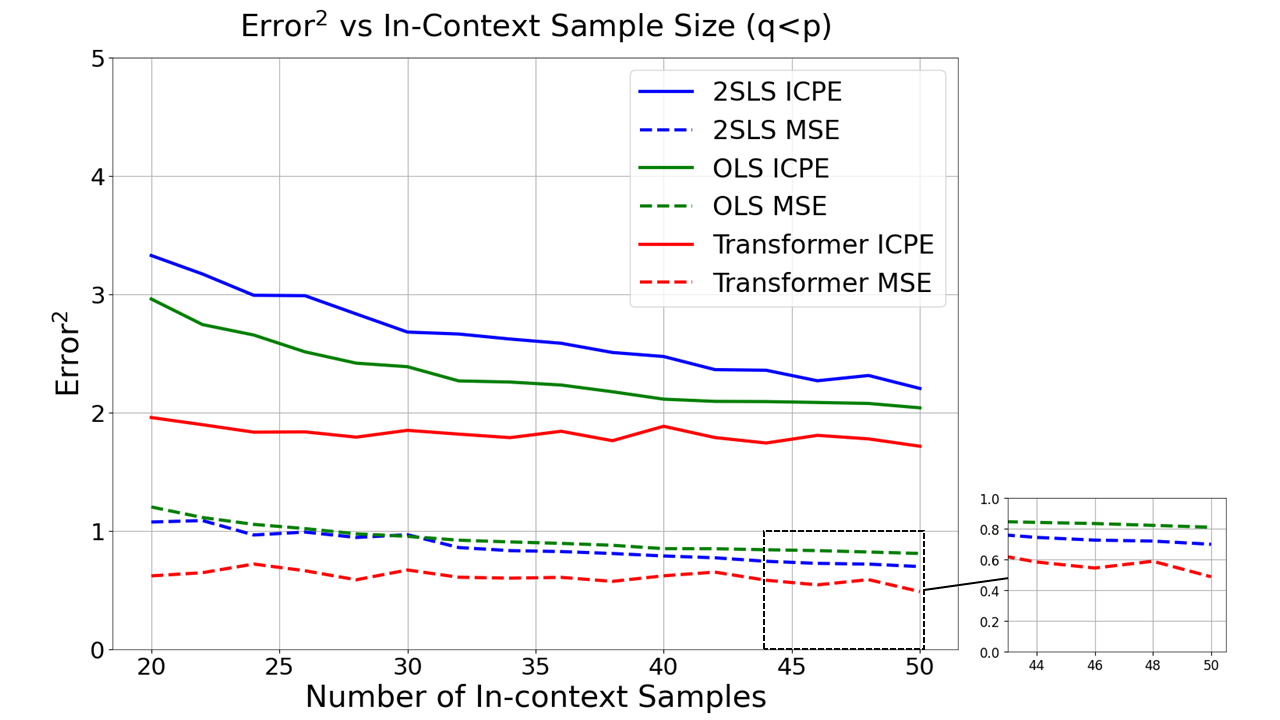}
        \caption{}
        \label{fig:mse_vs_ivstrength_underidentify}
    \end{subfigure}
    \caption{The ICL performance of the trained transformer model in non-standard endogeneity tasks: (a) The IV has quadratic effect on the endogenous variable; (b) The dimension of IV is not sufficient to identify the endogenous variable. The curves are averaged over 500 simulations.}
\end{figure}

\section{Conclusion}
This paper presents a novel perspective on the transformer model in its ability to handle endogeneity in in-context linear regression. We have theoretically shown that the transformer model exists an intrinsic structure that enables it to learn the \textsf{2SLS} algorithm through an efficient GD procedure. We have further provided a theoretical guarantee that the trained transformer model can achieve a small excess loss over the optimal loss, under our proposed ICL training scheme. Our simulation study demonstrates that the trained transformer model can achieve comparable performance to the \textsf{2SLS} estimator in handling standard endogeneity tasks. Furthermore, our investigation illustrates that it exhibits significantly better performances in handling complex scenarios such as weak instruments, non-linear IV, and underdetermined IV problems, compared to the \textsf{2SLS} estimator. These results suggest that the ICL pre-trained transformer model is a promising tool for making reliable in-context predictions and coefficient estimates under endogeneity, especially when dealing with non-standard IV problems.

\section*{Acknowledgements}
The work of H. Liang and L. Lai was supported by National Science Foundation (NSF) under grants CCF-2112504 and CCF-2232907. The work of K. Balasubramanian was supported by NSF under grants DMS-2413426 and DMS-2053918.

\bibliographystyle{unsrtnat}
\bibliography{references}  

\appendix
\section{Proofs For Section \ref{subsec:IV Regression}} 
\subsection{Proof of Theorem \ref{thm:consistency}} \label{proof:thm:consistency}
We first introduce the following lemmas that are used in the proof of Theorem \ref{thm:consistency}.
\begin{lemma}[Bernstein Inequality, from Theorem 6.1.1 in \cite{tropp2015introductionmatrixconcentrationinequalities}]\label{lemma:bernstein}
    Let $\boldsymbol{S}_1,\ldots,\boldsymbol{S}_n$ be independent, centered random matrices with common dimension $d_1\times d_2$, and assume that each one is almost surely bounded:
    \begin{align*}
        \mathbb{E}[\boldsymbol{S}_i]=\boldsymbol{0}, \mathbb{P}(\left\|\boldsymbol{S}_{i}\right\|\leq b)=1, \quad\forall i=1,\ldots,n.
    \end{align*}
    With the sum:
    \begin{align*}
        \boldsymbol{\Omega}=\sum_{i=1}^n\boldsymbol{S}_i,
    \end{align*}
    and the matrix variance statistic of the sum:
    \begin{align*}
        \nu(\boldsymbol{\Omega}):=\max\left\{\left\|\mathbb{E}(\boldsymbol{\Omega}\boldsymbol{\Omega}^\top)\right\|,\left\|\mathbb{E}(\boldsymbol{\Omega}^\top\boldsymbol{\Omega})\right\|\right\},
    \end{align*}
    then the following inequality holds:
    \begin{align*}
        \mathbb{P}\{\left\|\boldsymbol{\Omega}\right\| \geq \varepsilon\} \leq (d_1 + d_2) \cdot \exp \left( \frac{-\varepsilon^2 / 2}{\nu(\boldsymbol{\Omega}) + b\varepsilon / 3} \right) \text{ for any } \varepsilon \geq 0.
    \end{align*}
\end{lemma}

\begin{lemma}[Inverse Convergence, adapted from Lemma 2.1 in \cite{jin2024metalearning}]\label{lemma:jin}
	Suppose we have a random invertible matrix $\boldsymbol{\Omega}$ and invertible matrix sequence $\{\hat{\boldsymbol{\Omega}}^{(n)}\}$ such that $\hat{\boldsymbol{\Omega}}^{(n)}\convp \boldsymbol{\Omega}$. If there exists a constant  $\tilde{\lambda}>0$  such that $\sigma_{\min}(\hat{\boldsymbol{\Omega}})\geq \tilde{\lambda}$ almost surely, then it holds that:
	\begin{align*}
		(\hat{\boldsymbol{\Omega}}^{(n)})^{-1}\convp \boldsymbol{\Omega}^{-1}.
	\end{align*}
	Further, given convergence rate
	\begin{align*}
		\mathbb{P}\left\{\left\|\hat{\boldsymbol{\Omega}}^{(n)} - \boldsymbol{\Omega}\right\|\geq\varepsilon\right\} \leq\xi(n,\varepsilon),
	\end{align*} then:
	\begin{align*}
		\mathbb{P}\left\{\left\|(\hat{\boldsymbol{\Omega}}^{(n)})^{-1} - \boldsymbol{\Omega}^{-1}\right\|\geq\varepsilon\right\} \leq \xi(n,\tilde{\lambda}^2\varepsilon).
	\end{align*}
\end{lemma}
\begin{proof}
    We have the following decomposition:
    \begin{align*}
        (\hat{\boldsymbol{\Omega}}^{(n)})^{-1} - \boldsymbol{\Omega}^{-1} = (\hat{\boldsymbol{\Omega}}^{(n)})^{-1}(\boldsymbol{\Omega} - \hat{\boldsymbol{\Omega}}^{(n)})\boldsymbol{\Omega}^{-1}.
    \end{align*}
    It follows that:
    \begin{align*}
        \left\|(\hat{\boldsymbol{\Omega}}^{(n)})^{-1}- \boldsymbol{\Omega}^{-1}\right\| &\leq \left\|(\hat{\boldsymbol{\Omega}}^{(n)})^{-1}\right\|\left\|\boldsymbol{\Omega} - \hat{\boldsymbol{\Omega}}^{(n)}\right\|\left\|\boldsymbol{\Omega}^{-1}\right\|\\
        &\leq \frac{1}{\tilde{\lambda}^2}\left\|\boldsymbol{\Omega} - \hat{\boldsymbol{\Omega}}^{(n)}\right\|.
    \end{align*}
    Then
    \begin{align*}
        \begin{split}
            \mathbb{P}\left\{\left\|(\hat{\boldsymbol{\Omega}}^{(n)})^{-1} - \boldsymbol{\Omega}^{-1}\right\|\geq\varepsilon\right\} &\leq \mathbb{P}\left\{\frac{1}{\tilde{\lambda}^2}\left\|\boldsymbol{\Omega} - \hat{\boldsymbol{\Omega}}^{(n)}\right\|\geq\varepsilon\right\}\\
            &\leq \xi(n,\tilde{\lambda}^2\varepsilon).
        \end{split}
    \end{align*}
\end{proof}

\begin{lemma}[Product Convergence]\label{lemma:productconvergence}
    Let $\{\hat{\boldsymbol{\Omega}}_1^{(n)}\}, \{\hat{\boldsymbol{\Omega}}_2^{(n)}\},\ldots,\{\hat{\boldsymbol{\Omega}}_K^{(n)}\}$ be $K$ sequences of matrices such that $\hat{\boldsymbol{\Omega}}_1^{(n)} \convp \boldsymbol{\Omega}_1$, $\hat{\boldsymbol{\Omega}}_2^{(n)} \convp \boldsymbol{\Omega}_2, \ldots, \hat{\boldsymbol{\Omega}}_K^{(n)}\convp\boldsymbol{\Omega}_K$, where each $\|\hat{\boldsymbol{\Omega}}_k^{(n)}\|$ is almost surely bounded for every $k=1,\ldots,K$. If the dimensions match, then it holds that:
    \begin{align*}
        \hat{\boldsymbol{\Omega}}_1^{(n)}\hat{\boldsymbol{\Omega}}_2^{(n)}\cdots\hat{\boldsymbol{\Omega}}_K^{(n)} \convp \boldsymbol{\Omega}_1\boldsymbol{\Omega}_2\cdots\boldsymbol{\Omega}_K.
    \end{align*} 
    
    Further, given convergence rates:
    \begin{align*}
        \begin{split}
            \mathbb{P}&\left\{\left\|\hat{\boldsymbol{\Omega}}_1^{(n)} - \boldsymbol{\Omega}_1\right\| \geq \varepsilon\right\} \leq \xi_1(n, \varepsilon), \\
            \mathbb{P}&\left\{\left\|\hat{\boldsymbol{\Omega}}_2^{(n)} - \boldsymbol{\Omega}_2\right\| \geq \varepsilon\right\} \leq \xi_2(n, \varepsilon), \\
            &\quad\quad\quad\quad\quad\quad\quad\quad\vdots\\
            \mathbb{P}&\left\{\left\|\hat{\boldsymbol{\Omega}}_K^{(n)} - \boldsymbol{\Omega}_K\right\| \geq \varepsilon\right\} \leq \xi_K(n, \varepsilon),
        \end{split}
    \end{align*}
    then it holds that:
    \begin{align}\label{eq:productconvergence}
        \mathbb{P}\left\{\left\|\hat{\boldsymbol{\Omega}}_1^{(n)}\hat{\boldsymbol{\Omega}}_2^{(n)}\cdots\hat{\boldsymbol{\Omega}}_K^{(n)} - \boldsymbol{\Omega}_1\boldsymbol{\Omega}_2\cdots\boldsymbol{\Omega}_K\right\|\geq \varepsilon\right\} &\leq \sum_{i=1}^K \xi_i\left(n, \frac{\varepsilon}{K\prod_{k\neq i}^K M_k}\right),
    \end{align}
    where $M_k$ is an upper bound such that $\|\hat{\boldsymbol{\Omega}}_k^{(n)}\|\leq M_k$ almost surely, $\forall k=1,\ldots, K$.
\end{lemma}

\begin{proof}
    We begin by showing the case of $K=2$. By the triangle inequality, we have:
    \begin{align*}
        \begin{aligned}
			\left\|\hat{\boldsymbol{\Omega}}_1^{(n)}\hat{\boldsymbol{\Omega}}_2^{(n)} - \boldsymbol{\Omega}_1\boldsymbol{\Omega}_2\right\| &\leq \left\|\hat{\boldsymbol{\Omega}}_1^{(n)}\hat{\boldsymbol{\Omega}}_2^{(n)} - \boldsymbol{\Omega}_1\hat{\boldsymbol{\Omega}}_2^{(n)}\right\| + \left\|\boldsymbol{\Omega}_1\hat{\boldsymbol{\Omega}}_2^{(n)} - \boldsymbol{\Omega}_1\boldsymbol{\Omega}_2\right\|\\
            &\leq \left\|\hat{\boldsymbol{\Omega}}_2^{(n)}\right\|\left\|\hat{\boldsymbol{\Omega}}_1^{(n)} - \boldsymbol{\Omega}_1\right\| + \left\|\hat{\boldsymbol{\Omega}}_2^{(n)} - \boldsymbol{\Omega}_2\right\|\left\|\boldsymbol{\Omega}_1\right\|\\
            &\leq M_2\left\|\hat{\boldsymbol{\Omega}}_1^{(n)} - \boldsymbol{\Omega}_1\right\| + M_1\left\|\hat{\boldsymbol{\Omega}}_2^{(n)} - \boldsymbol{\Omega}_2\right\|.
        \end{aligned}
    \end{align*}
	Using the union bound, we have:
    \begin{align*}
        \begin{split}
            &\mathbb{P}\left\{\left\|\hat{\boldsymbol{\Omega}}_1^{(n)}\hat{\boldsymbol{\Omega}}_2^{(n)} - \boldsymbol{\Omega}_1\boldsymbol{\Omega}_2\right\|\geq \varepsilon\right\}\\
            \leq & \mathbb{P}\left\{M_2\left\|\hat{\boldsymbol{\Omega}}_1^{(n)} - \boldsymbol{\Omega}_1\right\| + M_1\left\|\hat{\boldsymbol{\Omega}}_2^{(n)} - \boldsymbol{\Omega}_2\right\|\geq\varepsilon\right\}\\
            \leq& \mathbb{P}\left\{M_2\left\|\hat{\boldsymbol{\Omega}}_1^{(n)} - \boldsymbol{\Omega}_1\right\|\geq\varepsilon/2\right\}+\mathbb{P}\left\{M_1\left\|\hat{\boldsymbol{\Omega}}_2^{(n)} - \boldsymbol{\Omega}_2\right\|\geq\varepsilon/2\right\} \\
            \leq &\xi_1\left(n, \frac{\varepsilon}{2M_2}\right) + \xi_2\left(n, \frac{\varepsilon}{2M_1}\right).
        \end{split}
    \end{align*}
    For any $K>2$, suppose the statement (\ref{eq:productconvergence}) holds for $k=2,\ldots,K-1$. Observe that:
    \begin{align}\label{eq:productconvergence proof} 
        \begin{split}
            &\left\|\hat{\boldsymbol{\Omega}}_1^{(n)}\hat{\boldsymbol{\Omega}}_2^{(n)}\cdots\hat{\boldsymbol{\Omega}}_K^{(n)} - \boldsymbol{\Omega}_1\boldsymbol{\Omega}_2\cdots\boldsymbol{\Omega}_K\right\|\\
            &\leq \left\|\hat{\boldsymbol{\Omega}}_1^{(n)}\hat{\boldsymbol{\Omega}}_2^{(n)}\cdots\hat{\boldsymbol{\Omega}}_K^{(n)} - \boldsymbol{\Omega}_1\boldsymbol{\Omega}_2\cdots\boldsymbol{\Omega}_{K-1}\hat{\boldsymbol{\Omega}}_K^{(n)}\right\|+\left\|\boldsymbol{\Omega}_1\boldsymbol{\Omega}_2\cdots\boldsymbol{\Omega}_{K-1}\hat{\boldsymbol{\Omega}}_K^{(n)} - \boldsymbol{\Omega}_1\boldsymbol{\Omega}_2\cdots\boldsymbol{\Omega}_K\right\|\\
            &\leq M_K\left\|\hat{\boldsymbol{\Omega}}_1^{(n)}\hat{\boldsymbol{\Omega}}_2^{(n)}\cdots\hat{\boldsymbol{\Omega}}_{K-1}^{(n)} - \boldsymbol{\Omega}_1\boldsymbol{\Omega}_2\cdots\boldsymbol{\Omega}_{K-1}\right\| + \prod_{k=1}^{K-1}M_k\left\|\hat{\boldsymbol{\Omega}}_K^{(n)} - \boldsymbol{\Omega}_K\right\|.\\
        \end{split}
    \end{align}
    Then it follows that:
    \begin{align*}
        \begin{split}
            &\mathbb{P}\left\{\left\|\hat{\boldsymbol{\Omega}}_1^{(n)}\hat{\boldsymbol{\Omega}}_2^{(n)}\cdots\hat{\boldsymbol{\Omega}}_K^{(n)} - \boldsymbol{\Omega}_1\boldsymbol{\Omega}_2\cdots\boldsymbol{\Omega}_K\right\|\geq \varepsilon\right\} \\
            &\leq \mathbb{P}\left\{ M_K\left\|\hat{\boldsymbol{\Omega}}_1^{(n)}\hat{\boldsymbol{\Omega}}_2^{(n)}\cdots\hat{\boldsymbol{\Omega}}_{K-1}^{(n)} - \boldsymbol{\Omega}_1\boldsymbol{\Omega}_2\cdots\boldsymbol{\Omega}_{K-1}\right\| + \prod_{k=1}^{K-1}M_k\left\|\hat{\boldsymbol{\Omega}}_K^{(n)} - \boldsymbol{\Omega}_K\right\|\geq\varepsilon\right\}\\
            &\leq \mathbb{P}\left\{M_K\left\|\hat{\boldsymbol{\Omega}}_1^{(n)}\hat{\boldsymbol{\Omega}}_2^{(n)}\cdots\hat{\boldsymbol{\Omega}}_{K-1}^{(n)} - \boldsymbol{\Omega}_1\boldsymbol{\Omega}_2\cdots\boldsymbol{\Omega}_{K-1}\right\|\geq \frac{K-1}{K}\varepsilon\right\} \\
            &\qquad+ \mathbb{P}\left\{\prod_{k=1}^{K-1}M_k\left\|\hat{\boldsymbol{\Omega}}_K^{(n)} - \boldsymbol{\Omega}_K\right\|\geq \frac{1}{K}\varepsilon\right\}\\
            &\leq\sum_{i=1}^{K-1} \xi_i\left(n, \frac{\varepsilon}{KM_K\prod_{k\neq i}^{K-1} M_k}\right)+\xi_K\left(n, \frac{\varepsilon}{K\prod_{k=1}^{K-1} M_k}\right)\\
            &=\sum_{i=1}^K \xi_i\left(n, \frac{\varepsilon}{K\prod_{k\neq i}^K M_k}\right).
        \end{split}
    \end{align*}
	Thus, by induction, the proof is complete.
\end{proof}

\begin{rmk}\label{rmk:productconvergence}
    In Lemma \ref{lemma:productconvergence}, consider the special case where $\boldsymbol{\Omega}_{1}=\boldsymbol{0}$. Then the inequality (\ref{eq:productconvergence proof}) can be simplified as follows:
    \begin{align*}
        \begin{split}
            \left\|\hat{\boldsymbol{\Omega}}_1^{(n)}\hat{\boldsymbol{\Omega}}_2^{(n)}\cdots\hat{\boldsymbol{\Omega}}_K^{(n)}-\boldsymbol{0}\right\| &\leq \prod_{k=2}^KM_k\left\|\hat{\boldsymbol{\Omega}}_1^{(n)}\right\|.
        \end{split}
    \end{align*}
    And we have the following simplified form:
    \begin{align*}
        \begin{split}
            \mathbb{P}\left\{\left\|\hat{\boldsymbol{\Omega}}_1^{(n)}\hat{\boldsymbol{\Omega}}_2^{(n)}\cdots\hat{\boldsymbol{\Omega}}_K^{(n)}-\boldsymbol{0}\right\|\geq \varepsilon\right\} &\leq \mathbb{P}\left\{\prod_{k=2}^KM_k\left\|\hat{\boldsymbol{\Omega}}_1^{(n)}\right\|\geq\varepsilon\right\}\\
            &\leq \xi_1\left(n, \frac{\varepsilon}{\prod_{k=2}^KM_k}\right).
        \end{split}
    \end{align*}
\end{rmk}

\begin{lemma}[Deviation Inequality for Minimum Eigenvalue of Projected Sample Covariance Matrix]\label{lemma:deviation projected}
    Suppose Assumption \ref{regularity assumption} holds. 

    When $n\geq \max\left\{\frac{qe^{\frac{3}{2}}}{K}, \frac{p^2(q+1)^2K}{qK_0^2}\right\}$, the following inequality holds with probability at least $1-\frac{3qe^{\frac{1}{2}}}{Kn}$:
    \begin{align*}
        \lambda_{\min}\left(\frac{1}{n}\boldsymbol{X}^\top\boldsymbol{P}_Z\boldsymbol{X}\right)\geq \lambda_z \left(\sigma_{\min}(\boldsymbol{\Theta})-\sqrt{\frac{2p(q+1)B_{\epsilon_2}^2\log(\frac{K}{q}n)}{\lambda_{\min}(\boldsymbol{\Sigma}_z)n}}\right)^2:=\lambda_{\tilde{x}},
    \end{align*}

    where $K:=\frac{\lambda_{\min}(\boldsymbol{\Sigma}_z)}{6B_z^2}, K_0:=\frac{\lambda_{\min}(\boldsymbol{\Sigma}_z)\sigma_{\min}^2(\boldsymbol{\Theta})}{2B_{\epsilon_2}^2}, \boldsymbol{\Sigma}_z:=\mathbb{E}[\boldsymbol{zz}^\top], \boldsymbol{P}_Z:=\boldsymbol{Z}(\boldsymbol{Z}^\top\boldsymbol{Z})^{-1}\boldsymbol{Z}^\top$ , and $\lambda_z$ is a lower bound of $\lambda_{\min}(\frac{\boldsymbol{Z}^\top\boldsymbol{Z}}{n})$.
\end{lemma}
\begin{proof}
    Let 
    \begin{align*}
        \boldsymbol{E}_\parallel:=\boldsymbol{P}_Z\boldsymbol{\mathcal{E}}_2, \boldsymbol{E}_\perp:=(\boldsymbol{I}-\boldsymbol{P}_Z)\boldsymbol{\mathcal{E}}_2.
    \end{align*}
    We have the following decomposition:
    \begin{align*}
        \begin{split}
            \boldsymbol{X}=\boldsymbol{Z\Theta}+\boldsymbol{\mathcal{E}}_2=\boldsymbol{Z\Theta}+\boldsymbol{E}_\parallel+\boldsymbol{E}_\perp=\boldsymbol{Z}(\boldsymbol{\Theta}+(\boldsymbol{Z}^T\boldsymbol{Z})^{-1}\boldsymbol{Z}^T\boldsymbol{\mathcal{E}}_2)+\boldsymbol{E}_\perp.
        \end{split}
    \end{align*}
    Let $\boldsymbol{\Psi}:=(\boldsymbol{Z}^\top\boldsymbol{Z})^{-1}\boldsymbol{Z}^\top\boldsymbol{\mathcal{E}}_2$. Since $\boldsymbol{P}_Z\boldsymbol{E}_\perp=\boldsymbol{0}$, we have
    \begin{align*}
        \begin{split}
            \boldsymbol{X}^\top\boldsymbol{P}_Z\boldsymbol{X}&=\left(\boldsymbol{Z}\left(\boldsymbol{\Theta}+\boldsymbol{\Psi}\right)\right)^\top\boldsymbol{P}_Z\left(\boldsymbol{Z}\left(\boldsymbol{\Theta}+\boldsymbol{\Psi}\right)\right)\\
            &=\left(\boldsymbol{\Theta}+\boldsymbol{\Psi}\right)^\top\boldsymbol{Z}^\top\boldsymbol{Z}\left(\boldsymbol{\Theta}+\boldsymbol{\Psi}\right).
        \end{split}
    \end{align*}
    We can now write:
    \begin{align}\label{eq:min eigenvalue}
        \begin{split}
            \lambda_{\min}\left(\frac{1}{n}\boldsymbol{X}^\top\boldsymbol{P}_Z\boldsymbol{X}\right)
            &=\lambda_{\min}\left(\left(\boldsymbol{\Theta}+\boldsymbol{\Psi}\right)^\top\frac{\boldsymbol{Z}^\top\boldsymbol{Z}}{n}\left(\boldsymbol{\Theta}+\boldsymbol{\Psi}\right)\right).
        \end{split}
    \end{align}

    Note that in general, for a positive semi-definite matrix $\boldsymbol{A}$, we have
    \begin{align*}
        \lambda_{\min}(\boldsymbol{A})=\min_{\boldsymbol{u}\neq \boldsymbol{0}}\frac{\boldsymbol{u}^\top\boldsymbol{A}\boldsymbol{u}}{\boldsymbol{u}^\top\boldsymbol{u}},
    \end{align*}
    and
    \begin{align*}
        \begin{split}
            \lambda_{\min}(\boldsymbol{B}^\top\boldsymbol{AB})&= \min_{\boldsymbol{u}\neq \boldsymbol{0}}\frac{(\boldsymbol{Bu})^\top\boldsymbol{AB}\boldsymbol{u}}{\boldsymbol{u}^\top\boldsymbol{u}}\\
            &\geq \min_{\boldsymbol{u}\neq \boldsymbol{0}}\lambda_{\min}(\boldsymbol{A})\frac{(\boldsymbol{Bu})^\top \boldsymbol{Bu}}{\boldsymbol{u}^\top\boldsymbol{u}}\\
            &=\lambda_{\min}(\boldsymbol{A})\lambda_{\min}(\boldsymbol{B}^\top\boldsymbol{B}).
        \end{split}
    \end{align*}

    Thus, from \eqref{eq:min eigenvalue}, with probability at least $1-\xi$, we have: 
    \begin{align*}
        \begin{split}
            \lambda_{\min}\left(\frac{1}{n}\boldsymbol{X}^\top\boldsymbol{P}_Z\boldsymbol{X}\right)
            &= \lambda_{\min}\left(\left(\boldsymbol{\Theta}+\boldsymbol{\Psi}\right)^\top\frac{\boldsymbol{Z}^\top\boldsymbol{Z}}{n}\left(\boldsymbol{\Theta}+\boldsymbol{\Psi}\right)\right)\\
            &\geq \lambda_z\lambda_{\min}\left(\left(\boldsymbol{\Theta}+\boldsymbol{\Psi}\right)^\top\left(\boldsymbol{\Theta}+\boldsymbol{\Psi}\right)\right).
        \end{split}
    \end{align*}
    It now remains to bound $\lambda_{\min}\left(\left(\boldsymbol{\Theta}+\boldsymbol{\Psi}\right)^\top\left(\boldsymbol{\Theta}+\boldsymbol{\Psi}\right)\right)=\sigma_{\min}^2(\boldsymbol{\Theta}+\boldsymbol{\Psi}).$

    From \citep{hsu2014randomdesignanalysisridge}, for each $k\in[p]$ and any given $t>1$, with sample size satisfying \begin{align} \label{eq:hsu condition}
        n\geq \frac{6B_z^2(\log q+t)}{\lambda_{\min}(\boldsymbol{\Sigma}_z)},
    \end{align} 
    we have the following holds with probability at least $1-3e^{-t}$:
    \begin{align*}
        \|\boldsymbol{\Psi}_k\|_{\boldsymbol{\Sigma}_z}^2&=\|\hat{\boldsymbol{\Theta}}_k-\boldsymbol{\Theta}_k\|_{\boldsymbol{\Sigma}_z}^2\leq \frac{B_{\epsilon_2}^2\left(q+2\sqrt{qt}+2t\right)}{n}< \frac{B_{\epsilon_2}^2\left[q+2(q+1)t\right]}{n}.
    \end{align*}
    
    Note that 
    \begin{align*}
        \|\boldsymbol{\Psi}_k\|_{\boldsymbol{\Sigma}_z}^2=\boldsymbol{\Psi}_k^\top\boldsymbol{\Sigma}_z\boldsymbol{\Psi}_k=\boldsymbol{\Psi}_k^\top\boldsymbol{U\Lambda}_z\boldsymbol{U}^\top\boldsymbol{\Psi}_k=\sum_{i=1}^q\lambda_{z,i}(\boldsymbol{U}^\top\boldsymbol{\Psi}_k)_i^2,
    \end{align*}
    and 
    \begin{align*}
        \|\boldsymbol{\Psi}_k\|=\boldsymbol{\Psi}_k^\top\boldsymbol{\Psi}_k=\boldsymbol{\Psi}_k^\top\boldsymbol{U}\boldsymbol{U}^\top\boldsymbol{\Psi}_k=\sum_{i=1}^q(\boldsymbol{U}^\top\boldsymbol{\Psi}_k)_i^2.
    \end{align*}
    We have 
    \begin{align*}
        \lambda_{\min}(\boldsymbol{\Sigma}_z)\|\boldsymbol{\Psi}_k\|^2\leq\|\boldsymbol{\Psi}_k\|_{\boldsymbol{\Sigma}_z}^2\leq \lambda_{\max}(\boldsymbol{\Sigma}_z)\|\boldsymbol{\Psi}_k\|^2.
    \end{align*}
    Then
    \begin{align} \label{eq:bound on Psi original}
        \|\boldsymbol{\Psi}\|\leq\|\boldsymbol{\Psi}\|_F=\sqrt{\sum_{k=1}^p\|\boldsymbol{\Psi}_k\|^2}\leq\sqrt{\sum_{k=1}^p\frac{1}{\lambda_{\min}(\boldsymbol{\Sigma}_z)}\|\boldsymbol{\Psi}_k\|_{\boldsymbol{\Sigma}_z}^2}<\sqrt{\frac{pB_{\epsilon_2}^2\left[q+2(q+1)t\right]}{\lambda_{\min}(\boldsymbol{\Sigma}_z)n}}.
    \end{align}
    
    Hence, by Weyl's inequality, we have 
    \begin{align}\label{eq: Weyl ineq}
        \sigma_{\min}(\boldsymbol{\Theta}+\boldsymbol{\Psi})\geq\sigma_{\min}(\boldsymbol{\Theta})-\|\boldsymbol{\Psi}\|>\sigma_{\min}(\boldsymbol{\Theta})-\sqrt{\frac{pB_{\epsilon_2}^2\left[q+2(q+1) t\right]}{\lambda_{\min}(\boldsymbol{\Sigma}_z)n}},
    \end{align}
    where $t$ is taken to be small enough such that the RHS $\geq 0$, i.e.
    \begin{align}\label{eq:t condition 0}
        1<t\leq \frac{K_0n}{p(q+1)}-\frac{q}{2(q+1)},
    \end{align}
    where $K_0:=\frac{\lambda_{\min}(\boldsymbol{\Sigma}_z)\sigma_{\min}^2(\boldsymbol{\Theta})}{2B_{\epsilon_2}^2}$. We now rewrite inequality \eqref{eq: Weyl ineq} in terms of $n$ only. From condition \eqref{eq:hsu condition}, for any given sample size $n$, the range for $t$ is:
    \begin{align}\label{eq:t condition 1}
        1<t\leq {Kn}-\log q,
    \end{align}
    where $K:=\frac{\lambda_{\min}(\boldsymbol{\Sigma}_z)}{6B_z^2}$. We take
    \begin{align*}
        t=\log\left({Kn}\right)-\log q-\frac{1}{2}=\log\left(\frac{K}{q}n\right)-\frac{1}{2},
    \end{align*}
    So that condition \eqref{eq:t condition 1} is satisfied when $n\geq\frac{qe^{\frac{3}{2}}}{K}$. To satisfy condition \eqref{eq:t condition 0}, a sufficient condition is:
    \begin{align*}
        \log\left(\frac{K}{q}n\right)\leq \frac{K_0n}{p(q+1)}.
    \end{align*}
    Note that when $n\geq \frac{qe^\frac{3}{2}}{K}$, we also have:
    \begin{align*}
        \log\left(\frac{K}{q}n\right)\leq \sqrt{\frac{K}{q}n}.
    \end{align*}
    So a sufficient condition to satisfy both \eqref{eq:t condition 0} and \eqref{eq:t condition 1} is:
    \begin{align*}
        n\geq \max\left\{\frac{qe^{\frac{3}{2}}}{K}, \frac{p^2(q+1)^2K}{qK_0^2}\right\}.
    \end{align*}

    Then the bound \eqref{eq:bound on Psi original} can be rewritten as:
    \begin{align}\label{eq:bound on Psi}
        \|\boldsymbol{\Psi}\|\leq\sqrt{\frac{pB_{\epsilon_2}^2\left[q+2(q+1)\left(\log\left(\frac{K}{q}n\right)-\frac{1}{2}\right)\right]}{\lambda_{\min}(\boldsymbol{\Sigma}_z)n}}<\sqrt{\frac{2p(q+1)B_{\epsilon_2}^2\log\left(\frac{K}{q}n\right)}{\lambda_{\min}(\boldsymbol{\Sigma}_z)n}}.
    \end{align}
    
    Finally, from \eqref{eq:min eigenvalue},
    \begin{align*}
        \lambda_{\min}\left(\frac{1}{n}\boldsymbol{X}^\top\boldsymbol{P}_Z\boldsymbol{X}\right)\geq \lambda_z \left(\sigma_{\min}(\boldsymbol{\Theta})-\sqrt{\frac{2p(q+1)B_{\epsilon_2}^2\log\left(\frac{K}{q}n\right)}{\lambda_{\min}(\boldsymbol{\Sigma}_z)n}}\right)^2.
    \end{align*}
\end{proof}

\begin{proof}[Proof of Theorem \ref{thm:consistency}]
    We denote the observational values $(\boldsymbol{Z},\boldsymbol{X},\boldsymbol{Y})=\{(\boldsymbol{z}_i,\boldsymbol{x}_i, y_i)\}_{i=1}^n$, and $\boldsymbol{\mathcal{E}}_1=\{\epsilon_{1,i}\}_{i=1}^n, \boldsymbol{\mathcal{E}}_2=\{\boldsymbol{\epsilon}_{2,i}\}_{i=1}^n$. The \textsf{2SLS} estimator is given by:
    \begin{align}\label{eq:beta2sls}
        \begin{split}
        \hat{\boldsymbol{\beta}}_{\textsf{2SLS}} &= \left(\hat{\boldsymbol{\Theta}}^\top\boldsymbol{Z}^\top\boldsymbol{Z}\hat{\boldsymbol{\Theta}}\right)^{-1}\hat{\boldsymbol{\Theta}}^\top\boldsymbol{Z}^\top\boldsymbol{Y}\\
       &= \left[\left((\boldsymbol{Z}^\top\boldsymbol{Z})^{-1}\boldsymbol{Z}^\top\boldsymbol{X}\right)^\top\boldsymbol{Z}^\top\boldsymbol{Z}(\boldsymbol{Z}^\top\boldsymbol{Z})^{-1}\boldsymbol{Z}^\top\boldsymbol{X}\right]^{-1}\left((\boldsymbol{Z}^\top\boldsymbol{Z})^{-1}\boldsymbol{Z}^\top\boldsymbol{X}\right)^\top\boldsymbol{Z}^\top\boldsymbol{Y}\\
       &=\left(\boldsymbol{X}^\top\boldsymbol{Z}(\boldsymbol{Z}^\top\boldsymbol{Z})^{-1}\boldsymbol{Z}^\top\boldsymbol{X}\right)^{-1}\boldsymbol{X}^\top\boldsymbol{Z}(\boldsymbol{Z}^\top\boldsymbol{Z})^{-1}\boldsymbol{Z}^\top\boldsymbol{Y}\\
	   &=\boldsymbol{\beta}+\left(\boldsymbol{X}^\top\boldsymbol{Z}(\boldsymbol{Z}^\top\boldsymbol{Z})^{-1}\boldsymbol{Z}^\top\boldsymbol{X}\right)^{-1}\boldsymbol{X}^\top\boldsymbol{Z}(\boldsymbol{Z}^\top\boldsymbol{Z})^{-1}\boldsymbol{Z}^\top\boldsymbol{\mathcal{E}}_1.
	\end{split}
\end{align}
Define constants $\lambda_z,\lambda_{\tilde{x}}>0$, such that the following event  $\mathcal{A}$ holds with probability at least $1-\xi$: 
\begin{align}\label{eq:sample conditions}
    \mathcal{A}=\left\{\lambda_{\min}\bigg(\frac{\boldsymbol{Z}^\top\boldsymbol{Z}}{n}\bigg)\geq \lambda_z, \lambda_{\min}\bigg(\frac{\boldsymbol{X}^\top\boldsymbol{P}_Z\boldsymbol{X}}{n}\bigg)\geq \lambda_{\tilde{x}}\right\},
\end{align} 
where $\lambda_{\min}(\cdot)$ denotes the smallest eigenvalue of a matrix, $\boldsymbol{P}_Z:=\boldsymbol{Z}(\boldsymbol{Z}^\top\boldsymbol{Z})^{-1}\boldsymbol{Z}^\top$ denotes the projection matrix. We will first assume the existence of such $\lambda_z, \lambda_{\tilde{x}}$, with their values to be determined later.

We first consider the case when event $\mathcal{A}$ is true. Let $\boldsymbol{Q}_{zz}:=\mathbb{E}[\boldsymbol{z}\boldsymbol{z}^\top|\mathcal{A}], \boldsymbol{Q}_{zx}:=\mathbb{E}[\boldsymbol{z}\boldsymbol{x}^\top|\mathcal{A}]$, $\bar{\boldsymbol{\Omega}}_{zz}:=\sum_{i=1}^n(\boldsymbol{z}_i\boldsymbol{z}_i^\top-\boldsymbol{Q}_{zz}), \bar{\boldsymbol{\Omega}}_{zx}:=\sum_{i=1}^n(\boldsymbol{z}_i\boldsymbol{x}_i^\top-\boldsymbol{Q}_{zx}), \boldsymbol{\Omega}_{z\epsilon_1}:=\sum_{i=1}^n\boldsymbol{z}_i\epsilon_{1,i}$. 

Let $\bar{B}_{zz},\bar{B}_{zx},B_{zx}, B_{z\epsilon_1}$ be some upper bounds such that $\left\|\boldsymbol{z}_i\boldsymbol{z}_i^\top-\boldsymbol{Q}_{zz}\right\|\leq \bar{B}_{zz}, \left\|\boldsymbol{z}_i\boldsymbol{x}_i^\top-\boldsymbol{Q}_{zx}\right\|\leq \bar{B}_{zx},\left\|\boldsymbol{z}_i\boldsymbol{x}_i^\top\right\|\leq B_{zx}, \left\|\boldsymbol{z}_i\epsilon_{1,i}\right\|\leq B_{z\epsilon_1}$ almost surely, for all $i=1,\ldots,n$. The existence of $\bar{B}_{zz},\bar{B}_{zx},B_{zx},B_{z\epsilon_1}$ is guaranteed under Assumption \ref{regularity assumption}(ii). 

By Lemma \ref{lemma:bernstein}, we have:
    \begin{align}\label{eq:ZZ}
        \begin{split}
            \mathbb{P}\left\{\left\|\frac{\boldsymbol{Z}^\top\boldsymbol{Z}}{n}-\boldsymbol{Q}_{zz}\right\|\geq \varepsilon\middle|\mathcal{A}\right\}&=
            \mathbb{P}\left\{\left\|\frac{\sum_{i=1}^n\boldsymbol{z}_i\boldsymbol{z}_i^\top}{n}-\boldsymbol{Q}_{zz}\right\|\geq \varepsilon\middle|\mathcal{A}\right\}\\
            &=\mathbb{P}\left\{\left\|\sum_{i=1}^n(\boldsymbol{z}_i\boldsymbol{z}_i^\top-\boldsymbol{Q}_{zz})\right\|\geq n\varepsilon\middle|\mathcal{A}\right\}\\
            &\leq 2q\exp\left(-\frac{n^2\varepsilon^2/2}{\nu(\bar{\boldsymbol{\Omega}}_{zz}|\mathcal{A})+\bar{B}_{zz}n\varepsilon/3}\right).
        \end{split}
    \end{align}
    Similarly,
	\begin{align}\label{eq:ZX}
        \begin{split}
            \mathbb{P}\left\{\left\|\frac{\boldsymbol{Z}^\top\boldsymbol{X}}{n}-\boldsymbol{Q}_{zx}\right\|\geq \varepsilon\middle|\mathcal{A}\right\}
            &=\mathbb{P}\left\{\left\|\frac{\sum_{i=1}^n\boldsymbol{z}_i\boldsymbol{x}_i^\top}{n}-\boldsymbol{Q}_{zx}\right\|\geq \varepsilon\middle|\mathcal{A}\right\}\\
            &=\mathbb{P}\left\{\left\|\sum_{i=1}^n(\boldsymbol{z}_i\boldsymbol{x}_i^\top-\boldsymbol{Q}_{zx})\right\|\geq n\varepsilon\middle|\mathcal{A}\right\}\\
            &\leq (p+q)\exp\left(-\frac{n^2\varepsilon^2/2}{\nu(\bar{\boldsymbol{\Omega}}_{zx}|\mathcal{A})+\bar{B}_{zx}n\varepsilon/3}\right).
        \end{split}
    \end{align}
	By Assumption \ref{iv assumption}(iii), the instrument $\boldsymbol{z}$ is uncorrelated with the error term $\epsilon_1$, which implies $\mathbb{E}[\boldsymbol{z}\epsilon_1|\mathcal{A}] = \boldsymbol{0}$. Applying Lemma \ref{lemma:bernstein} again, we have:
    \begin{align}\label{eq:ZU}
        \begin{split}
            \mathbb{P}\left\{\left\|\frac{\boldsymbol{Z}^\top\boldsymbol{\mathcal{E}}_1}{n}\right\|\geq \varepsilon\middle|\mathcal{A}\right\}&=\mathbb{P}\left\{\left\|\frac{\sum_{i=1}^n\boldsymbol{z}_i{\epsilon}_{1,i}}{n}\right\|\geq \varepsilon\middle|\mathcal{A}\right\}\\
            &=\mathbb{P}\left\{\left\|\sum_{i=1}^n \boldsymbol{z}_i{\epsilon}_{1,i}\right\|\geq n\varepsilon\middle|\mathcal{A}\right\}\\
            &\leq (q+1)\exp\left(-\frac{n^2\varepsilon^2/2}{\nu(\boldsymbol{\Omega}_{z\epsilon_1}|\mathcal{A})+B_{z\epsilon_1}n\varepsilon/3}\right).\\
        \end{split}
    \end{align}

	With Lemma \ref{lemma:jin} and (\ref{eq:ZZ}), we have:
    \begin{align}\label{eq:ZZinv}
        \begin{split}
            \mathbb{P}\left\{\left\|n(\boldsymbol{Z}^\top\boldsymbol{Z})^{-1}-\boldsymbol{Q}_{ZZ}^{-1}\right\|\geq \varepsilon\middle|\mathcal{A}\right\}&\leq 2q\exp\left(-\frac{n^2(\lambda_z^2\varepsilon)^2/2}{\nu(\bar{\boldsymbol{\Omega}}_{zz}|\mathcal{A})+\bar{B}_{zz}n(\lambda_z^2\varepsilon)/3}\right) \\
            &= 2q\exp\left(-\frac{\lambda_z^4n^2\varepsilon^2/2}{\nu(\bar{\boldsymbol{\Omega}}_{zz}|\mathcal{A})+\lambda_z^2\bar{B}_{zz} n\varepsilon/3}\right).
        \end{split}
    \end{align}

    Note that we have $\hat{\boldsymbol{\Theta}}=\boldsymbol{\Theta}+(\boldsymbol{Z}^\top\boldsymbol{Z})^{-1}\boldsymbol{Z}^\top\boldsymbol{\mathcal{E}}_2:=\boldsymbol{\Theta}+\boldsymbol{\Psi}$. Under event $\mathcal{A}$,
    \begin{align}\label{eq:Theta hat}
        \|\hat{\boldsymbol{\Theta}}\|=\|\boldsymbol{\Theta}+\boldsymbol{\Psi}\|\leq\|\boldsymbol{\Theta}\|+\|\boldsymbol{\Psi}\|\leq B_{\Theta}+B_\Psi:=B_{\hat{\Theta}}.
    \end{align}
    With Lemma \ref{lemma:productconvergence} (Remark \ref{rmk:productconvergence}), combining (\ref{eq:ZU})(\ref{eq:Theta hat}), we have:
    \begin{align}\label{eq:sandwich0}
        \begin{split}
            \mathbb{P}\left\{\left\|\frac{1}{n}\boldsymbol{X}^\top\boldsymbol{Z}(\boldsymbol{Z}^\top\boldsymbol{Z})^{-1}\boldsymbol{Z}^\top\boldsymbol{\mathcal{E}}_1-\boldsymbol{0}\right\|\geq \varepsilon\middle|\mathcal{A}\right\}
            &\leq (q+1)\exp\left(-\frac{n^2(\frac{\varepsilon}{B_{\hat{\Theta}}})^2/2}{\nu(\boldsymbol{\Omega}_{z\epsilon_1}|\mathcal{A})+B_{z\epsilon_1}n(\frac{\varepsilon}{B_{\hat{\Theta}}})/3}\right)\\
            &= (q+1)\exp\left(-\frac{n^2\varepsilon^2/2}{B_{\hat{\Theta}}^2\nu(\boldsymbol{\Omega}_{z\epsilon_1}|\mathcal{A})+B_{\hat{\Theta}}B_{z\epsilon_1}n\varepsilon/3}\right).
        \end{split}
    \end{align}
	Additionally, with Lemma \ref{lemma:productconvergence}, combining (\ref{eq:ZX})(\ref{eq:ZZinv}), we have:
    \begin{align*}
    \begin{split}
        &\mathbb{P}\left\{\left\|\frac{1}{n}\boldsymbol{X}^\top\boldsymbol{Z}(\boldsymbol{Z}^\top\boldsymbol{Z})^{-1}\boldsymbol{Z}^\top\boldsymbol{X}-\boldsymbol{Q}_{zx}^\top\boldsymbol{Q}_{zz}^{-1}\boldsymbol{Q}_{zx}\right\|\geq \varepsilon\middle|\mathcal{A}\right\}\\
        &\leq 2(p+q)\exp\left(-\frac{n^2(\frac{\lambda_z\varepsilon}{3B_{zx}})^2/2}{\nu(\bar{\boldsymbol{\Omega}}_{zx}|\mathcal{A})+\bar{B}_{zx}n(\frac{\lambda_z\varepsilon}{3B_{zx}})/3}\right)+2q\exp\left(-\frac{\lambda_z^4n^2(\frac{\varepsilon}{3B_{zx}^2})^2/2}{\nu(\bar{\boldsymbol{\Omega}}_{zz}|\mathcal{A})+\lambda_z^2\bar{B}_{zz} n(\frac{\varepsilon}{3B_{zx}^2})/3}\right)\\
        &=  2(p+q)\exp\left(-\frac{\lambda_z^2n^2\varepsilon^2/2}{9B_{zx}^2\nu(\bar{\boldsymbol{\Omega}}_{zx}|\mathcal{A})+\lambda_zB_{zx}\bar{B}_{zx}n\varepsilon}\right)+2q\exp\left(-\frac{\lambda_z^4n^2\varepsilon^2/2}{9B_{zx}^4\nu(\bar{\boldsymbol{\Omega}}_{zz}|\mathcal{A})+\lambda_z^2B_{zx}^2\bar{B}_{zz} n\varepsilon}\right).
    \end{split}
    \end{align*}
	Applying Lemma \ref{lemma:jin} again, we have:
    \begin{align}\label{eq:sandwichinv}
        \begin{split}
            &\mathbb{P}\left\{\left\|n\left(\boldsymbol{X}^\top\boldsymbol{Z}(\boldsymbol{Z}^\top\boldsymbol{Z})^{-1}\boldsymbol{Z}^\top\boldsymbol{X}\right)^{-1}-(\boldsymbol{Q}_{ZX}^\top\boldsymbol{Q}_{zz}^{-1}\boldsymbol{Q}_{zx})^{-1}\right\|\geq \varepsilon\middle|\mathcal{A}\right\}\\
            &\leq 2(p+q)\exp\left(-\frac{\lambda_z^2n^2(\lambda_{\tilde{x}}^2\varepsilon)^2/2}{9B_{zx}^2\nu(\bar{\boldsymbol{\Omega}}_{zx}|\mathcal{A})+\lambda_zB_{zx}\bar{B}_{zx}n(\lambda_{\tilde{x}}^2\varepsilon)}\right)+2q\exp\left(-\frac{\lambda_z^4n^2(\lambda_{\tilde{x}}^2\varepsilon)^2/2}{9B_{zx}^4\nu(\bar{\boldsymbol{\Omega}}_{zz}|\mathcal{A})+\lambda_z^2B_{zx}^2\bar{B}_{zz} n(\lambda_{\tilde{x}}^2\varepsilon)}\right)\\
            &= 2(p+q)\exp\left(-\frac{\lambda_z^2\lambda_{\tilde{x}}^4n^2\varepsilon^2/2}{9B_{zx}^2\nu(\bar{\boldsymbol{\Omega}}_{zx}|\mathcal{A})+\lambda_z\lambda_{\tilde{x}}^2 B_{zx}\bar{B}_{zx}n\varepsilon}\right)+2q\exp\left(-\frac{\lambda_z^4\lambda_{\tilde{x}}^4n^2\varepsilon^2/2}{9B_{zx}^4\nu(\bar{\boldsymbol{\Omega}}_{zz}|\mathcal{A})+\lambda_z^2\lambda_{\tilde{x}}^2B_{zx}^2\bar{B}_{zz} n\varepsilon}\right).
        \end{split}
    \end{align}

    Therefore, we have shown that under event $\mathcal{A}$,
    \begin{align*}
        n\left(\boldsymbol{X}^\top\boldsymbol{Z}(\boldsymbol{Z}^\top\boldsymbol{Z})^{-1}\boldsymbol{Z}^\top\boldsymbol{X}\right)^{-1}\convp (\boldsymbol{Q}_{ZX}^\top\boldsymbol{Q}_{zz}^{-1}\boldsymbol{Q}_{zx})^{-1}.
    \end{align*}
    From equation (\ref{eq:beta2sls}), combining (\ref{eq:sandwich0}) and (\ref{eq:sandwichinv}) with Lemma \ref{lemma:productconvergence} (Remark \ref{rmk:productconvergence}), we have:
    \begin{align*}
            &\mathbb{P}\left\{\left\|\hat{\boldsymbol{\beta}}_{\textsf{2SLS}}-\boldsymbol{\beta}\right\|\geq \varepsilon\middle|\mathcal{A}\right\}\nonumber\\
            &=\mathbb{P}\left\{\left\|\left(\boldsymbol{X}^\top\boldsymbol{Z}(\boldsymbol{Z}^\top\boldsymbol{Z})^{-1}\boldsymbol{Z}^\top\boldsymbol{X}\right)^{-1}\boldsymbol{X}^\top\boldsymbol{Z}(\boldsymbol{Z}^\top\boldsymbol{Z})^{-1}\boldsymbol{Z}^\top\boldsymbol{\mathcal{E}}_1-\boldsymbol{0}\right\|\geq \varepsilon\middle|\mathcal{A}\right\}\nonumber\\
            &\leq (q+1)\exp\left(-\frac{n^2(\lambda_{\tilde{x}}\varepsilon)^2/2}{B_{\hat{\Theta}}^2\nu(\boldsymbol{\Omega}_{z\epsilon_1}|\mathcal{A})+B_{\hat{\Theta}}B_{z\epsilon_1}n(\lambda_{\tilde{x}}\varepsilon)/3}\right)\nonumber\\
            &= (q+1)\exp\left(-\frac{\lambda_{\tilde{x}}^2n^2\varepsilon^2/2}{B_{\hat{\Theta}}^2\nu(\boldsymbol{\Omega}_{z\epsilon_1}|\mathcal{A})+\lambda_{\tilde{x}} B_{\hat{\Theta}}B_{z\epsilon_1}n\varepsilon/3}\right).
        \end{align*}
    
    For the second part of the theorem, let $c:=\left(\frac{3B_{\hat{\Theta}}\nu(\boldsymbol{\Omega}_{z\epsilon_1}|\mathcal{A})}{\lambda_{\tilde{x}}B_{z\epsilon_1}n}\right)^2$, we have:
    \begin{align}\label{eq:2sls squared error expectation}
        \begin{split}
            &\mathbb{E}\left[\|\textsf{clip}_{B_\beta}(\hat{\boldsymbol{\beta}}_{\textsf{2SLS}})-\boldsymbol{\beta}\|^2\right]\\
            &=\mathbb{E}\left[\|\textsf{clip}_{B_\beta}(\hat{\boldsymbol{\beta}}_{\textsf{2SLS}})-\boldsymbol{\beta}\|^2\middle|\mathcal{A}\right]\mathbb{P}\left\{\mathcal{A}\right\}+\mathbb{E}\left[\|\textsf{clip}_{B_\beta}(\hat{\boldsymbol{\beta}}_{\textsf{2SLS}})-\boldsymbol{\beta}\|^2\middle|\mathcal{A}^c\right]\mathbb{P}\left\{\mathcal{A}^c\right\}\\
            &\leq \mathbb{E}\left[\|\hat{\boldsymbol{\beta}}_{\textsf{2SLS}}-\boldsymbol{\beta}\|^2\middle|\mathcal{A}\right]\mathbb{P}\left\{\mathcal{A}\right\}+\mathbb{E}\left[\|\textsf{clip}_{B_\beta}(\hat{\boldsymbol{\beta}}_{\textsf{2SLS}})-\boldsymbol{\beta}\|^2\middle|\mathcal{A}^c\right]\mathbb{P}\left\{\mathcal{A}^c\right\}\\
            &\leq \mathbb{E}\left[\|\hat{\boldsymbol{\beta}}_{\textsf{2SLS}}-\boldsymbol{\beta}\|^2\middle|\mathcal{A}\right]+\mathbb{E}\left[\|\textsf{clip}_{B_\beta}(\hat{\boldsymbol{\beta}}_{\textsf{2SLS}})-\boldsymbol{\beta}\|^2\middle|\mathcal{A}^c\right]\cdot \xi,
        \end{split}
    \end{align}
    where 
    \begin{align} \label{eq:2sls squared error expectation part 2}
        \mathbb{E}\left[\|\textsf{clip}_{B_\beta}(\hat{\boldsymbol{\beta}}_{\textsf{2SLS}})-\boldsymbol{\beta}\|^2\middle|\mathcal{A}^c\right]\leq 4B_\beta^2,
    \end{align} and
    \begin{align}\label{eq:2sls squared error expectation part 1}
    \begin{split}
        &\mathbb{E}\left[\|\hat{\boldsymbol{\beta}}_{\textsf{2SLS}}-\boldsymbol{\beta}\|^2\middle|\mathcal{A}\right]\\
        &=\int_0^\infty\mathbb{P}\left\{\|\hat{\boldsymbol{\beta}}_{\textsf{2SLS}}-\boldsymbol{\beta}\|^2\geq \varepsilon\middle|\mathcal{A}\right\}d\varepsilon\\
            &=\int_0^\infty\mathbb{P}\left\{\|\hat{\boldsymbol{\beta}}_{\textsf{2SLS}}-\boldsymbol{\beta}\|\geq \sqrt{\varepsilon}\middle|\mathcal{A}\right\}d\varepsilon\\
            &\leq\int_{0}^\infty (q+1)\exp\left(-\frac{\lambda_{\tilde{x}}^2n^2\varepsilon/2}{B_{\hat{\Theta}}^2\nu(\boldsymbol{\Omega}_{z\epsilon_1}|\mathcal{A})+\lambda_{\tilde{x}} B_{\hat{\Theta}}B_{z\epsilon_1}n\sqrt{\varepsilon}/3}\right)d\varepsilon\\
            &\leq (q+1)\left[\int_0^c\exp\left(-\frac{\lambda_{\tilde{x}}^2n^2\varepsilon/2}{2B_{\hat{\Theta}}^2\nu(\boldsymbol{\Omega}_{z\epsilon_1}|\mathcal{A})}\right)d\varepsilon+\int_c^\infty\exp\left(-\frac{\lambda_{\tilde{x}}n\sqrt{\varepsilon}/2}{2 B_{\hat{\Theta}}B_{z\epsilon_1}/3}\right)d\varepsilon\right]\\
            &=(q+1)\left[\frac{4B_{\hat{\Theta}}^2\nu(\boldsymbol{\Omega}_{z\epsilon_1}|\mathcal{A})}{\lambda_{\tilde{x}}^2n^2}\left(1-\exp\left(-\frac{9\nu(\boldsymbol{\Omega}_{z\epsilon_1}|\mathcal{A})}{4B_{z\epsilon_1}^2}\right)\right)+\left(\frac{8B_{\hat{\Theta}}^2\nu(\boldsymbol{\Omega}_{z\epsilon_1}|\mathcal{A})}{\lambda_{\tilde{x}}^2n^2}+\frac{32B_{\hat{\Theta}}^2B_{z\epsilon_1}^2}{9\lambda_{\tilde{x}}^2n^2}\right)\exp\left(-\frac{9\nu(\boldsymbol{\Omega}_{z\epsilon_1}|\mathcal{A})}{4B_{z\epsilon_1}^2}\right)\right]\\
            &\leq (q+1)\left[\frac{4B_{\hat{\Theta}}^2\nu(\boldsymbol{\Omega}_{z\epsilon_1}|\mathcal{A})}{\lambda_{\tilde{x}}^2n^2}+\frac{8B_{\hat{\Theta}}^2\nu(\boldsymbol{\Omega}_{z\epsilon_1}|\mathcal{A})}{\lambda_{\tilde{x}}^2n^2}+\frac{32B_{\hat{\Theta}}^2B_{z\epsilon_1}^2}{9\lambda_{\tilde{x}}^2n^2}\right]\\
            &=\frac{(q+1)B_{\hat{\Theta}}^2}{\lambda_{\tilde{x}}^2n^2}\left[12\nu(\boldsymbol{\Omega}_{z\epsilon_1}|\mathcal{A})+\frac{32B_{z}^2B_{\epsilon_1}^2}{9}\right].
        \end{split}
    \end{align}
  
    Note that we further have the following bound:
    \begin{align}\label{eq:nu zepsilon1}
        \begin{split}
            \nu(\boldsymbol{\Omega}_{z\epsilon_1}|\mathcal{A})&=\max\left\{\bigg\|\mathbb{E}\Big[\Big(\sum_{i=1}^n\boldsymbol{z}_i\epsilon_{1,i}\Big)^\top\Big(\sum_{j=1}^n\boldsymbol{z}_j\epsilon_{1,j}\Big)\Big|\mathcal{A}\Big]\bigg\|,\bigg\|\mathbb{E}\Big[\Big(\sum_{i=1}^n\boldsymbol{z}_i\epsilon_{1,i}\Big)\Big(\sum_{j=1}^n\boldsymbol{z}_j\epsilon_{1,j}\Big)^\top\Big|\mathcal{A}\Big]\bigg\|\right\}\\
            &=\max\left\{\bigg\|\mathbb{E}\Big[\sum_{i=1}^n\epsilon_{1,i}^2\boldsymbol{z}_i^\top\boldsymbol{z}_i\Big|\mathcal{A}\Big]\bigg\|,\bigg\|\mathbb{E}\bigg[\sum_{i=1}^n\epsilon_{1,i}^2\boldsymbol{z}_i\boldsymbol{z}_i^\top\Big|\mathcal{A}\Big]\bigg\|\right\}\\
            &\leq nB_z^2\sigma_1^2.
        \end{split}
    \end{align}
    It now remains to determine $\lambda_z, \lambda_{\tilde{x}}$, and $B_{\hat{\Theta}}$. 

    From Theorem 4.6.1 of \citep{Vershynin_2018}\label{lemma:deviation}, when $n\geq c^2B_z^4(\sqrt{q}+\sqrt{t})^2$, with probability at least $1-2e^{-t}$, we have: 
    \begin{align*}
        \lambda_{\min}\left(\frac{1}{n}\boldsymbol{Z}^\top\boldsymbol{Z}\right)\geq \lambda_{\min}(\boldsymbol{\Sigma}_z)\left(1-\frac{cB_{z}^2(\sqrt{q}+\sqrt{t})}{\sqrt{n}}\right)^2,
    \end{align*}
    where c is an absolute constant. We rewrite the theorem by taking $t=\log(\frac{K}{q}n)-\frac{1}{2}$ (similar to the proof in Lemma \ref{lemma:deviation projected}). Then we need the following condition to be satisfied:
    \begin{align}\label{eq:n condition}
        n\geq c^2B_z^4\left(\sqrt{q}+\sqrt{\log\left({\frac{K}{q}n}\right)-\frac{1}{2}}\right)^2.
    \end{align}
    We can bound the RHS of $\eqref{eq:n condition}$ as follows:
    \begin{align*}
        c^2B_z^4\left(\sqrt{q}+\sqrt{\log\left({\frac{K}{q}n}\right)-\frac{1}{2}}\right)^2&\leq 2c^2B_z^4\left(q+\log\left(\frac{K}{q}n\right)-\frac{1}{2}\right)\\
        &\leq \frac{n}{2}+2c^2B_z^4\left(q+\log\left({\frac{4c^2B_z^4K}{q}}\right)-\frac{3}{2}\right),
    \end{align*}
    where the second line follows from the inequality $\log(x)\leq \frac{x}{C}+\log(C)-1$, with $x=\frac{K}{q}n$ and $C=\frac{4c^2B_z^4K}{q}$.
    So a sufficient condition for \eqref{eq:n condition} to hold is:
    \begin{align}\label{eq:n condition final form}
        n\geq 4c^2B_z^4\left(q+\log\left(\frac{4c^2B_z^4K}{q}\right)-\frac{3}{2}\right).
    \end{align}
    With condition \eqref{eq:n condition final form}, we have the following bound holds with probability at least $1-\frac{2qe^{\frac{1}{2}}}{Kn}$:
     \begin{align}\label{eq:lambdaz}
        \lambda_{\min}\left(\frac{1}{n}\boldsymbol{Z}^\top\boldsymbol{Z}\right)\geq \lambda_{\min}(\boldsymbol{\Sigma}_z)\left(1-\frac{cB_{z}^2\left(\sqrt{q}+\sqrt{\log\left(\frac{K}{q}n\right)-\frac{1}{2}}\right)}{\sqrt{n}}\right)^2:=\lambda_z.
    \end{align}
    Furthermore, with Lemma \ref{lemma:deviation projected}, when $n\geq \max\left\{\frac{qe^{\frac{3}{2}}}{K}, \frac{p^2(q+1)^2K}{qK_0^2}\right\}$, we have the following holds with probability at least $1-\frac{5qe^{\frac{1}{2}}}{Kn}$:
     \begin{align}\label{eq:lambdaxtilde}
        \begin{aligned}
         \lambda_{\min}\left(\frac{1}{n}\boldsymbol{X}^\top\boldsymbol{P}_Z\boldsymbol{X}\right)&\geq \lambda_z \left(\sigma_{\min}(\boldsymbol{\Theta})-\sqrt{\frac{2p(q+1)B_{\epsilon_2}^2\log\left(\frac{K}{q}n\right)}{\lambda_{\min}(\boldsymbol{\Sigma}_z)n}}\right)^2\\
         &=\lambda_{\min}(\boldsymbol{\Sigma}_z)\left(1-\frac{cB_{z}^2\left(\sqrt{q}+\sqrt{\log\left(\frac{K}{q}n\right)-\frac{1}{2}}\right)}{\sqrt{n}}\right)^2\left(\sigma_{\min}(\boldsymbol{\Theta})-\sqrt{\frac{2p(q+1)B_{\epsilon_2}^2\log\left(\frac{K}{q}n\right)}{\lambda_{\min}(\boldsymbol{\Sigma}_z)n}}\right)^2\\
         &:=\lambda_{\tilde{x}}.
        \end{aligned}
    \end{align}

    From \eqref{eq:bound on Psi} and \eqref{eq:Theta hat}, we have: 
    \begin{align} \label{eq:bound on BhatTheta}
        B_{\hat{\Theta}}=B_{\Theta}+\sqrt{\frac{2p(q+1)B_{\epsilon_2}^2\log\left(\frac{K}{q}n\right)}{\lambda_{\min}(\boldsymbol{\Sigma}_z)n}}.
    \end{align}

    With $\xi=\frac{5qe^{\frac{1}{2}}}{Kn}$, putting together \eqref{eq:nu zepsilon1}\eqref{eq:lambdaz}\eqref{eq:lambdaxtilde}\eqref{eq:bound on BhatTheta} into \eqref{eq:2sls squared error expectation part 1}, and \eqref{eq:2sls squared error expectation part 2} \eqref{eq:2sls squared error expectation part 1} into \eqref{eq:2sls squared error expectation} completes the proof.
    
\end{proof}

\section{Proofs For Section \ref{sec:mainresults}}
   \subsection{Proof of Theorem \ref{thm:2SLS GD}}\label{proof:thm:2SLS GD}
	\begin{lemma}\label{lemma:convergence}
		Suppose $\{\boldsymbol{\Omega}^{(1)}, \ldots, \boldsymbol{\Omega}^{(t)},\ldots\}$ is a $d\times d$-matrix sequence decaying with exponential rate $r$, i.e. for some constant $c>0$ and $0<r<1$,
		\begin{align*}
			\left\|\boldsymbol{\Omega}^{(t)}\right\|_F\leq cr^t.
		\end{align*} 
		Then for any $\varepsilon>0$, there exists a finite constant:
		\begin{align*}
			T_0 = \left\lceil \log_{r}\frac{(1-r)(\varepsilon/d)}{c\left(1+(1-r)(\varepsilon/d)\right)}\right\rceil,
		\end{align*}
		such that \begin{align*}\left\|\prod_{t=T_0}^\infty (\boldsymbol{I}+\boldsymbol{\Omega}^{(t)}) -\boldsymbol{I}\right\|_F<\varepsilon,\end{align*}
		and hence \begin{align*}\left\|\prod_{t=T_0}^\infty (\boldsymbol{I}+\boldsymbol{\Omega}^{(t)})\right\|_F<\sqrt{d}+\varepsilon.\end{align*}
   \end{lemma}

   \begin{proof}
	By definition,	
	\begin{align*}
		\left\|\boldsymbol{\Omega}^{(k)}\right\|_F = \sqrt{\sum_{i,j=1}^p\boldsymbol{\Omega}_{ij}^{(k)2}}\leq cr^k,
	\end{align*}
	which implies: 
	\begin{align*}
		\left|\boldsymbol{\Omega}_{ij}^{(k)}\right|\leq cr^k, \quad\forall i,j,k.
	\end{align*}
	Consider the product of any two matrices. By sub-multiplicativity,
	\begin{align*}
		\left\|\boldsymbol{\Omega}^{(k)}\boldsymbol{\Omega}^{(l)}\right\|_F\leq \left\|\boldsymbol{\Omega}^{(k)}\right\|_F\left\|\boldsymbol{\Omega}^{(l)}\right\|_F\leq c^2r^{k+l},
	\end{align*}
	which implies:
	\begin{align*}
		\left|\left[\boldsymbol{\Omega}^{(k)}\boldsymbol{\Omega}^{(l)}\right]_{ij}\right|\leq c^2r^{k+l}, \quad\forall i,j, k, l.
	\end{align*}
	Similarly, for the product of any number of matrices:
	\begin{align*}
		\left|\left[\boldsymbol{\Omega}^{(k_1)}\boldsymbol{\Omega}^{(k_2)}\cdots \boldsymbol{\Omega}^{(k_n)}\right]_{ij}\right|\leq c^n r^{k_1+k_2+\cdots+k_n}, \quad\forall i,j,k_1,\ldots,k_n.
	\end{align*}
	Thus 
	\begin{align}\label{diffIW}
		\begin{split}
		&\left\|\prod_{t=t_1}^{t_2}\left(\boldsymbol{I}+\boldsymbol{\Omega}^{(t)}\right)-\boldsymbol{I}\right\|_F\\
		&=\left\|\left(\boldsymbol{I}+\boldsymbol{\Omega}^{(t_1)}\right)\left(\boldsymbol{I}+\boldsymbol{\Omega}^{(t_1+1)}\right)\cdots\left(\boldsymbol{I}+\boldsymbol{\Omega}^{(t_2)}\right)-\boldsymbol{I}\right\|_F\\
		&=\left\|\sum_{t_1\leq k\leq t_2}\boldsymbol{\Omega}^{(k)}+\sum_{t_1\leq k<l\leq t_2}\boldsymbol{\Omega}^{(k)}\boldsymbol{\Omega}^{(l)}+\cdots+\boldsymbol{\Omega}^{(t_1)}\boldsymbol{\Omega}^{(t_1+1)}\cdots \boldsymbol{\Omega}^{(t_2)}\right\|_F\\
		&\leq \left\|\sum_{t_1\leq k\leq t_2}cr^k\boldsymbol{1}\boldsymbol{1}^\top+\sum_{t_1\leq k<l\leq t_2}c^2r^{k+l}\boldsymbol{1}\boldsymbol{1}^\top+\cdots+c^{t2-t1+1}r^{t_1+\cdots+t_2}\boldsymbol{1}\boldsymbol{1}^\top\right\|_F.
		\end{split}
	\end{align}
   Note that the last inequality can be checked by comparing matrix elements of both sides. For any $\varepsilon>0$, we take $T_0 = \lceil \log_{r}\frac{(1-r)(\varepsilon/d)}{c(1+(1-r)(\varepsilon/d))}\rceil$. 
	Consider $t_1=T_0$ and $t_2\rightarrow\infty$ in (\ref{diffIW}). For notation convenience, let 
	\begin{align*}\boldsymbol{\Xi}:=\sum_{T_0\leq k}cr^k\boldsymbol{1}\boldsymbol{1}^\top+\sum_{T_0\leq k<l}c^2r^{k+l}\boldsymbol{1}\boldsymbol{1}^\top+\sum_{T_0\leq k<l<m}c^3r^{k+l+m}\boldsymbol{1}\boldsymbol{1}^\top+\cdots.\end{align*}
	Then 
	\begin{align*}
		\begin{split}
			\boldsymbol{\Xi}_{ij} 
			&= \sum_{T_0\leq k} cr^k+\sum_{T_0\leq k<l}c^2r^{k+l}+\sum_{T_0\leq k<l<m}c^3r^{k+l+m}+\cdots\\
			&<c\sum_{k\geq T_0} r^k+c^2r^{T_0}\sum_{k\geq T_0}r^k+c^3r^{2T_0}\sum_{k\geq T_0} r^k+\cdots\\
			&=\frac{cr^{T_0}}{1-r}+\frac{c^2r^{2T_0}}{1-r}+\frac{c^3r^{3T_0}}{1-r}+\cdots\\
			&=\frac{cr^{T_0}}{(1-r)(1-cr^{T_0})}\\
			&\leq\frac{\varepsilon}{d}.
		\end{split}
	\end{align*}
	Thus
	\begin{align*}
		\begin{split}
			\left\|\prod_{t=T_0}^{\infty}(\boldsymbol{I}+\boldsymbol{\Omega}^{(t)})-\boldsymbol{I}\right\|_F=\left\|\boldsymbol{\Xi}\right\|_F=\sqrt{\sum_{i,j=1}^d\boldsymbol{\Xi}_{ij}^2}\leq\varepsilon.
		\end{split}
	\end{align*}
	Hence completes the proof.
\end{proof}

\begin{proof}[Proof of Theorem \ref{thm:2SLS GD}]
	In the following proof, we treat $\boldsymbol{Z},\boldsymbol{X},\boldsymbol{Y}$ as deterministic matrices.

	We begin by checking the inner loop (\ref{eq:InnerLoopUpdate}):
	\begin{align*}
		 \begin{split}
		  \boldsymbol{\Theta}^{(t)} - \hat{\boldsymbol{\Theta}} &= \boldsymbol{\Theta}^{(t-1)} - \hat{\boldsymbol{\Theta}} - \eta \boldsymbol{Z}^\top\left(\boldsymbol{Z\Theta}^{(t-1)}-\boldsymbol{X}\right) \\
		  &= \left(\boldsymbol{I}-\eta \boldsymbol{Z}^\top\boldsymbol{Z}\right)\left(\boldsymbol{\Theta}^{(t-1)} - \hat{\boldsymbol{\Theta}}\right) + \eta\boldsymbol{Z}^\top\left(\boldsymbol{X}-\boldsymbol{Z\hat{\Theta}}\right)\\
		  &= \left(\boldsymbol{I} - \eta \boldsymbol{Z}^\top\boldsymbol{Z}\right)^2\left(\boldsymbol{\Theta}^{(t-2)}-\hat{\boldsymbol{\Theta}}\right)+\eta \boldsymbol{Z}^\top\left(\boldsymbol{X}-\boldsymbol{Z\hat{\Theta}}\right) + \eta\left(\boldsymbol{I} - \eta\boldsymbol{Z}^\top\boldsymbol{Z}\right)\boldsymbol{Z}^\top\left(\boldsymbol{X}-\boldsymbol{Z\hat{\Theta}}\right)\\
		  &\quad\quad\quad\quad\quad\quad\quad\quad\vdots\\
		  &= \left(\boldsymbol{I} - \eta\boldsymbol{Z}^\top\boldsymbol{Z}\right)^{t}\left(\boldsymbol{\Theta}^{(0)} - \hat{\boldsymbol{\Theta}}\right) + \sum_{i=0}^{t-1}\eta\left(\boldsymbol{I}-\eta\boldsymbol{Z}^\top\boldsymbol{Z}\right)^{t-1-i}\boldsymbol{Z}^\top\left(\boldsymbol{X}-\boldsymbol{Z\hat{\Theta}}\right)\\
		  &=\left(\boldsymbol{I} - \eta\boldsymbol{Z}^\top\boldsymbol{Z}\right)^{t}\left(\boldsymbol{\Theta}^{(0)}-\hat{\boldsymbol{\Theta}}\right)+\eta\left[\boldsymbol{I}-(\boldsymbol{I}-\eta \boldsymbol{Z}^\top\boldsymbol{Z})^{t}\right]\left(\eta\boldsymbol{Z}^\top\boldsymbol{Z}\right)^{-1}\boldsymbol{Z}^\top\left(\boldsymbol{X}-\boldsymbol{Z\hat{\Theta}}\right)\\
		  &=\left(\boldsymbol{I} - \eta\boldsymbol{Z}^\top\boldsymbol{Z}\right)^{t}\left(\boldsymbol{\Theta}^{(0)}-\hat{\boldsymbol{\Theta}}\right)+\left[\boldsymbol{I}-(\boldsymbol{I}-\eta\boldsymbol{Z}^\top\boldsymbol{Z})^{t}\right]\left[(\boldsymbol{Z}^\top\boldsymbol{Z})^{-1}\boldsymbol{Z}^\top\boldsymbol{X}-\hat{\boldsymbol{\Theta}}\right]\\
		  &=\left(\boldsymbol{I} - \eta\boldsymbol{Z}^\top\boldsymbol{Z}\right)^{t}\left(\boldsymbol{\Theta}^{(0)}-\hat{\boldsymbol{\Theta}}\right).
		 \end{split}
	\end{align*}
	With learning rate $0<\eta<\frac{2}{\sigma_{\max}^2(\boldsymbol{Z})}$, let $\kappa(\eta) := \rho\left(\boldsymbol{I} - \eta\boldsymbol{Z}^\top\boldsymbol{Z}\right)$, where $\rho(\cdot)$ denotes the spectral radius. Then it follows that $0<\kappa(\eta)<1$. We have:
   \begin{align*}
         \left\|\boldsymbol{\Theta}^{(t)} - \hat{\boldsymbol{\Theta}}\right\| &=\left\|(\boldsymbol{I} - \eta\boldsymbol{Z}^\top\boldsymbol{Z})^{t}(\boldsymbol{\Theta}^{(0)}-\hat{\boldsymbol{\Theta}}) \right\| \nonumber\\
         &\leq \left\|(\boldsymbol{I} - \eta\boldsymbol{Z}^\top\boldsymbol{Z})^{t}\right\|\left\|\boldsymbol{\Theta}^{(0)}-\hat{\boldsymbol{\Theta}}\right\|\nonumber\\
         &\leq \kappa(\eta)^{t}\left\|\boldsymbol{\Theta}^{(0)}-\hat{\boldsymbol{\Theta}}\right\|\nonumber\\
         &=\mathcal{O}(\kappa(\eta)^{t}).
   \end{align*}
   Thus $\{\boldsymbol{\Theta}^{(t)}\}$ converges to $\hat{\boldsymbol{\Theta}}$ exponentially with rate $\kappa(\eta)$.

   For the outer loop (\ref{eq:OuterLoopUpdate}), we have:
    \begin{align}\label{eq:diffbeta}
		\begin{split}
			\boldsymbol{\beta}^{(t)} - \hat{\boldsymbol{\beta}}_{\textsf{2SLS}} &= \boldsymbol{\beta}^{(t-1)} - \hat{\boldsymbol{\beta}}_{\textsf{2SLS}} - \alpha\boldsymbol{\Theta}^{(t-1)\top}\boldsymbol{Z}^\top\left(\boldsymbol{Z\Theta}^{(t-1)}\boldsymbol{\beta}^{(t-1)}-\boldsymbol{Y}\right)\\
		&= \left(\boldsymbol{I}-\alpha\boldsymbol{\Theta}^{(t-1)\top}\boldsymbol{Z}^\top\boldsymbol{Z}\boldsymbol{\Theta}^{(t-1)}\right)\left(\boldsymbol{\beta}^{(t-1)}-\hat{\boldsymbol{\beta}}_{\textsf{2SLS}}\right) + \alpha\boldsymbol{\Theta}^{(t-1)\top}\boldsymbol{Z}^\top\left(\boldsymbol{Y}-\boldsymbol{Z\Theta}^{(t-1)}\hat{\boldsymbol{\beta}}_{\textsf{2SLS}}\right)\\
		&\quad\quad\quad\quad\quad\quad\quad\quad\vdots\\
		&=\underbrace{\prod_{i=0}^{t-1}\left(\boldsymbol{I}-\alpha\boldsymbol{\Theta}^{(i)\top}\boldsymbol{Z}^\top\boldsymbol{Z\Theta}^{(i)}\right)\left(\boldsymbol{\beta}^{(0)}-\hat{\boldsymbol{\beta}}_{\textsf{2SLS}}\right)}_{\boldsymbol{\Delta}_1\boldsymbol{\beta}^{(t)}}\\
		&\quad+\underbrace{\sum_{i=0}^{t-1}\alpha\left[\prod_{j=i+1}^{t-1}\left(\boldsymbol{I} - \alpha\boldsymbol{\Theta}^{(j)\top}\boldsymbol{Z}^\top\boldsymbol{Z\Theta}^{(j)}\right) \right]\boldsymbol{\Theta}^{(i)\top}\boldsymbol{Z}^\top\left(\boldsymbol{Y}-\boldsymbol{Z\Theta}^{(i)}\hat{\boldsymbol{\beta}}_{\textsf{2SLS}}\right)}_{\boldsymbol{\Delta}_2\boldsymbol{\beta}^{(t)}}.
		\end{split}
    \end{align}
    To simplify notations, let 
    \begin{align*}
		\boldsymbol{R}^{(t)}&\coloneqq\boldsymbol{\Theta}^{(t)} - \hat{\boldsymbol{\Theta}} = \left(\boldsymbol{I}-\eta\boldsymbol{Z}^\top\boldsymbol{Z}\right)^t\left(\boldsymbol{\Theta}^{(0)} - \hat{\boldsymbol{\Theta}}\right),\\
		\boldsymbol{V}^{(t)}&\coloneqq\left(\boldsymbol{I}-\alpha\hat{\boldsymbol{\Theta}}^\top\boldsymbol{Z}^\top\boldsymbol{Z}\hat{\boldsymbol{\Theta}}\right)^{t},\\
		\boldsymbol{W}^{(t)}&\coloneqq\boldsymbol{R}^{(t)\top}\boldsymbol{Z}^\top\boldsymbol{Z}\hat{\boldsymbol{\Theta}}+\hat{\boldsymbol{\Theta}}^\top\boldsymbol{Z}^\top\boldsymbol{Z}\boldsymbol{R}^{(t)}+\boldsymbol{R}^{(t)\top}\boldsymbol{Z}^\top\boldsymbol{Z}\boldsymbol{R}^{(t)}.
	\end{align*}
    With learning rates $0<\alpha<\frac{2}{\sigma_{\max}^2(\boldsymbol{Z}\hat{\boldsymbol{\Theta}})}$, $0<\eta<\frac{2}{\sigma_{\max}^2(\boldsymbol{Z})}$, let $\gamma(\alpha):=\rho\left(\boldsymbol{I}-\alpha\hat{\boldsymbol{\Theta}}^\top\boldsymbol{Z}^\top\boldsymbol{Z}\hat{\boldsymbol{\Theta}}\right)$. Then it follows that $0<\gamma(\alpha)<1$. We have:
    \begin{align}\label{eq:RBound}
        \left\|\boldsymbol{R}^{(t)}\right\| \leq \kappa(\eta)^{t}\left\|\boldsymbol{\Theta}^{(0)}-\hat{\boldsymbol{\Theta}}\right\|,
    \end{align}
    \begin{align*}
        \left\|\boldsymbol{V}^{(t)}\right\| \leq \gamma(\alpha)^{t},\end{align*}
        and
   \begin{align*}
		\begin{split}
			\left\|\boldsymbol{W}^{(t)}\right\|&=\left\|\boldsymbol{R}^{(t)\top}\boldsymbol{Z}^\top\boldsymbol{Z}\hat{\boldsymbol{\Theta}}+\hat{\boldsymbol{\Theta}}^\top\boldsymbol{Z}^\top\boldsymbol{ZR}^{(t)}+\boldsymbol{R}^{(t)\top}\boldsymbol{Z}^\top\boldsymbol{ZR}^{(t)}\right\|\\
			&\leq 2\left\|\hat{\boldsymbol{\Theta}}^\top\boldsymbol{Z}^\top\boldsymbol{Z}\right\|\left\|\boldsymbol{R}^{(t)}\right\| +\left\| \boldsymbol{Z}^\top\boldsymbol{Z}\right\|\left\|\boldsymbol{R}^{(t)}\right\|^2\\
			&\leq 2\kappa(\eta)^{t}\left\|\hat{\boldsymbol{\Theta}}\boldsymbol{Z}^\top\boldsymbol{Z}\right\|\left\|\boldsymbol{\Theta}^{(0)}-\hat{\boldsymbol{\Theta}}\right\| +\kappa(\eta)^{2t}\left\|\boldsymbol{Z}^\top\boldsymbol{Z}\right\|\left\|\boldsymbol{\Theta}^{(0)}-\hat{\boldsymbol{\Theta}}\right\|^2\\
			&\leq \kappa(\eta)^t \left(2\left\|\hat{\boldsymbol{\Theta}}\boldsymbol{Z}^\top\boldsymbol{Z}\right\|\left\|\boldsymbol{\Theta}^{(0)}-\hat{\boldsymbol{\Theta}}\right\|+\left\|\boldsymbol{Z}^\top\boldsymbol{Z}\right\|\left\|\boldsymbol{\Theta}^{(0)}-\hat{\boldsymbol{\Theta}}\right\|^2\right)\\
			&=\mathcal{O}(\kappa(\eta)^t).
	  	\end{split}
	\end{align*}
	Then from equation (\ref{eq:diffbeta}), we have:
    \begin{align*}
        \begin{split}
			\boldsymbol{\Delta}_1\boldsymbol{\beta}^{(t)}&=\prod_{i=0}^{t-1}\left(\boldsymbol{I}-\alpha\boldsymbol{\Theta}^{(i)\top}\boldsymbol{Z}^\top\boldsymbol{Z\Theta}^{(i)}\right)\left(\boldsymbol{\beta}^{(0)}-\hat{\boldsymbol{\beta}}_{\textsf{2SLS}}\right)\\
			&=\prod_{i=0}^{t-1}\left[\boldsymbol{I}-\alpha \hat{\boldsymbol{\Theta}}^\top\boldsymbol{Z}^\top\boldsymbol{Z}\hat{\boldsymbol{\Theta}}-\alpha\left(\boldsymbol{R}^{(i)\top}\boldsymbol{Z}^\top\boldsymbol{Z}\hat{\boldsymbol{\Theta}}+\hat{\boldsymbol{\Theta}}^\top\boldsymbol{Z}^\top\boldsymbol{Z}\boldsymbol{R}^{(i)}+\boldsymbol{R}^{(i)\top}\boldsymbol{Z}^\top\boldsymbol{Z}\boldsymbol{R}^{(i)}\right)\right]\left(\boldsymbol{\beta}^{(0)}-\hat{\boldsymbol{\beta}}_{\textsf{2SLS}}\right)\\
			&=\prod_{i=0}^{t-1}\left[\boldsymbol{I}-\alpha\hat{\boldsymbol{\Theta}}^\top\boldsymbol{Z}^\top\boldsymbol{Z}\hat{\boldsymbol{\Theta}}-\alpha\boldsymbol{W}^{(i)}\right]\left(\boldsymbol{\beta}^{(0)}-\hat{\boldsymbol{\beta}}_{\textsf{2SLS}}\right)\\
			&=\left(\boldsymbol{I}-\alpha \hat{\boldsymbol{\Theta}}^\top\boldsymbol{Z}^\top\boldsymbol{Z}\hat{\boldsymbol{\Theta}}\right)^{t}\prod_{i=0}^{t-1}\left[\boldsymbol{I}-\alpha\left(\boldsymbol{I}-\alpha \hat{\boldsymbol{\Theta}}^\top\boldsymbol{Z}^\top\boldsymbol{Z}\hat{\boldsymbol{\Theta}}\right)^{-1}\boldsymbol{W}^{(i)}\right]\left(\boldsymbol{\beta}^{(0)}-\hat{\boldsymbol{\beta}}_{\textsf{2SLS}}\right)\\
			&=\boldsymbol{V}^{(t)}\prod_{i=0}^{t-1}\left[\boldsymbol{I}-\alpha\left(\boldsymbol{I}-\alpha \hat{\boldsymbol{\Theta}}^\top\boldsymbol{Z}^\top\boldsymbol{Z}\hat{\boldsymbol{\Theta}}\right)^{-1}\boldsymbol{W}^{(i)}\right]\left(\boldsymbol{\beta}^{(0)}-\hat{\boldsymbol{\beta}}_{\textsf{2SLS}}\right).
        \end{split}
	\end{align*}
	We denote $\boldsymbol{\Psi} := \alpha\left(\boldsymbol{I}-\alpha\hat{\boldsymbol{\Theta}}^\top\boldsymbol{Z}^\top\boldsymbol{Z}\hat{\boldsymbol{\Theta}}\right)^{-1}$. By Lemma ~\ref{lemma:convergence}, we take $\varepsilon=1$, $c_0$ be a constant such that $\left\|\boldsymbol{W}^{(t)}\right\|_F\leq c_0\kappa(\eta)^t$, and $T_0=\lceil \log_{\kappa(\eta)}\frac{(1-\kappa(\eta))}{\|\Psi\|_Fc_0(p+(1-\kappa(\eta)))}\rceil$.

	Then we have:
    \begin{align}\label{eq:ProductBound}
        \left\|\prod_{i=T_0}^{t-1}\left(\boldsymbol{I}-\boldsymbol{\Psi W}^{(i)}\right)\right\|\leq \left\|\prod_{i=T_0}^{t-1}\left(\boldsymbol{I}-\boldsymbol{\Psi W}^{(i)}\right)\right\|_F< \sqrt{p}+1.
    \end{align}
    Hence
	\begin{align}\label{eq:Delta1betaBound}
        \begin{split}
            \left\|\boldsymbol{\Delta}_1\boldsymbol{\beta}^{(t)}\right\| &= \left\|\boldsymbol{V}^{(t)}\prod_{i=0}^{t-1}(\boldsymbol{I}-\boldsymbol{\Psi W}^{(i)})(\boldsymbol{\beta}^{(0)}-\hat{\boldsymbol{\beta}}_{\textsf{2SLS}})\right\|\\
            &\leq \left\|\boldsymbol{V}^{(t)}\right\|\left\|\prod_{i=0}^{T_0-1}\left(\boldsymbol{I}-\boldsymbol{\Psi W}^{(i)}\right)\right\|\left\|\prod_{i=T_0}^{t-1}(\boldsymbol{I}-\boldsymbol{\Psi W}^{(i)})\right\|\left\|\boldsymbol{\beta}^{(0)}-\hat{\boldsymbol{\beta}}_{\textsf{2SLS}}\right\|\\
            &<\gamma(\alpha)^{t}\left\|\prod_{i=0}^{T_0-1}\left(\boldsymbol{I}-\boldsymbol{\Psi W}^{(i)}\right)\right\|\left(\sqrt{p}+1\right)\left\|\boldsymbol{\beta}^{(0)}-\hat{\boldsymbol{\beta}}_{\textsf{2SLS}}\right\|\\
            &=\mathcal{O}(\gamma(\alpha)^{t}).
        \end{split}
    \end{align}
	Next we consider $\boldsymbol{\Delta}_2\boldsymbol{\beta}^{(t)}$:
    \begin{align*}
		\begin{split}
			\boldsymbol{\Delta}_2\boldsymbol{\beta}^{(t)} &= \sum_{i=0}^{t-1}\alpha\left[\prod_{j=i+1}^{t-1}\left(\boldsymbol{I} - \alpha\boldsymbol{\Theta}^{(j)\top}\boldsymbol{Z}^\top\boldsymbol{Z\Theta}^{(j)}\right) \right]\boldsymbol{\Theta}^{(i)\top}\boldsymbol{Z}^\top\left(\boldsymbol{Y}-\boldsymbol{Z\Theta}^{(i)}\hat{\boldsymbol{\beta}}_{\textsf{2SLS}}\right)\\
			&=\sum_{i=0}^{t-1}\alpha\left[\prod_{j=i+1}^{t-1}\left(\boldsymbol{I} - \alpha\hat{\boldsymbol{\Theta}}^\top\boldsymbol{Z}^\top\boldsymbol{Z}\hat{\boldsymbol{\Theta}}-\alpha\boldsymbol{W}^{(j)}\right)\right]\left(\boldsymbol{R}^{(i)} + \hat{\boldsymbol{\Theta}}\right)^\top\boldsymbol{Z}^\top\left[\boldsymbol{Y}-\boldsymbol{Z}\left(\boldsymbol{R}^{(i)} + \hat{\boldsymbol{\Theta}}\right)\hat{\boldsymbol{\beta}}_{\textsf{2SLS}}\right]\\
			&=\sum_{i=0}^{t-1}\alpha\left(\boldsymbol{I}-\alpha\hat{\boldsymbol{\Theta}}^\top\boldsymbol{Z}^\top\boldsymbol{Z}\hat{\boldsymbol{\Theta}}\right)^{t-1-i}\prod_{j=i+1}^{t-1}\left[\boldsymbol{I} -\alpha\left(\boldsymbol{I}-\alpha\hat{\boldsymbol{\Theta}}^\top\boldsymbol{Z}^\top\boldsymbol{Z}\hat{\boldsymbol{\Theta}}\right)^{-1}\boldsymbol{W}^{(j)}\right]\\&\quad\quad\quad\quad\cdot\left(\boldsymbol{R}^{(i)}+\hat{\boldsymbol{\Theta}}\right)^\top\boldsymbol{Z}^\top\left[\boldsymbol{Y}-\boldsymbol{Z}\left(\boldsymbol{R}^{(i)} + \hat{\boldsymbol{\Theta}}\right)\hat{\boldsymbol{\beta}}_{\textsf{2SLS}}\right]\\
			&=\sum_{i=0}^{t-1}\alpha\left[\boldsymbol{V}^{(t-1-i)}\prod_{j=i+1}^{t-1}\left(\boldsymbol{I} -\boldsymbol{\Psi}\boldsymbol{W}^{(j)}\right)\right]\left(\boldsymbol{R}^{(i)}+\hat{\boldsymbol{\Theta}}\right)^\top\boldsymbol{Z}^\top\left[\boldsymbol{Y}-\boldsymbol{Z}\left(\boldsymbol{R}^{(i)} + \hat{\boldsymbol{\Theta}}\right)\hat{\boldsymbol{\beta}}_{\textsf{2SLS}}\right].
		\end{split} 
   \end{align*}
   For convenience, let $\boldsymbol{\Delta}_2\boldsymbol{\beta}^{(t)}:=\boldsymbol{\Delta}_{21}\boldsymbol{\beta}^{(t)}+\boldsymbol{\Delta}_{22}\boldsymbol{\beta}^{(t)}$, where
    \begin{align*}
        \boldsymbol{\Delta}_{21}\boldsymbol{\beta}^{(t)}:&=\sum_{i=0}^{t-1}\alpha\left[\boldsymbol{V}^{(t-1-i)}\prod_{j=i+1}^{t-1}\left(\boldsymbol{I} -\boldsymbol{\Psi}\boldsymbol{W}^{(j)}\right)\right]\boldsymbol{R}^{(i)\top}\boldsymbol{Z}^\top\left[\boldsymbol{Y}-\boldsymbol{Z}\left(\boldsymbol{R}^{(i)} + \hat{\boldsymbol{\Theta}}\right)\hat{\boldsymbol{\beta}}_{\textsf{2SLS}}\right],\\
        \boldsymbol{\Delta}_{22}\boldsymbol{\beta}^{(t)}:&=\sum_{i=0}^{t-1}\alpha\left[\boldsymbol{V}^{(t-1-i)}\prod_{j=i+1}^{t-1}\left(\boldsymbol{I} -\boldsymbol{\Psi}\boldsymbol{W}^{(j)}\right)\right]\hat{\boldsymbol{\Theta}}^\top\boldsymbol{Z}^\top\left[\boldsymbol{Y}-\boldsymbol{Z}\left(\boldsymbol{R}^{(i)} + \hat{\boldsymbol{\Theta}}\right)\hat{\boldsymbol{\beta}}_{\textsf{2SLS}}\right].
    \end{align*}
    Suppose $\tilde{M}_1,\tilde{M}_2$ are the upper bounds such that 
    \begin{align*}
        \left\|\boldsymbol{Z}^\top\left[\boldsymbol{Y}-\boldsymbol{Z}\left(\boldsymbol{R}^{(i)} + \hat{\boldsymbol{\Theta}}\right)\hat{\boldsymbol{\beta}}_{\textsf{2SLS}}\right]\right\|\leq \tilde{M}_1,\quad \forall i=0,\ldots,t-1,
    \end{align*}
    \begin{align*}
        \left\|\prod_{j=i+1}^{t-1}\left(\boldsymbol{I} -\boldsymbol{\Psi}\boldsymbol{W}^{(j)}\right)\right\|\leq \tilde{M}_2, \quad\forall i=0,\ldots,t-1.
    \end{align*}
    We know such $\tilde{M}_1,\tilde{M}_2$ exist because of the bounds given by (\ref{eq:RBound}) and (\ref{eq:ProductBound}). Let $\tilde{M}=\tilde{M}_1\tilde{M}_2$. Then
	\begin{align*}
        \begin{split}
            \left\|\boldsymbol{\Delta}_{21}\boldsymbol{\beta}^{(t)}\right\|&\leq \tilde{M}\left\|\sum_{i=0}^{t-1}\alpha\boldsymbol{V}^{(t-1-i)}\boldsymbol{R}^{(i)\top}\right\|\\
            &\leq \tilde{M}\alpha\sum_{i=0}^{t-1}\left\|\boldsymbol{V}^{(t-1-i)}\right\|\left\|\boldsymbol{R}^{(i)}\right\|\\
            &\leq \tilde{M}\alpha\left\|\boldsymbol{\Theta}^{(0)}-\hat{\boldsymbol{\Theta}}\right\|\sum_{i=0}^{t-1}\gamma(\alpha)^{t-1-i}\kappa(\eta)^i\\
            &= \tilde{M}\alpha\left\|\boldsymbol{\Theta}^{(0)}-\hat{\boldsymbol{\Theta}}\right\|\sum_{i=0}^{t-1}\gamma(\alpha)^{t-1}\left(\frac{\kappa(\eta)}{\gamma(\alpha)}\right)^i\\
            &=\mathcal{O}\left(\frac{\gamma(\alpha)^t-\kappa(\eta)^t}{\gamma(\alpha)-\kappa(\eta)}\right)\\
            &\leq\mathcal{O}(\max\{\gamma(\alpha)^t,\kappa(\eta)^t\}),
        \end{split} 
    \end{align*}
	and similarly,
    \begin{align*}
      \begin{split}
          \left\|\boldsymbol{\Delta}_{22}\boldsymbol{\beta}^{(t)}\right\|&\leq \tilde{M}\left\|\sum_{i=0}^{t-1}\alpha\boldsymbol{V}^{(t-1-i)}\hat{\boldsymbol{\Theta}}^{\top}\right\|\\
          &\leq \tilde{M}\alpha\left\|\hat{\boldsymbol{\Theta}}\right\|\sum_{i=0}^{t-1}\left\|\boldsymbol{V}^{(t-1-i)}\right\|\\
        &\leq \tilde{M}\alpha\left\|\hat{\boldsymbol{\Theta}}\right\|\sum_{i=0}^{t-1}\gamma(\alpha)^{t-1-i}\\
        &=\mathcal{O}(\gamma(\alpha)^t).
        \end{split}
    \end{align*}
	Thus
	\begin{align}\label{eq:Delta2betaBound}
		\begin{split}
			\left\|\boldsymbol{\Delta}_2\boldsymbol{\beta}^{(t)}\right\|&= \left\|\boldsymbol{\Delta}_{21}\boldsymbol{\beta}^{(t)}+\boldsymbol{\Delta}_{22}\boldsymbol{\beta}^{(t)}\right\|\\
			&\leq \left\|\boldsymbol{\Delta}_{21}\boldsymbol{\beta}^{(t)}\right\|+\left\|\boldsymbol{\Delta}_{22}\boldsymbol{\beta}^{(t)}\right\|\\
			&\leq\mathcal{O}(\max\{\gamma(\alpha)^t,\kappa(\eta)^t\}).
		\end{split}
	\end{align}
	Therefore, plugging (\ref{eq:Delta1betaBound}) and (\ref{eq:Delta2betaBound}) into (\ref{eq:diffbeta}), we have:
	\begin{align*}
		\left\|\boldsymbol{\beta}^{(t)} - \hat{\boldsymbol{\beta}}_{\textsf{2SLS}}\right\| \leq \mathcal{O}(\max\{\gamma(\alpha)^t,\kappa(\eta)^t\}),
	\end{align*}
	Hence completes the proof.
\end{proof}

\subsection{Proof of Theorem \ref{thm:transformer}}\label{proof:thm:transformer}
\begin{proof}[Proof of Theorem \ref{thm:transformer}]
   For ease of notations, we ignore $l$ in the following proof. Consider the input matrix taking the form:
   \begin{align*}
       \boldsymbol{H}^{(0)}=\begin{bmatrix}
           \boldsymbol{z}_1 & \cdots & \boldsymbol{z}_n &\boldsymbol{z}_{n+1}\\ \boldsymbol{x}_1 & \cdots &\boldsymbol{x}_n &\boldsymbol{x}_{n+1}\\ y_1 &\cdots & y_n & 0\\ \boldsymbol{\Theta}_{:,1}^{(0)} & \cdots& \boldsymbol{\Theta}_{:,1}^{(0)} & \boldsymbol{\Theta}_{:,1}^{(0)}\\\vdots & \vdots &\vdots&\vdots \\\boldsymbol{\Theta}_{:,p}^{(0)} & \cdots &\boldsymbol{\Theta}_{:,p}^{(0)} & \boldsymbol{\Theta}_{:,p}^{(0)} \\\boldsymbol{\beta}^{(0)} &\cdots&\boldsymbol{\beta}^{(0)}&\boldsymbol{\beta}^{(0)}\\ \hat{\boldsymbol{x}}_1^{(0)} & \cdots&\hat{\boldsymbol{x}}_n^{(0)}& \hat{\boldsymbol{x}}_{n+1}^{(0)}\\
           1&\cdots &1&1\\1&\cdots &1&0
       \end{bmatrix}  \in \mathbb{R}^{D\times (n+1)},
   \end{align*}
    i.e., element-wise, 
    \begin{align*}\boldsymbol{h}_i^{(0)}=\left(\boldsymbol{z}_i,\boldsymbol{x}_i,y_it_i,\boldsymbol{\Theta}_{:,1}^{(0)},\ldots,\boldsymbol{\Theta}_{:,p}^{(0)},\boldsymbol{\beta}^{(0)},\hat{\boldsymbol{x}}_i^{(0)},1,t_i\right)^
    \top,\quad i=1,\ldots,n+1,\end{align*}

    where $D=qp+3p+q+3$, $t_i:=\mathbbm{1}\{i\leq n\}$ is the indicator for training sample. We can take any initialization for $\boldsymbol{\Theta}^{(0)}$, $\boldsymbol{\beta}^{(0)}$ and $\hat{\boldsymbol{x}}^{(0)}$. To avoid abuse of notations, we omit the superscript of those parameters to be updated in the following proof.

    Recall the definitions (\ref{def:H}) and (\ref{def:h}). Our goal is to show that there exists a series of attention parameters $\boldsymbol{\theta}_{\textsf{ATTN}}^{(1:2)} = \{(\boldsymbol{Q}_m^{(1:2)}, \boldsymbol{K}_m^{(1:2)}, \boldsymbol{V}_m^{(1:2)})\}_{m\in[M]}\subset \mathbb{R}^{D\times D}$ such that $\boldsymbol{\theta}_{\textsf{ATTN}}^{(1:2)}$ updates $\boldsymbol{\Theta},\boldsymbol{\beta}$  on the corresponding rows. i.e, if we denote $D_0:=q+p+1$, the updates on row $D_0+1$ to row $D_0+qp$ correspond to $\boldsymbol{\Theta}$, and the updates on row $D_0+qp+1$ to row $D_0+qp+p$ correspond to $\boldsymbol{\beta}$.

	\noindent   \textbf{1) In the first layer, the transformer updates the current first-stage estimate $\hat{\boldsymbol{x}}$}.

    For $m=2k-1,k=1,\ldots,p$, define $\boldsymbol{Q}_m^{(1)},\boldsymbol{K}_m^{(1)},\boldsymbol{V}_m^{(1)}$ such that:
    \begin{align} \label{eq:first layer construction 1}
        \boldsymbol{Q}_m^{(1)}\boldsymbol{h}_i^{(0)}=\begin{bmatrix}                                                       
           z_{i1}\\ \vdots\\z_{iq}\\ \hat{x}_{ik}^{(0)}\\\boldsymbol{0}
        \end{bmatrix},
        \boldsymbol{K}_m^{(1)}\boldsymbol{h}_j^{(0)}=\begin{bmatrix}
           \boldsymbol{\Theta}_{1k}^{(0)}\\\vdots\\ \boldsymbol{\Theta}_{qk}^{(0)}\\ -1\\\boldsymbol{0}
        \end{bmatrix},       
        \boldsymbol{V}_m^{(1)}\boldsymbol{h}_j^{(0)}=\boldsymbol{e}_{D_0+qp+p+k}.
      \end{align}

      For $m=2k,k=1,\ldots,p$, define $\boldsymbol{Q}_m^{(1)},\boldsymbol{K}_m^{(1)},\boldsymbol{V}_m^{(1)}$ such that:
      \begin{align} \label{eq:first layer construction 2}
         \boldsymbol{Q}_m^{(1)}\boldsymbol{h}_i^{(0)}=\begin{bmatrix}                                                       
             -z_{i1}\\ \vdots\\-z_{iq}\\  \hat{x}_{ik}^{(0)}\\\boldsymbol{0}
          \end{bmatrix},
          \boldsymbol{K}_m^{(1)}\boldsymbol{h}_j^{(0)}=\begin{bmatrix}
             \boldsymbol{\Theta}_{1k}^{(0)}\\\vdots\\\boldsymbol{\Theta}_{qk}^{(0)}\\ 1\\\boldsymbol{0}
          \end{bmatrix},       
          \boldsymbol{V}_m^{(1)}\boldsymbol{h}_j^{(0)}=-\boldsymbol{e}_{D_0+qp+p+k},
      \end{align}
	where $\boldsymbol{e}_j\in \mathbb{R}^{D}$ is the standard unit vector with only one 1 at the $j$-th coordinate. Note that the above are just linear transformations on $\boldsymbol{h}_i$ or $\boldsymbol{h}_j$, hence such matrices $\boldsymbol{Q}_m^{(1)},\boldsymbol{K}_m^{(1)},\boldsymbol{V}_m^{(1)}$ must exist. Then we have:
	\begin{align*}
		\begin{split}
			\boldsymbol{h}_i^{(1)}&=\boldsymbol{h}_i^{(0)} + \sum_{m=1}^{2p}\frac{1}{n+1}\sum_{j=1}^{n+1}\sigma\left(\langle\boldsymbol{Q}_m^{(1)}\boldsymbol{h}_i^{(0)}, \boldsymbol{K}_m^{(1)}\boldsymbol{h}_j^{(0)}\rangle\right)\cdot \boldsymbol{V}_m^{(1)}\boldsymbol{h}_j^{(0)}\\
			&=\boldsymbol{h}_i^{(0)}+\sum_{k=1}^p\frac{1}{n+1}\sum_{j=1}^{n+1}\left[\sigma\left(\sum_{l=1}^qz_{il}\boldsymbol{\Theta}_{lk}^{(0)}- \hat{x}_{ik}^{(0)}\right)-\sigma\left(-\sum_{l=1}^qz_{il}\boldsymbol{\Theta}_{lk}^{(0)}+\hat{x}_{ik}^{(0)}\right)\right] \cdot\boldsymbol{e}_{D_0+qp+p+k}\\
			&=\boldsymbol{h}_i^{(0)}+\sum_{k=1}^p\left[\sum_{l=1}^qz_{il}\boldsymbol{\Theta}_{lk}^{(0)}-\hat{x}_{ik}^{(0)}\right]\boldsymbol{e}_{D_0+qp+p+k}\\
			&=\boldsymbol{h}_i^{(0)}+\sum_{k=1}^p\left(\hat{x}_{ik}^{(1)}-\hat{x}_{ik}^{(0)}\right)\boldsymbol{e}_{D_0+qp+p+k}.
		\end{split}
	\end{align*}
	Thus this layer correctly updates the first-stage prediction values  $\hat{\boldsymbol{x}}^{(1)}_1,\ldots,\hat{\boldsymbol{x}}_{n+1}^{(1)}$, where $\hat{\boldsymbol{x}}_i^{(1)}:=[\boldsymbol{Z\Theta}^{(0)}]_{i,:}=\sum_{l=1}^qz_{il}\boldsymbol{\Theta}_{l,:}^{(0)}$. We will have:
	\begin{align*}
		\boldsymbol{H}^{(1)}=\begin{bmatrix}
			\boldsymbol{z}_1 & \cdots & \boldsymbol{z}_n &\boldsymbol{z}_{n+1}\\ \boldsymbol{x}_1 & \cdots &\boldsymbol{x}_n &\boldsymbol{x}_{n+1}\\ y_1 &\cdots & y_n & 0\\ \boldsymbol{\Theta}_{:,1}^{(0)} & \cdots& \boldsymbol{\Theta}_{:,1}^{(0)} & \boldsymbol{\Theta}_{:,1}^{(0)}\\\vdots & \vdots &\vdots&\vdots \\\boldsymbol{\Theta}_{:,p}^{(0)} & \cdots &\boldsymbol{\Theta}_{:,p}^{(0)} & \boldsymbol{\Theta}_{:,p}^{(0)} \\\boldsymbol{\beta}^{(0)} &\cdots&\boldsymbol{\beta}^{(0)}&\boldsymbol{\beta}^{(0)}\\ \hat{\boldsymbol{x}}_1^{(1)} & \cdots&\hat{\boldsymbol{x}}_n^{(1)}& \hat{\boldsymbol{x}}_{n+1}^{(1)}\\
			1&\cdots &1&1\\1&\cdots &1&0
		\end{bmatrix}.
	\end{align*}
	\textbf{2) In the second layer, the transformer does the gradient updates on the parameters $\boldsymbol{\Theta}$ and $\boldsymbol{\beta}$}. 

	For $m=2k-1, k=1,\ldots,p,$ define $\boldsymbol{Q}_m^{(2)},\boldsymbol{K}_m^{(2)},\boldsymbol{V}_m^{(2)}$ such that:
      \begin{align} \label{eq:second layer construction 1}
         \boldsymbol{Q}_m^{(2)}\boldsymbol{h}_i^{(1)}=\begin{bmatrix}
               \boldsymbol{\Theta}_{:,k}^{(0)}\\ -1\\-1\\ \vdots\\\boldsymbol{0}
         \end{bmatrix},
         \boldsymbol{K}_m^{(2)}\boldsymbol{h}_j^{(1)}=\begin{bmatrix}
               \boldsymbol{z}_j\\x_{jk}t_j\\R(1-t_j)\\\vdots\\ \boldsymbol{0}
         \end{bmatrix},       
         \boldsymbol{V}_m^{(2)}\boldsymbol{h}_j^{(1)}=-(n+1)\eta\sum_{l=1}^qz_{jl}\boldsymbol{e}_{D_0+(k-1)q+l}.
      \end{align}

   For $m=2k,k=1,\ldots,p$, define $\boldsymbol{Q}_m^{(2)},\boldsymbol{K}_m^{(2)},\boldsymbol{V}_m^{(2)}$  such that:
      \begin{align} \label{eq:second layer construction 2}
        \boldsymbol{Q}_m^{(2)}\boldsymbol{h}_i^{(1)}=\begin{bmatrix}
            -\boldsymbol{\Theta}_{:,k}^{(0)}\\ 1\\-1\\ \vdots\\\boldsymbol{0}
        \end{bmatrix},
        \boldsymbol{K}_m^{(2)}\boldsymbol{h}_j^{(1)}=\begin{bmatrix}
            \boldsymbol{z}_j\\x_{jk}t_j\\R(1-t_j)\\ \vdots\\\boldsymbol{0}
        \end{bmatrix},       
        \boldsymbol{V}_m^{(2)}\boldsymbol{h}_j^{(1)}=(n+1)\eta\sum_{l=1}^qz_{jl}\boldsymbol{e}_{D_0+(k-1)q+l},
      \end{align}
	where $R=\underset{\substack{i = 1, \ldots, n+1\\t = 0, 1, \ldots}}{\max}\{||\boldsymbol{\Theta}^{(t)\top}\boldsymbol{z}_i||\}$. Then we have:
    \begin{align*}
        \begin{split}
            \sigma\left(\langle\boldsymbol{Q}_{2k-1}^{(2)}\boldsymbol{h}_i^{(1)},\boldsymbol{K}_{2k-1}^{(2)}\boldsymbol{h}_j^{(1)}\rangle\right)
            &=\sigma\left(\boldsymbol{\Theta}_{:,k}^{(0)\top}\boldsymbol{z}_j-x_{jk}t_j-R(1-t_j)\right)\\
            &=\sigma\left(\boldsymbol{\Theta}_{:,k}^{(0)\top}\boldsymbol{z}_j-x_{jk}\right)\mathbbm{1}\{t_j=1\}\\
            &=\sigma\left(\boldsymbol{\Theta}_{:,k}^{(0)\top}\boldsymbol{z}_j-x_{jk}\right)t_j,\\
        \end{split}
    \end{align*}
    and
    \begin{align*}
        \begin{split}
            \sigma\left(\langle\boldsymbol{Q}_{2k}^{(2)}\boldsymbol{h}_i^{(1)},\boldsymbol{K}_{2k}^{(2)}\boldsymbol{h}_j^{(1)}\rangle\right)
            &=\sigma\left(-\boldsymbol{\Theta}_{:,k}^{(0)\top}\boldsymbol{z}_j+x_{jk}t_j-R(1-t_j)\right)\\
            &=\sigma\left(-\boldsymbol{\Theta}_{:,k}^{(0)\top}\boldsymbol{z}_j+x_{jk}\right)\mathbbm{1}\{t_j=1\}\\
            &=\sigma\left(-\boldsymbol{\Theta}_{:,k}^{(0)\top}\boldsymbol{z}_j+x_{jk}\right)t_j.
        \end{split}
    \end{align*}
    So that
    \begin{align*}
        \begin{split}
            &\sum_{m=1}^{2p}  \sigma\left(\langle\boldsymbol{Q}_{m}^{(2)}\boldsymbol{h}_i^{(1)},\boldsymbol{K}_{m}^{(2)}\boldsymbol{h}_j^{(1)}\rangle\right)\boldsymbol{V}_m^{(2)}\boldsymbol{h}_j^{(1)}\\
            &=-(n+1)t_j\eta\sum_{k=1}^p\left[\sigma\left(\boldsymbol{\Theta}_{:,k}^{(0)\top}\boldsymbol{z}_j-x_{jk}\right)-\sigma\left(-\boldsymbol{\Theta}_{:,k}^{(0)\top}\boldsymbol{z}_j+x_{jk}\right)\right]\cdot\sum_{l=1}^qz_{jl}\boldsymbol{e}_{D_0+(k-1)q+l} \\
            &=-(n+1)t_j\eta\sum_{k=1}^p\sum_{l=1}^qz_{jl}\left(\boldsymbol{\Theta}_{:,k}^{(0)\top}\boldsymbol{z}_j-x_{jk}\right)\boldsymbol{e}_{D_0+(k-1)q+l}.
        \end{split}
    \end{align*} 
	Similarly, for $m=2p+1, 2p+2$, define $\boldsymbol{Q}_m^{(2)},\boldsymbol{K}_m^{(2)},\boldsymbol{V}_m^{(2)}$  such that:
	\begin{align} \label{eq:second layer construction 3}
		\boldsymbol{Q}_{2p+1}^{(2)}\boldsymbol{h}_i^{(1)}=\begin{bmatrix}
			\boldsymbol{\beta}^{(0)}\\ -1\\-1\\\vdots\\ \boldsymbol{0}
		\end{bmatrix},
		\boldsymbol{K}_{2p+1}^{(2)}\boldsymbol{h}_j^{(1)}=\begin{bmatrix}
			\hat{\boldsymbol{x}}_j^{(1)}\\y_jt_j\\R'(1-t_j)\\ \vdots\\\boldsymbol{0}
		\end{bmatrix},       
		\boldsymbol{V}_{2p+1}\boldsymbol{h}_j^{(1)}=-(n+1)\alpha\sum_{l=1}^p\hat{x}_{jl}^{(1)}\boldsymbol{e}_{D_0+qp+l},
	\end{align}
	\begin{align} \label{eq:second layer construction 4}
		\boldsymbol{Q}_{2p+2}^{(2)}\boldsymbol{h}_i^{(1)}=\begin{bmatrix}
			-\boldsymbol{\beta}^{(0)}\\ 1\\-1\\\vdots\\ \boldsymbol{0}
		\end{bmatrix},
		\boldsymbol{K}_{2p+2}^{(2)}\boldsymbol{h}_j^{(1)}=\begin{bmatrix}
			\hat{\boldsymbol{x}}_j^{(1)}\\y_jt_j\\R'(1-t_j)\\ \vdots\\ \boldsymbol{0}
		\end{bmatrix},       
		\boldsymbol{V}_{2p+2}\boldsymbol{h}_j^{(1)}=(n+1)\alpha\sum_{l=1}^p\hat{x}_{jl}^{(1)}\boldsymbol{e}_{D_0+qp+l},
	\end{align}
	where $R'=\underset{\substack{i = 1, \ldots, n+1\\t = 0, 1, \ldots}}{\max}\{|\boldsymbol{\beta}^{(t)\top}\boldsymbol{x}_i|\}$. Then  
	\begin{align*}
		\begin{split}
			\sigma\left(\langle \boldsymbol{Q}_{2p+1}^{(2)}\boldsymbol{h}_i^{(1)},\boldsymbol{K}_{2p+1}^{(2)}\boldsymbol{h}_j^{(1)}\rangle\right)
			&=\sigma\left(\boldsymbol{\beta}^{(0)\top}\hat{\boldsymbol{x}}_j^{(1)}-y_jt_j-R'(1-t_j)\right)\\
			&=\sigma\left(\boldsymbol{\beta}^{(0)\top}\hat{\boldsymbol{x}}_j^{(1)}-y_j\right)\mathbbm{1}\{t_j=1\}\\
			&=\sigma\left(\boldsymbol{\beta}^{(0)\top}\hat{\boldsymbol{x}}_j^{(1)}-y_j\right)t_j,
			\end{split}
			\end{align*}
			and
			\begin{align*}
			\begin{split}
			\sigma\left(\langle\boldsymbol{Q}_{2p+2}^{(2)}\boldsymbol{h}_i^{(1)},\boldsymbol{K}_{2p+2}^{(2)}\boldsymbol{h}_j^{(1)}\rangle\right)
			&=\sigma\left(-\boldsymbol{\beta}^{(0)\top}\hat{\boldsymbol{x}}_j^{(1)}+y_jt_j-R'(1-t_j)\right)\\
			&=\sigma\left(-\boldsymbol{\beta}^{(0)\top}\hat{\boldsymbol{x}}_j^{(1)}+y_j\right)\mathbbm{1}\{t_j=1\}\\
			&=\sigma\left(-\boldsymbol{\beta}^{(0)\top}\hat{\boldsymbol{x}}_j^{(1)}+y_j\right)t_j.
		\end{split}
	\end{align*}
	So that
	\begin{align*}
		\begin{split}
			&\sum_{m=2p+1}^{2p+2} \sigma\left(\langle\boldsymbol{Q}_{m}^{(2)}\boldsymbol{h}_i^{(1)},\boldsymbol{K}_{m}^{(2)}\boldsymbol{h}_j^{(1)}\rangle\right)\boldsymbol{V}_m^{(2)}\boldsymbol{h}_j^{(1)}\\
			&=-(n+1)t_j\alpha\left[\sigma\left(\boldsymbol{\beta}^{(0)\top}\hat{\boldsymbol{x}}_j^{(1)}-y_j\right)-\sigma\left(-\boldsymbol{\beta}^{(0)\top}\hat{\boldsymbol{x}}_j^{(1)}+y_j\right)\right]\cdot \sum_{l=1}^p\hat{x}_{jl}^{(1)}\boldsymbol{e}_{D_0+qp+l}\\
			&=-(n+1)t_j\alpha\sum_{l=1}^p\hat{x}_{jl}^{(1)}\left(\boldsymbol{\beta}^{(0)\top}\hat{\boldsymbol{x}}_j^{(1)}-y_j\right)\boldsymbol{e}_{D_0+qp+l}.
		\end{split}
	\end{align*}
	Thus the final output, for $i=1,\ldots,n+1$:
      \begin{align*}
         \begin{split}
               \boldsymbol{h}_i^{(2)}&= \boldsymbol{h}_i^{(1)} + \sum_{m=1}^{2p+2}\frac{1}{n+1}\sum_{j=1}^n\sigma\left(\langle\boldsymbol{Q}_{m}^{(2)}\boldsymbol{h}_i^{(1)},\boldsymbol{K}_{m}^{(2)}\boldsymbol{h}_j^{(1)}\rangle\right)\boldsymbol{V}_m^{(2)}\boldsymbol{h}_j^{(1)}\\
               &=\begin{bmatrix}
                  \boldsymbol{z}_i\\ \boldsymbol{x}_i\\ y_it_i \\
                  \boldsymbol{\Theta}_{:,1}^{(0)}-\eta\sum_{j=1}^n\boldsymbol{z}_j\left(\boldsymbol{z}_j^\top\boldsymbol{\Theta}_{:,1}^{(0)}-x_{j1}\right)\\
                  \vdots\\
                  \boldsymbol{\Theta}_{:,p}^{(0)}-\eta\sum_{j=1}^n\boldsymbol{z}_j\left(\boldsymbol{z}_j^\top\boldsymbol{\Theta}_{:,p}^{(0)}-x_{jp}\right)\\
                  \boldsymbol{\beta}^{(0)}-\alpha\sum_{j=1}^n\hat{\boldsymbol{x}}_j^{(1)}\left(\hat{\boldsymbol{x}}_j^{(1)\top}\boldsymbol{\beta}^{(0)}-y_j\right)\\
                  \hat{\boldsymbol{x}}_i^{(1)}\\1\\t_i
               \end{bmatrix}\\
               &=\begin{bmatrix}
                  \boldsymbol{z}_i\\ \boldsymbol{x}_i\\ y_it_i \\
                     \boldsymbol{\Theta}_{:,1}^{(0)}-\eta\boldsymbol{Z}^\top\left[\boldsymbol{Z\Theta}^{(0)}-\boldsymbol{X}\right]_{:,1}\\
                  \vdots\\
                  \boldsymbol{\Theta}_{:,p}^{(0)}-\eta\boldsymbol{Z}^\top\left[\boldsymbol{Z\Theta}^{(0)}-\boldsymbol{X}\right]_{:,p}\\
                  \boldsymbol{\beta}^{(0)}-\alpha\boldsymbol{\Theta}^{(0)\top}\boldsymbol{Z}^\top\left(\boldsymbol{Z\Theta}^{(0)}\boldsymbol{\beta}^{(0)}-\boldsymbol{y}\right)\\
                  \hat{\boldsymbol{x}}_i^{(1)}\\1\\t_i
               \end{bmatrix}.
         \end{split}
      \end{align*}
    This corresponds to a one-step 2SLS GD update of $\boldsymbol{\Theta},\boldsymbol{\beta}$, according to (\ref{eq:2SLS GD Update}). Therefore, the final output matrix is:

    \begin{align*}
        \boldsymbol{H}^{(2)}=\begin{bmatrix}
            \boldsymbol{z}_1 & \cdots & \boldsymbol{z}_n &\boldsymbol{z}_{n+1}\\ \boldsymbol{x}_1 & \cdots &\boldsymbol{x}_n &\boldsymbol{x}_{n+1}\\ y_1 &\cdots & y_n & 0\\ \boldsymbol{\Theta}_{:,1}^{(1)} & \cdots& \boldsymbol{\Theta}_{:,1}^{(1)} & \boldsymbol{\Theta}_{:,1}^{(1)}\\\vdots & \vdots &\vdots&\vdots \\\boldsymbol{\Theta}_{:,p}^{(1)} & \cdots &\boldsymbol{\Theta}_{:,p}^{(1)} & \boldsymbol{\Theta}_{:,p}^{(1)} \\\boldsymbol{\beta}^{(1)} &\cdots&\boldsymbol{\beta}^{(1)}&\boldsymbol{\beta}^{(1)}\\ \hat{\boldsymbol{x}}_1^{(1)} & \cdots&\hat{\boldsymbol{x}}_n^{(1)}& \hat{\boldsymbol{x}}_{n+1}^{(1)}\\
            1&\cdots &1&1\\1&\cdots &1&0
        \end{bmatrix}.
    \end{align*}
	Thus the proof is complete. We further note that in construction steps (\ref{eq:first layer construction 1})(\ref{eq:first layer construction 2})(\ref{eq:second layer construction 1})(\ref{eq:second layer construction 2})(\ref{eq:second layer construction 3})(\ref{eq:second layer construction 4}), regardless of the initial values of $\boldsymbol{\Theta}^{(0)},\boldsymbol{\beta}^{(0)},$ and $\hat{\boldsymbol{x}}^{(0)}$, the matrices $\boldsymbol{Q}_m^{(1:2)},\boldsymbol{K}_m^{(1:2)},\boldsymbol{V}_m^{(1:2)}$ do the same linear transformations on the input vectors. Therefore they are identical across different layers.
\end{proof}

\subsection{Proof of Theorem \ref{thm:excessloss}}\label{proof:thm:excessloss}
\begin{lemma}[Generalization of pretraining, from Theorem 20 in \cite{statistician}]\label{lemma:generalize pretraining}
    Given optimization problm (\ref{def:trainedtf}), with probability at least $1-\zeta$, the solution $\hat{\boldsymbol{\theta}}$ satisfies: 
    \begin{align*}
        L_{\textsf{ICL}}(\hat{\boldsymbol{\theta}})\leq \underset{\boldsymbol{\theta}\in\boldsymbol{\vartheta}}{\inf}L_{\textsf{ICL}}(\boldsymbol{\theta})+\mathcal{O}\left(B_y^2\sqrt{\frac{L^2(MD^2+DD')\log(2+\max\{B_\theta,R,B_y\})+\log(1/\zeta)}{N}}\right).
    \end{align*}
\end{lemma}

\begin{proof}[Proof of Theorem \ref{thm:excessloss}] 
	We begin by showing the (clipped) \textsf{2SLS} predictor achieves small excess loss under in-context distribution $\boldsymbol{\mathcal{P}}$:
	\begin{align*}
		&\mathbb{E}_{\mathcal{P}}\left[\left(\langle\textsf{clip}_{B_\beta}(\hat{\boldsymbol{\beta}}_{\textsf{2SLS}}),\boldsymbol{x}_{n+1}\rangle-y_{n+1}\right)^2\right] \\&=\mathbb{E}_{\mathcal{P}}\left[\left(\langle\textsf{clip}_{B_\beta}(\hat{\boldsymbol{\beta}}_{\textsf{2SLS}})-\boldsymbol{\beta},\boldsymbol{x}_{n+1}\rangle+\langle\boldsymbol{\beta},\boldsymbol{x}_{n+1}\rangle-y_{n+1}\right)^2\right]\\
		&=\mathbb{E}_{\mathcal{P}}\left[\langle\textsf{clip}_{B_\beta}(\hat{\boldsymbol{\beta}}_{\textsf{2SLS}})-\boldsymbol{\beta},\boldsymbol{x}_{n+1}\rangle^2\right]+2E_\mathcal{P}\left[\langle\textsf{clip}_{B_\beta}(\hat{\boldsymbol{\beta}}_{\textsf{2SLS}})-\boldsymbol{\beta},\boldsymbol{x}_{n+1}\rangle\big(\langle\boldsymbol{\beta},\boldsymbol{x}_{n+1}\rangle-y_{n+1}\big)\right]\\
		&\quad\quad+\mathbb{E}_{\mathcal{P}}\left[\big(\langle\boldsymbol{\beta},\boldsymbol{x}_{n+1}\rangle-y_{n+1}\big)^2\right]\\
		&=\underbrace{\mathbb{E}_{\mathcal{P}}\left[\langle\textsf{clip}_{B_\beta}(\hat{\boldsymbol{\beta}}_{\textsf{2SLS}})-\boldsymbol{\beta},\boldsymbol{x}_{n+1}\rangle^2\right]}_{\text{Excess Loss}}+\mathbb{E}_{\mathcal{P}}\left[\big(\langle\boldsymbol{\beta},\boldsymbol{x}_{n+1}\rangle-y_{n+1}\big)^2\right],\\
	\end{align*}
	where $\mathbb{E}_\mathcal{P}[\langle\textsf{clip}_{B_\beta}(\hat{\boldsymbol{\beta}}_{\textsf{2SLS}})-\boldsymbol{\beta},\boldsymbol{x}_{n+1}\rangle\big(\langle\boldsymbol{\beta},\boldsymbol{x}_{n+1}\rangle-y_{n+1}\big)]=0$ follows from the independence between $\langle\textsf{clip}_{B_\beta}(\hat{\boldsymbol{\beta}}_{\textsf{2SLS}})-\boldsymbol{\beta},\boldsymbol{x}_{n+1}\rangle$ and $\big(\langle\boldsymbol{\beta},\boldsymbol{x}_{n+1}\rangle-y_{n+1}\big)$ with $\mathbb{E}_{\mathcal{P}}\left[\langle\boldsymbol{\beta},\boldsymbol{x}_{n+1}\rangle-y_{n+1}\right]=\mathbb{E}_{\mathcal{P}}[\epsilon_{n+1}]=0$.

	To bound the excess loss, we have
	\begin{align}\label{eq:excess loss 2sls}
		\begin{split}
			\mathbb{E}_{\mathcal{P}}\left[\langle\textsf{clip}_{B_\beta}(\hat{\boldsymbol{\beta}}_{\textsf{2SLS}})-\boldsymbol{\beta},\boldsymbol{x}_{n+1}\rangle^2\right]&=\mathbb{E}_{\mathcal{P}}\left[\left\|\boldsymbol{x}_{n+1}^\top\left(\textsf{clip}_{B_\beta}(\hat{\boldsymbol{\beta}}_{\textsf{2SLS}})-\boldsymbol{\beta}\right)\right\|^2\right]\\
			&\leq\mathbb{E}_{\mathcal{P}}\left[\left\|\boldsymbol{x}_{n+1}\right\|^2\left\|\textsf{clip}_{B_\beta}(\hat{\boldsymbol{\beta}}_{\textsf{2SLS}})-\boldsymbol{\beta}\right\|^2\right]\\
			&=\mathbb{E}_{\mathcal{P}}\left[\left\|\boldsymbol{x}_{n+1}\right\|^2\right]\mathbb{E}_{\mathcal{P}}\left[\left\|\textsf{clip}_{B_\beta}(\hat{\boldsymbol{\beta}}_{\textsf{2SLS}})-\boldsymbol{\beta}\right\|^2\right]\\
			&\leq \mathcal{O}\left(B_x^2\left(\frac{q}{n}\left(\frac{B_\beta^2}{K}+C^2(n)(\boldsymbol{\phi}^\top\boldsymbol{\Sigma}_u\boldsymbol{\phi}+\sigma_\epsilon^2)\right)\right)\right),
		\end{split}
	\end{align}
	where the last inequality follows from (\ref{eq:Emse of 2SLS}).

	Next, for the ICL loss, we have
	\begin{align}\label{eq:ICL loss decomposition}
		\begin{split}
			&L_{\textsf{ICL}}(\boldsymbol{\theta})\\
			&=\mathbb{E}_\pi\mathbb{E}_{\mathcal{P}}\left[\left(\widetilde{\textsf{TF}}_{\boldsymbol{\theta}}(\boldsymbol{H})-y_{n+1}\right)^2\right]\\
			&=\mathbb{E}_\pi\mathbb{E}_{\mathcal{P}}\left[\left(\widetilde{\textsf{TF}}_{\boldsymbol{\theta}}(\boldsymbol{H})-\langle\textsf{clip}_{B_\beta}(\hat{\boldsymbol{\beta}}_{\textsf{2SLS}}),\boldsymbol{x}_{n+1}\rangle+\langle\textsf{clip}_{B_\beta}(\hat{\boldsymbol{\beta}}_{\textsf{2SLS}}),\boldsymbol{x}_{n+1}\rangle-y_{n+1}\right)^2\right]\\
			&=\mathbb{E}_\pi\bigg\{\mathbb{E}_{\mathcal{P}}\Big[\left(\widetilde{\textsf{TF}}_{\boldsymbol{\theta}}(\boldsymbol{H})-\langle\textsf{clip}_{B_\beta}(\hat{\boldsymbol{\beta}}_{\textsf{2SLS}}),\boldsymbol{x}_{n+1}\rangle\right)^2\Big]+\mathbb{E}_{\mathcal{P}}\Big[\left(\langle\textsf{clip}_{B_\beta}(\hat{\boldsymbol{\beta}}_{\textsf{2SLS}}),\boldsymbol{x}_{n+1}\rangle-y_{n+1}\right)^2\Big]\\
			&\quad\quad+2\mathbb{E}_{\mathcal{P}}\Big[\left(\widetilde{\textsf{TF}}_{\boldsymbol{\theta}}(\boldsymbol{H})-\langle\textsf{clip}_{B_\beta}(\hat{\boldsymbol{\beta}}_{\textsf{2SLS}}),\boldsymbol{x}_{n+1}\rangle\right)\left(\langle\textsf{clip}_{B_\beta}(\hat{\boldsymbol{\beta}}_{\textsf{2SLS}}),\boldsymbol{x}_{n+1}\rangle-y_{n+1}\right)\Big]\bigg\}\\
			&\leq\mathbb{E}_\pi\bigg\{\mathbb{E}_{\mathcal{P}}\Big[\left(\widetilde{\textsf{TF}}_{\boldsymbol{\theta}}(\boldsymbol{H})-\langle\textsf{clip}_{B_\beta}(\hat{\boldsymbol{\beta}}_{\textsf{2SLS}}),\boldsymbol{x}_{n+1}\rangle\right)^2\Big]+\mathbb{E}_{\mathcal{P}}\Big[\langle\textsf{clip}_{B_\beta}(\hat{\boldsymbol{\beta}}_{\textsf{2SLS}})-\boldsymbol{\beta},\boldsymbol{x}_{n+1}\rangle^2\Big]+\mathbb{E}_{\mathcal{P}}\Big[\big(\langle\boldsymbol{\beta},\boldsymbol{x}_{n+1}\rangle-y_{n+1}\big)^2\Big]\\
			&\quad\quad+2\mathbb{E}_{\mathcal{P}}\Big[\left|\widetilde{\textsf{TF}}_{\boldsymbol{\theta}}(\boldsymbol{H})-\langle\textsf{clip}_{B_\beta}(\hat{\boldsymbol{\beta}}_{\textsf{2SLS}}),\boldsymbol{x}_{n+1}\rangle\right|\Big]\mathbb{E}_{\mathcal{P}}\Big[\left|\langle\textsf{clip}_{B_\beta}(\hat{\boldsymbol{\beta}}_{\textsf{2SLS}}),\boldsymbol{x}_{n+1}\rangle-y_{n+1}\right|\Big]\bigg\}\\
			&\leq \mathbb{E}_\pi\bigg\{\mathbb{E}_{\mathcal{P}}\Big[\Big(\widetilde{\textsf{TF}}_{\boldsymbol{\theta}}(\boldsymbol{H})-\langle\textsf{clip}_{B_\beta}(\hat{\boldsymbol{\beta}}_{\textsf{2SLS}}),\boldsymbol{x}_{n+1}\rangle\Big)^2\Big]+\mathbb{E}_{\mathcal{P}}\Big[\langle\textsf{clip}_{B_\beta}(\hat{\boldsymbol{\beta}}_{\textsf{2SLS}})-\boldsymbol{\beta},\boldsymbol{x}_{n+1}\rangle^2\Big]+\mathbb{E}_{\mathcal{P}}\Big[\big(\langle\boldsymbol{\beta},\boldsymbol{x}_{n+1}\rangle-y_{n+1}\big)^2\Big]\\
			&\quad\quad+2\mathbb{E}_{\mathcal{P}}\Big[\left|\widetilde{\textsf{TF}}_{\boldsymbol{\theta}}(\boldsymbol{H})-\langle\textsf{clip}_{B_\beta}(\hat{\boldsymbol{\beta}}_{\textsf{2SLS}}),\boldsymbol{x}_{n+1}\rangle\right|\Big]\left(\mathbb{E}_{\mathcal{P}}\left[\big|\langle\textsf{clip}_{B_\beta}(\hat{\boldsymbol{\beta}}_{\textsf{2SLS}})-\boldsymbol{\beta},\boldsymbol{x}_{n+1}\rangle\big|\right]+\mathbb{E}_{\mathcal{P}}\Big[\big|\langle\boldsymbol{\beta},\boldsymbol{x}_{n+1}\rangle-y_{n+1}\big|\Big]\right)\bigg\}.
		\end{split}
	\end{align}
	From Corollary \ref{coro:transformer}, we know that there exists a $L=2\bar{L}+1$-layer attention-only transformer model $\boldsymbol{\theta}$, with $M=2(p+1)$ heads, and embedding dimension $D=qp+3p+q+3$, such that for any $\boldsymbol{H}$, given any learning rates $\alpha, \eta$ and $\Lambda$ as defined in~\eqref{eq:lambda}, the following holds\footnote{The clipping bound on $\hat{\boldsymbol{\beta}}_{\textsf{2SLS}}$ can be matched by adjusting the clipping threshold on the last layer of $\widetilde{\textsf{TF}}_{\boldsymbol{\theta}}$.}:
	\begin{align*}
		\left|\widetilde{\textsf{TF}}_{\boldsymbol{\theta}}(\boldsymbol{H})-\langle\textsf{clip}_{B_\beta}(\hat{\boldsymbol{\beta}}_{\textsf{2SLS}}), \boldsymbol{x}_{n+1}\rangle\right|\leq\mathcal{O}\left(B_x\Lambda^{\bar{L}}\right).
	\end{align*}
	Denote $\Lambda^\star:=\underset{\alpha,\eta}{\min}\;\mathbb{E}_\pi\mathbb{E}_{\mathcal{P}}\left[\Lambda|\boldsymbol{H},\alpha,\eta\right]$, then under $\alpha^\star,\eta^\star$, we have:
	\begin{align}\label{eq:ICL loss term 1}
		\mathbb{E}_\pi\mathbb{E}_{\mathcal{P}}\left[\left|\widetilde{\textsf{TF}}_{\boldsymbol{\theta}}(\boldsymbol{H})-\langle\textsf{clip}_{B_\beta}(\hat{\boldsymbol{\beta}}_{\textsf{2SLS}}), \boldsymbol{x}_{n+1}\rangle\right|\right]\leq\mathcal{O}\left(B_x(\Lambda^{\star})^{\bar{L}}\right),
	\end{align}
	and 
	\begin{align}\label{eq:ICL loss term 2}
		\mathbb{E}_\pi\mathbb{E}_{\mathcal{P}}\Big[\big(\widetilde{\textsf{TF}}_{\boldsymbol{\theta}}(\boldsymbol{H})-\langle\textsf{clip}_{B_\beta}(\hat{\boldsymbol{\beta}}_{\textsf{2SLS}}), \boldsymbol{x}_{n+1}\rangle\big)^2\Big]\leq\mathcal{O}\left(B_x^2\mu_{\Lambda,2}^\star\right),
	\end{align}
	where $\mu_{\Lambda,2}^\star:=\mathbb{E}_\pi\mathbb{E}_{\mathcal{P}}\left[\Lambda^{2\bar{L}}|\boldsymbol{H},\alpha^\star,\eta^\star\right]$ is close to 0.

	With condition (\ref{eq:meta distribution condition}), from Cauchy-Schwarz inequality, we have:
	\begin{align}\label{eq:ICL loss term 3}
		\mathbb{E}_\pi\mathbb{E}_{\mathcal{P}}\Big[\big|\langle\boldsymbol{\beta},\boldsymbol{x}_{n+1}\rangle-y_{n+1}\big|\Big]&\leq \mathbb{E}_\pi\left[\sqrt{\mathbb{E}_{\mathcal{P}}\left(\epsilon_{n+1}^2\right)}\right]
		=\mathbb{E}_\pi[\sigma_\epsilon]
		\leq\tilde{\sigma}_\epsilon.
	\end{align}
	Also, from (\ref{eq:excess loss 2sls}), we have:
	\begin{align}\label{eq:ICL loss term 4}
		\mathbb{E}_\pi\mathbb{E}_{\mathcal{P}}\left[\langle\textsf{clip}_{B_\beta}(\hat{\boldsymbol{\beta}}_{\textsf{2SLS}})-\boldsymbol{\beta},\boldsymbol{x}_{n+1}\rangle^2\right]\leq\mathcal{O}\left(B_x^2\left(\frac{q}{n}\left(\frac{B_\beta^2}{K}+C^2(n)\tilde{\sigma}^2\right)\right)\right).
	\end{align}
	Further, 
	\begin{align}\label{eq:ICL loss term 5}
		\begin{split}
			\mathbb{E}_\pi\mathbb{E}_{\mathcal{P}}\left[\big|\langle\textsf{clip}_{B_\beta}(\hat{\boldsymbol{\beta}}_{\textsf{2SLS}})-\boldsymbol{\beta},\boldsymbol{x}_{n+1}\rangle\big|\right]&\leq \sqrt{\mathbb{E}_\pi\mathbb{E}_{\mathcal{P}}\left[\langle\textsf{clip}_{B_\beta}(\hat{\boldsymbol{\beta}}_{\textsf{2SLS}})-\boldsymbol{\beta},\boldsymbol{x}_{n+1}\rangle^2\right]}\\
			&\leq \mathcal{O}\left(B_x\sqrt{\frac{q}{n}\left(\frac{B_\beta^2}{K}+C^2(n)\tilde{\sigma}^2\right)}\right).
		\end{split}
	\end{align}
	Finally, with (\ref{eq:ICL loss term 1})(\ref{eq:ICL loss term 2})(\ref{eq:ICL loss term 3})(\ref{eq:ICL loss term 4})(\ref{eq:ICL loss term 5}), rearranging the terms in (\ref{eq:ICL loss decomposition}), we have:
	\begin{align}\label{eq:excess loss}
		\begin{split}
			&L_{\textsf{ICL}}(\boldsymbol{\theta}) - \mathbb{E}_{\pi}\mathbb{E}_{\mathcal{P}}\left[(y_{n+1}-\langle\boldsymbol{\beta},\boldsymbol{x}_{n+1}\rangle)^2\right] \\&\leq \mathcal{O}\left(B_x^2\left(\mu_{\Lambda,2}^\star+(\Lambda^\star)^{\bar{L}}\sqrt{\frac{q}{n}\left(\frac{B_\beta^2}{K}+C^2(n)\tilde{\sigma}^2\right)}+\frac{q}{n}\left(\frac{B_\beta^2}{K}+C^2(n)\tilde{\sigma}^2\right)\right)+ B_x(\Lambda^\star)^{\bar{L}}\tilde{\sigma}_\epsilon\right).\\
			&\leq \mathcal{O}\left((\Lambda^\star)^{\bar{L}}\left(B_x^2\sqrt{\frac{q}{n}\left(\frac{B_\beta^2}{K}+C^2(n)\tilde{\sigma}^2\right)}+B_x\tilde{\sigma}_\epsilon\right)+B_x^2\left(\frac{q}{n}\left(\frac{B_\beta^2}{K}+C^2(n)\tilde{\sigma}^2\right)+\mu_{\Lambda,2}^\star\right)\right).
		\end{split}
	\end{align}
	Thus combining Lemma \ref{lemma:generalize pretraining} with (\ref{eq:excess loss}) completes the proof.
\end{proof}

\section{Additional Experiments and Discussions}
For all experiments in this section, to be consistent with our main experiment in Section \ref{sec:simulation}, we generate \(n=50\) training samples with $p=5, q=10$, following the data generating process described in Algorithm \ref{alg:data generating process}. The task parameters $\boldsymbol{\Theta}, \boldsymbol{\beta}, \boldsymbol{\Phi}, \boldsymbol{\phi}$ are sampled from standard Gaussian distribution, the covariance matrices $\boldsymbol{\Sigma}_z, \boldsymbol{\Sigma}_u, \boldsymbol{\Sigma}_\omega$ are set to be identity matrices, and the noise level $\sigma_\epsilon$ is set to 1.

\subsection{Simulations Verifying Theorem \ref{thm:2SLS GD}}\label{sec:exp 2SLS GD}

We use the GD-based \textsf{2SLS} method (\ref{eq:2SLS GD Update}) to estimate the causal effect $\boldsymbol{\beta}$. For the simulated data, we calculate the following metrics:
\begin{align*}\frac{2}{\sigma_{\max}^2(\boldsymbol{Z}\hat{\boldsymbol{\Theta}})}=0.0016, \frac{2}{\sigma_{\max}^2(\boldsymbol{Z})}=0.0212.\end{align*}

By Theorem \ref{thm:2SLS GD}, the gradient descent converges when $\alpha\in (0,0.0016)$ and $\eta\in(0,0.0212)$. The overall convergence rate is determined by  $\Lambda:= \max\{\gamma(\alpha),\kappa(\eta)\}$, where 
\begin{align*}
    \gamma(\alpha)&:=\rho\left(\boldsymbol{I}-\alpha\hat{\boldsymbol{\Theta}}^\top\boldsymbol{Z}^\top\boldsymbol{Z\hat{\Theta}}\right),\\
    \kappa(\eta)&:=\rho\left(\boldsymbol{I}-\eta\boldsymbol{Z}^\top\boldsymbol{Z}\right).
\end{align*}

We first set $\alpha=0.0012$ and vary $\eta$. The corresponding convergence rates are determined by $\Lambda = \max(0.87, \kappa(\eta))$. Next, we set $\eta=0.01$ and vary $\alpha$. The corresponding convergence curves are determined by $\Lambda = \max(\gamma(\alpha), 0.82)$. We compare the estimates $\hat{\boldsymbol{\beta}}^{(t)}$ with the \textsf{2SLS} estimate $\hat{\boldsymbol{\beta}}_{\textsf{2SLS}}$ as the iteration proceeds. The convergence results are shown in Figure \ref{fig:gradconv}.

\begin{figure}[H]
    \centering
    \begin{subfigure}[b]{0.45\textwidth}
        \includegraphics[width=\textwidth]{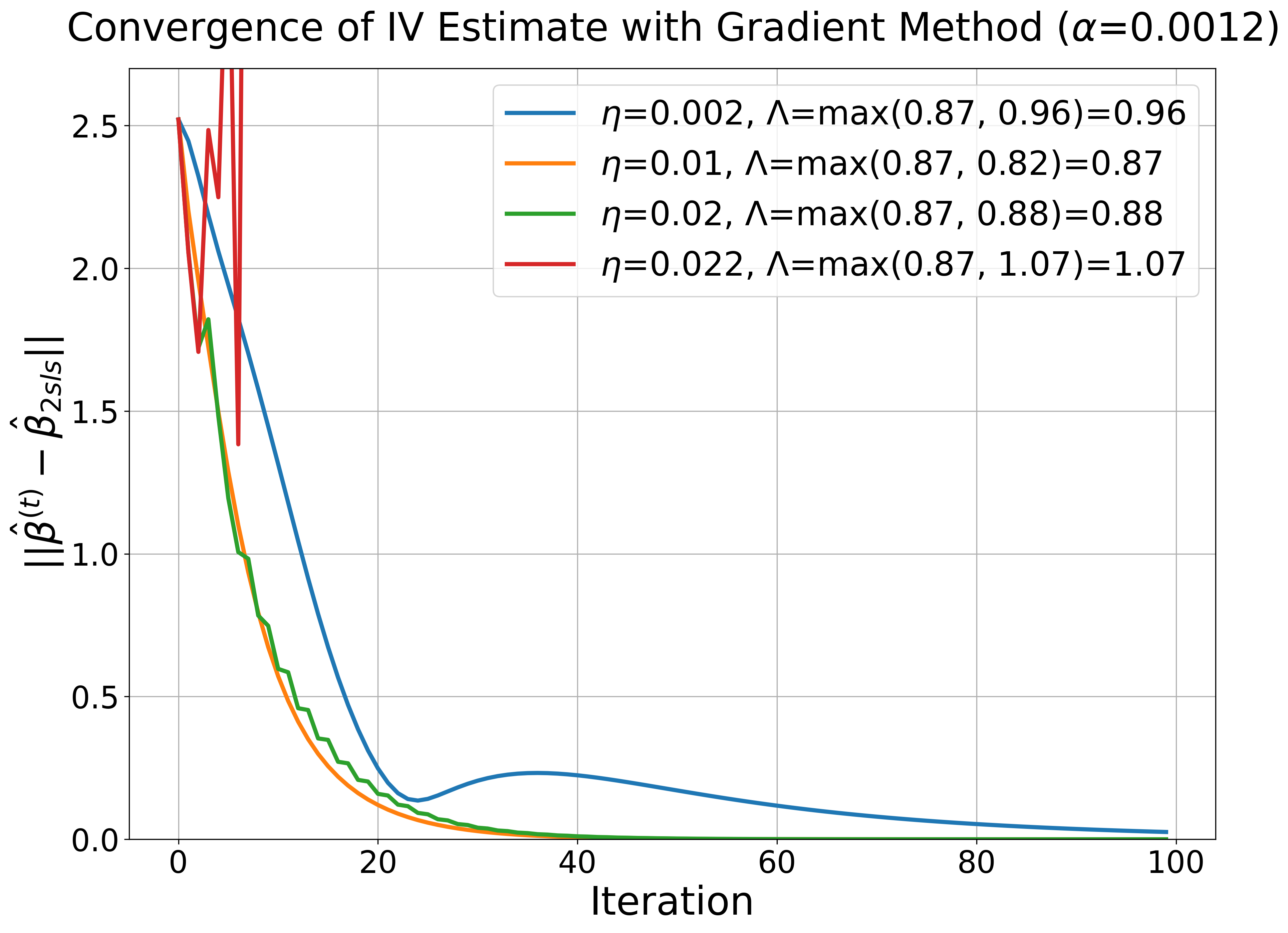}
        \caption{}
        \label{fig:gradconv(a)}
    \end{subfigure}
    \begin{subfigure}[b]{0.45\textwidth}
        \includegraphics[width=\textwidth]{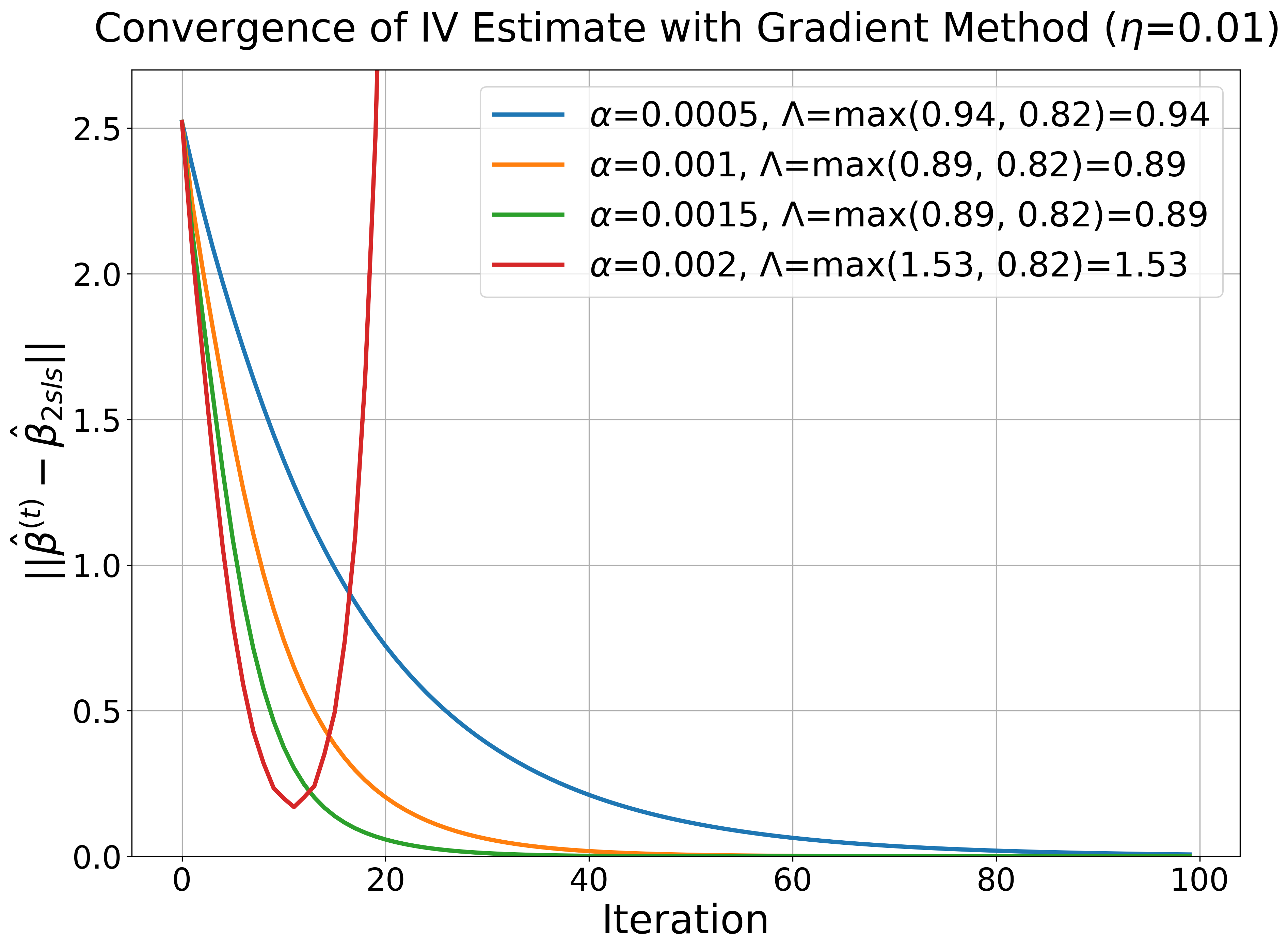}
        \caption{}
        \label{fig:gradconv(b)}
    \end{subfigure}
    \caption{The convergence of the GD-based \textsf{2SLS} method with (a) fixed $\alpha=0.0012$ and varying $\eta$ and (b) fixed $\eta=0.01$ and varying $\alpha$.}
    \label{fig:gradconv}
\end{figure}

The results in Figure \ref{fig:gradconv} are consistent with our theoretical analysis in Theorem \ref{thm:2SLS GD}. It is worth noting that in Figure \ref{fig:gradconv(a)}, when $\eta$ is relatively large (or small), the convergence curves exhibit some suiggly patterns. This is due to the innerloop updates (\ref{eq:InnerLoopUpdate}) are converging faster (or slower) than the outer loop updates (\ref{eq:OuterLoopUpdate}). However, the overall convergence rate is still determined by $\Lambda$. This pattern doesn't appear in Figure \ref{fig:gradconv(b)} as we set $\eta$ to be a moderate value, which ensures that the inner loop and outer loop converge synchronously.

Next, we show the bias of the GD estimator. For better convergence, we set $\alpha^\star=\frac{1}{\sigma_{\max}^2(\boldsymbol{Z\hat{\Theta}})}$ and $\eta^\star=\frac{1}{\sigma_{\max}^2(\boldsymbol{Z})}$. We compare the biases of the GD estimator with $n=50,100,150$ in-context samples. The results are shown in Figure \ref{fig:bias_vs_iteration}.
\begin{figure}[H]
    \centering
    \includegraphics[width=0.5\textwidth]{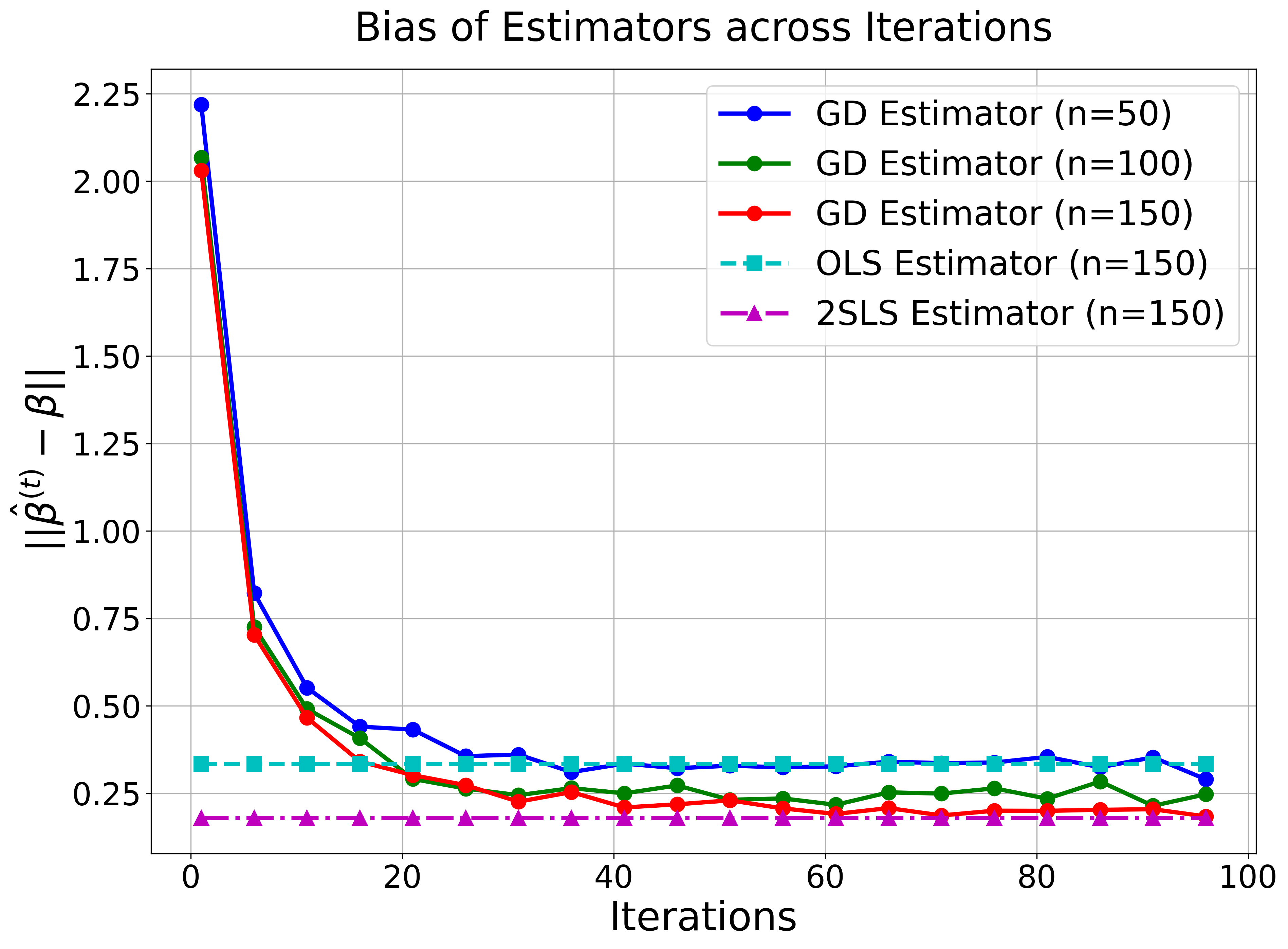}
    \caption{The convergence of the GD-based \textsf{2SLS} method with $\alpha^\star=\frac{1}{\sigma_{\max}^2(\boldsymbol{Z\hat{\Theta}})}$ and $\eta^\star=\frac{1}{\sigma_{\max}^2(\boldsymbol{Z})}$. The biases of \textsf{2SLS} estimator and \textsf{OLS} estimator at $n=150$ are plotted for comparison.}
\label{fig:bias_vs_iteration}
\end{figure}

\subsection{Discussions on \textsf{2SLS} with $\ell_2$-Regularization}
In this section, we briefly discuss a generalization of our analysis to the case where the \textsf{2SLS} estimator is regularized by the $\ell_2$ penalty (Ridge 2SLS). For this case, the bi-level optimization problem \eqref{eq:IV loss objective} is modified as follows:
\begin{align*}
    \begin{aligned}
    \min_{\boldsymbol{\beta}} \quad\mathcal{L}(\boldsymbol{\beta})&=\frac{1}{n}\sum_{i=1}^n(y_i-\boldsymbol{z}_i^\top\hat{\boldsymbol{\Theta}}\boldsymbol{\beta})^2+\frac{1}{2}\lambda\|\boldsymbol{\beta}\|^2,\\
    \text{where }\quad \hat{\boldsymbol{\Theta}}:=&\underset{\boldsymbol{\Theta}}{\arg\min}\quad\frac{1}{n}\sum_{j=1}^n(\boldsymbol{x}_j-\boldsymbol{z}_j^\top\boldsymbol{\Theta})^2+\frac{1}{2}\tau\|\boldsymbol{\Theta}\|_F^2,
    \end{aligned}
\end{align*}
where $\lambda, \tau\geq 0$ are regularization parameters. To solve this problem,
the GD updates in \eqref{eq:2SLS GD Update} is modified as follows:
\begin{subequations}\label{eq:2SLS Ridge GD Update}
    \begin{align}
        \boldsymbol{\Theta}^{(t+1)} &= \boldsymbol{\Theta}^{(t)} - \eta \left[\boldsymbol{Z}^\top(\boldsymbol{Z\Theta}^{(t)}-\boldsymbol{X})+\tau\boldsymbol{\Theta}^{(t)}\right],\label{eq:InnerLoopUpdate Ridge}\\
        \boldsymbol{\beta}^{(t+1)} &= \boldsymbol{\beta}^{(t)} - \alpha \left[\boldsymbol{\Theta}^{(t)\top}\boldsymbol{Z}^\top(\boldsymbol{Z\Theta}^{(t)}\boldsymbol{\beta}^{(t)}-\boldsymbol{Y})+\lambda\boldsymbol{\beta}^{(t)}\right].\label{eq:OuterLoopUpdate Ridge}
    \end{align}
\end{subequations}
The only difference between \eqref{eq:2SLS GD Update} and \eqref{eq:2SLS Ridge GD Update} is the additional terms $\tau\boldsymbol{\Theta}^{(t)}$ in \eqref{eq:InnerLoopUpdate Ridge}, and $\lambda\boldsymbol{\beta}^{(t)}$ in \eqref{eq:OuterLoopUpdate Ridge}. The convergence analysis of the $\ell_2$-regularized GD updates in \eqref{eq:2SLS Ridge GD Update} can be conducted in a similar manner as in Theorem \ref{thm:2SLS GD}. The only difference is that the convergence rate $\Lambda$ is now determined by the spectral radiuses of $\boldsymbol{I}-\eta(\boldsymbol{Z}^\top\boldsymbol{Z}+\tau\boldsymbol{I})$ and $\boldsymbol{I}-\alpha(\hat{\boldsymbol{\Theta}}^\top\boldsymbol{Z}^\top\boldsymbol{Z}\hat{\boldsymbol{\Theta}}+\lambda \boldsymbol{I})$, respectively.

With the same configuration as in Theorem \ref{thm:transformer} but adding $p+1$ attention heads in the second layer (i.e. the second layer needs $3p+3$ heads), we can show that transformers are able to implement the $\ell_2$-regularized GD updates in \eqref{eq:2SLS Ridge GD Update}. The proof follows directly from Appendix \ref{proof:thm:transformer}, with the construction of the new attention heads in the second layer as follows.

For $m=2p+2+k, k=1,\ldots,p$, define $\boldsymbol{Q}_m^{(2)}, \boldsymbol{K}_m^{(2)}, \boldsymbol{V}_m^{(2)}$ such that:

\begin{align*}
    \boldsymbol{Q}_m^{(2)}\boldsymbol{h}_i^{(1)}=\begin{bmatrix}
        1\\\boldsymbol{0}
    \end{bmatrix}, 
    \boldsymbol{K}_m^{(2)}\boldsymbol{h}_i^{(1)}=\begin{bmatrix}
        1\\\boldsymbol{0}
    \end{bmatrix}, 
    \boldsymbol{V}_m^{(2)}\boldsymbol{h}_i^{(1)}=-\eta\tau\sum_{l=1}^q\Theta_{lk}^{(0)}\boldsymbol{e}_{D_0+(k-1)q+l}.
\end{align*}

For $m=3p+3$, define $\boldsymbol{Q}_{3p+3}^{(2)}, \boldsymbol{K}_{3p+3}^{(2)}, \boldsymbol{V}_{3p+3}^{(2)}$ such that:
\begin{align*}
    \boldsymbol{Q}_{3p+3}^{(2)}\boldsymbol{h}_i^{(1)}=\begin{bmatrix}
        1\\\boldsymbol{0}
    \end{bmatrix}, 
    \boldsymbol{K}_{3p+3}^{(2)}\boldsymbol{h}_i^{(1)}=\begin{bmatrix}
        1\\\boldsymbol{0}
    \end{bmatrix}, 
    \boldsymbol{V}_{3p+3}^{(2)}\boldsymbol{h}_i^{(1)}=-\alpha\lambda\sum_{l=1}^q\beta_{l}^{(0)}\boldsymbol{e}_{D_0+qp+l}.
\end{align*}

Then the remaining proof follows exactly the same as Appendix \ref{proof:thm:transformer}. This result indicates that transformers are potentially capable of handling multicollinearity in IV regression problem. We conduct experiments to validate this and the results are shown in Appendix \ref{sec:exp multicollinearity}.

\subsection{Experiments on Multicollinearity IV Problem}\label{sec:exp multicollinearity}
As a supplement to Section \ref{sec:sim results}, we examine the case where multicollinearity occurs in the IV regression problem. We generate the test prompts in the same way using Algorithm \ref{alg:data generating process}, but introduce multicollinearity in the endogenous variable $\boldsymbol{x}$ and instrument $\boldsymbol{z}$. 

Specifically, we first generate 4 columns of $\boldsymbol{X}$, and 9 columns of $\boldsymbol{Z}$, and then set $\boldsymbol{X}_{:,5}\sim\mathcal{N}(2\boldsymbol{X}_{:,4}, 10^{-6}\boldsymbol{I}_n)$, and $\boldsymbol{Z}_{:,10}\sim\mathcal{N}(2\boldsymbol{Z}_{:,9},10^{-6}\boldsymbol{I}_n)$. The results are shown in Figure \ref{fig:multicollinearity}. As shown in Figure \ref{fig:multicollinearity}, both ordinary \textsf{OLS} and \textsf{2SLS} estimators fail to estimate the coefficients, while the trained transformer model is still able to provide consistent predictions and coefficient estimates.

\begin{figure}[h]
    \centering
    \begin{subfigure}[b]{0.48\textwidth}
        \includegraphics[width=\textwidth]{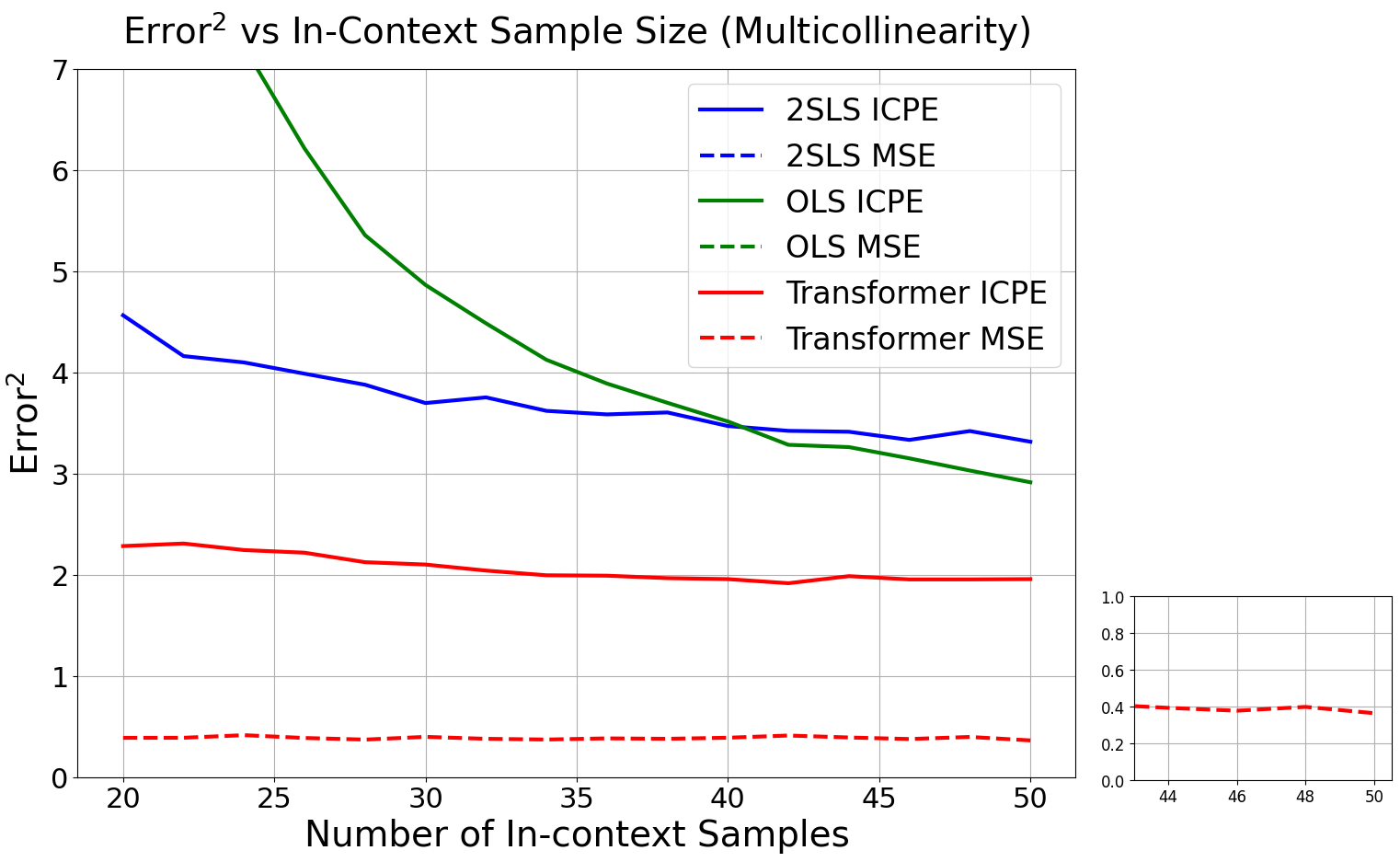}
        \caption{}
        \label{fig:multicollinearity}
    \end{subfigure}
    \begin{subfigure}[b]{0.48\textwidth}
        \includegraphics[width=\textwidth]{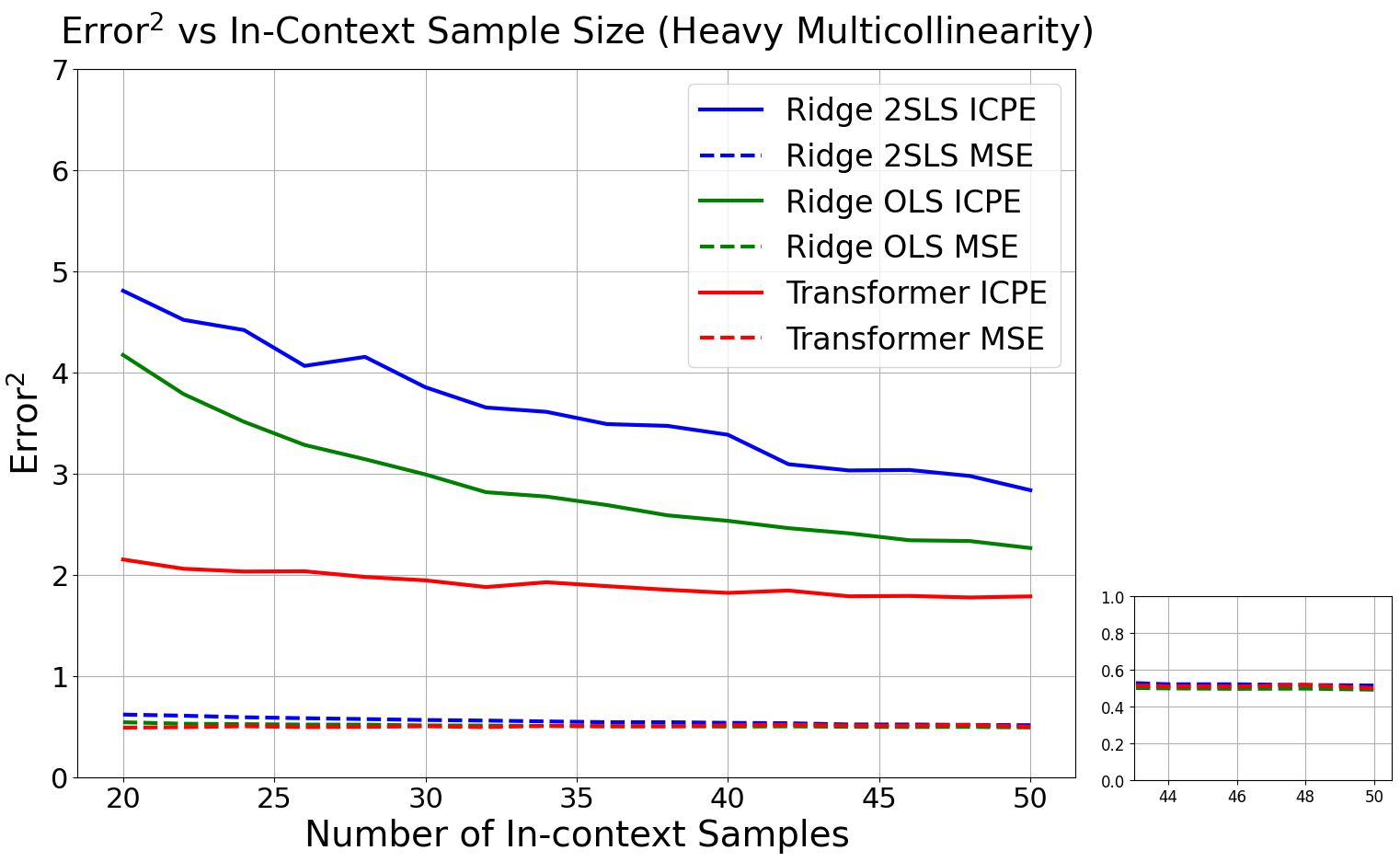}
        \caption{}
        \label{fig:heavy-multicollinearity}
    \end{subfigure}
    \caption{The ICL performance of the trained transformer model in endogeneity tasks with multicollinearity. (a) 1 collinear column in $\boldsymbol{X}$, and 1 collinear column in $\boldsymbol{Z}$. Note that the coefficient MSEs for $\textsf{2SLS}$ and $\textsf{OLS}$ are both out of range. (b) 2 collinear columns in $\boldsymbol{X}$, and 5 collinear columns in $\boldsymbol{Z}$. We compare the performance to the $\ell _2$-regularized $\textsf{2SLS}$ and $\textsf{OLS}$ estimators. The curves are averaged over 500 simulations.}
\end{figure}

We further examine a more difficult case where heavy multicollinearity occurs. We first generate 3 column of $\boldsymbol{X}$, and 5 column of $\boldsymbol{Z}$, and then set $\boldsymbol{X}_{:,j}\sim\mathcal{N}(2\boldsymbol{X}_{:,j-2}, 10^{-6}\boldsymbol{I}_n)$ for $j=4,5$, and $\boldsymbol{Z}_{:,j}\sim\mathcal{N}(2\boldsymbol{Z}_{:,j-5},10^{-6}\boldsymbol{I}_n)$ for $j=6,7,8,9,10$. For better comparisons, we now compare the performance of the trained transformer model to the $\ell_2$-regularized \textsf{2SLS} and \textsf{OLS} estimators (with all regularization parameters set to 1). The results are shown in Figure \ref{fig:heavy-multicollinearity}.

These results suggest that the trained transformer model is capable to handle multicollinearity in IV regression problems, even though it has not been specifically trained with multicollinearity tasks. 

\subsection{Experiments on Complex Non-Linear IV Problem}\label{sec:exp complex non-linear}

As a supplement to Section \ref{sec:sim results}, we examine a more complex scenario where the instrument $\boldsymbol{z}$ has non-linear effect on the endogenous regressor $\boldsymbol{x}$. We consider the following data generating process:
\begin{align*}
    y&=\langle\boldsymbol{\beta},\boldsymbol{x}\rangle+\epsilon_1, \quad\text{and}\quad
    \boldsymbol{x}=g(\boldsymbol{z})+\boldsymbol{\epsilon_2},
\end{align*}
where $g:\mathbb{R}^q\rightarrow \mathbb{R}^p$ is a two-layer fully connected neural network with ReLU activation function. Similar to Section \ref{sec:simulation}, the test prompts are generated using Algorithm \ref{alg:data generating process}, with all task parameters and weights of neural network sampled from standard Gaussian distribution. The results are shown in Figure \ref{fig:nonlinear}. From this figure, we can see that the trained transformer model still achieves optimal performance in this complex non-linear setting.
\begin{figure}[h]
    \centering
    \includegraphics[width=0.6\textwidth]{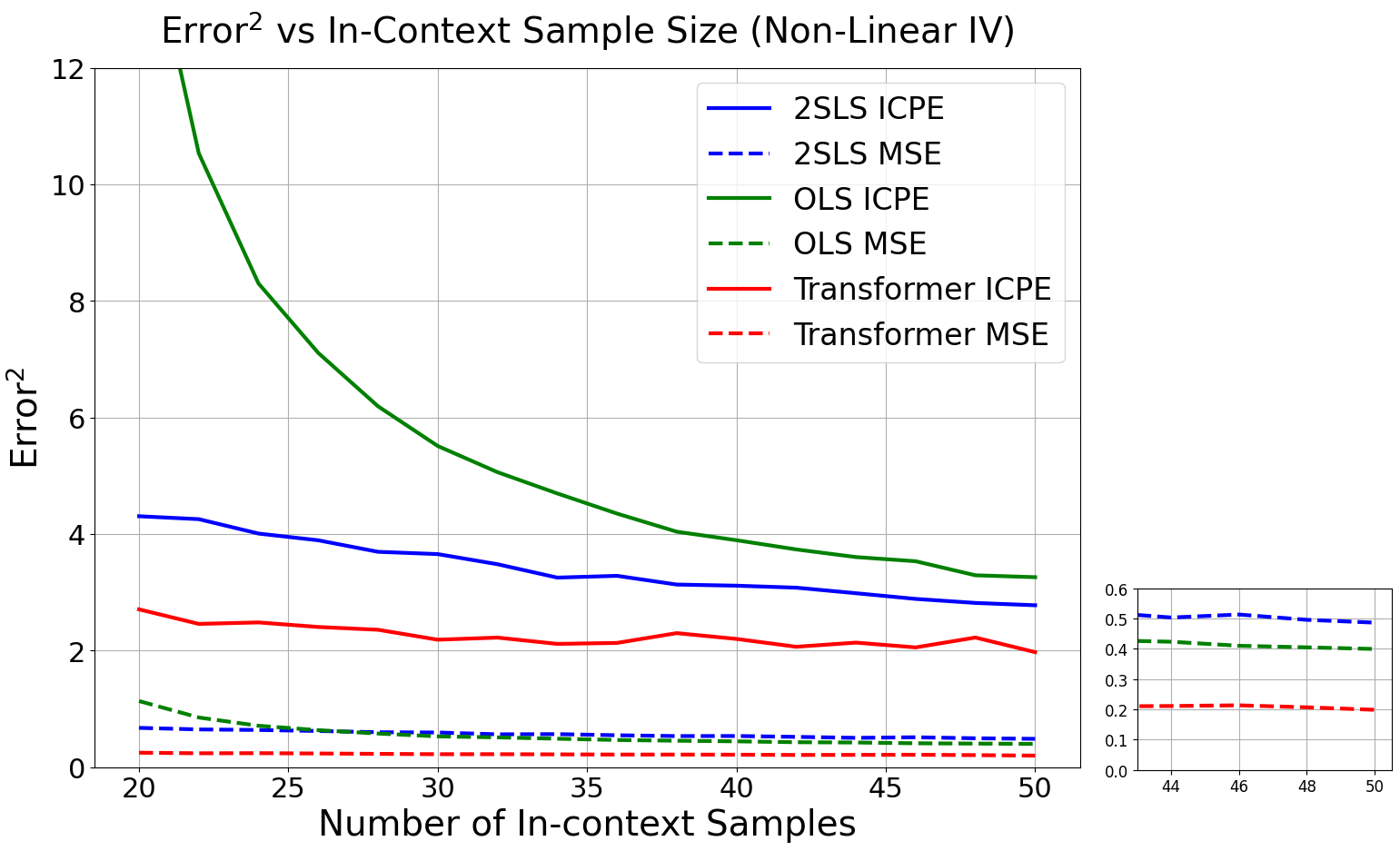}
    \caption{The ICL performance of the trained transformer model in complex non-linear endogeneity tasks where the IV has non-linear effect on the endogenous variable. The curves are averaged over 500 simulations.}
\label{fig:nonlinear}
\end{figure}

\subsection{Experiments on Varying Endogeneity Strength}\label{sec:exp vary endogeneity}
As a supplement to Section \ref{sec:sim results}, we examine the performance of the trained transformer model in standard IV tasks with varying endogeneity strengths. The strength of endogeneity is determined by the correlation between $\boldsymbol{x}$ and the endogenous error $\epsilon_1$. To vary the endogeneity strength, in Algorithm \ref{alg:data generating process}, we multiply $\boldsymbol{u}$ by a factor $r\in(0,2)$ when generating test prompts. The results are shown in Figure \ref{fig:endogeneity_strength}, which illustrates that the trained transformer model is comparable with the optimal \textsf{2SLS} estimator in these standard IV tasks, regardless of the endogeneity level.
\begin{figure}[H]
    \centering
    \includegraphics[width=0.6\textwidth]{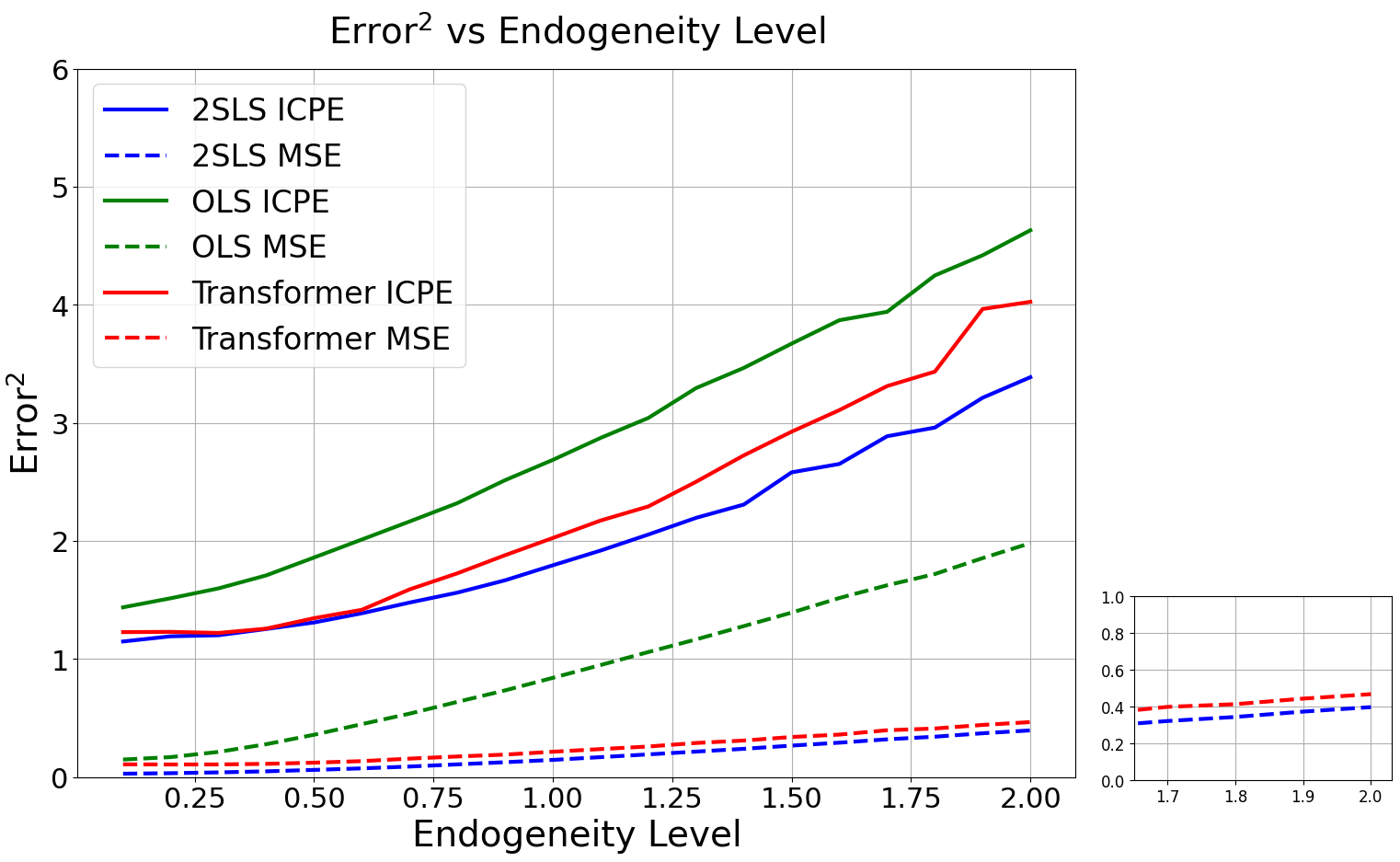}
    \caption{The ICL performance of the trained transformer model in tasks with varying endogeneity strengths. The curves are averaged over 500 simulations.}
\label{fig:endogeneity_strength}
\end{figure}

\subsection{Experiments on Real-World Dataset}
In this section we provide an example to illustrate how the pretrained transformer model can be applied to a real-world dataset. We use the dataset from the study of \cite{laborsupply}. This study investigates the effect of childbearing on labor supply. For demonstration purpose, we consider a simplified setup. We focus on a subset of the dataset that contains 6421 samples from Alabama. The outcome variable $y$ is mother's labor supply (number of working weeks in a year divided by 52), the endogenous variable $x$ is the number of children ($\geq 2$), and the instrument $z$ is an indicator variable of whether the first and second children are of the same sex\footnote{Research found that parents of same-sex siblings are significantly more likely to go on to have an additional child \citep{Westoff1972}, while it is not directly correlated with mother's labor supply as mixture of sex of the first two children can be considered as randomly assigned.}.

For each run we randomly select 50 samples from the dataset, and make the boxplot of the estimated $\beta$ over 500 runs\footnote{For large enough model that can fit in the entire dataset, this step can be ignored. As shown in the simulation study in Section \ref{sec:simulation}, a single estimate is expected to perform at least as good as \textsf{2SLS} estimator, given the same number of samples.}. As the ground truth effect $\beta$ is unknown, we compare them to the \textsf{OLS} and \textsf{2SLS} estimates over all samples. The results are shown in Figure \ref{fig:labor_supplyt.png}.

\begin{figure}[h]
    \centering
    \includegraphics[width=0.6\textwidth]{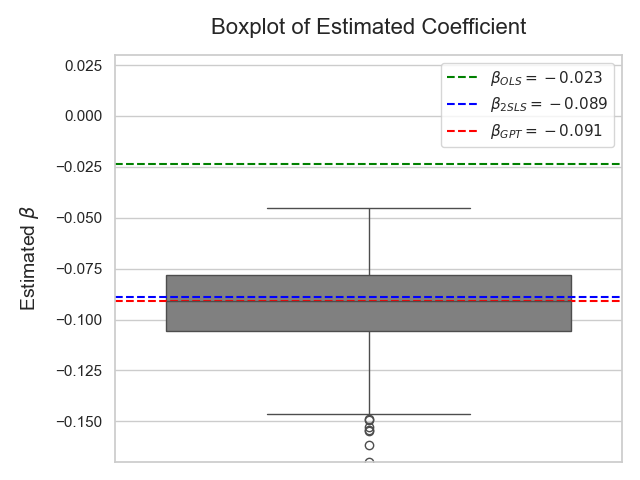}
    \caption{The boxplot of transformer's estimates over 500 runs on the labor supply dataset, comparing to the \textsf{OLS} and \textsf{2SLS} estimates. $\beta_{\textsf{GPT}}$ is taken to be the median of all runs. The gray box represents the interquartile range, where the middle 50\% of the estimated values fall. The whiskers of the box indicate the spread of the estimates. Any points falling outside of the whisker can be considered as outliers.}
\label{fig:labor_supplyt.png}
\end{figure}

The final estimate $\beta_{\textsf{GPT}}=-0.091$, which suggests that with each increase in the number of children, the mother's labor supply is expected to drop 9.1\% (approximately 4.73 weeks per year). This result is closer to the \textsf{2SLS} estimate $\beta_{\textsf{2SLS}}=-0.089$ than the \textsf{OLS} estimate $\beta_{\textsf{OLS}}=-0.023$. This example demonstrates the potential of the pretrained transformer model in handling real-world IV problems.

\subsection{Experimental Detail}
The training of the transformer in our experiment was conducted on a Windows 11 machine with the following specifications:

\begin{itemize}
    \item GPU: NVIDIA GeForce RTX 4090
    \item CPU: Intel Core i9-14900KF
    \item Memory: 32 GB DDR5, 5600MHz
\end{itemize}

The training process took around 10 hours.  
\end{document}